\documentclass[]{fairmeta}

\usepackage{subcaption}

\usepackage{amsmath,amsthm,amssymb,bm, mathtools}
\usepackage{amsfonts}
\usepackage{thmtools} 
\usepackage{bbm}
\usepackage{algorithm}
\usepackage{algorithmic}

\usepackage{float}
\newfloat{algorithm}{t}{lop}
\usepackage{enumitem}

\usepackage{pifont}

\newcommand{\eTnn}{\texttt{e3nn}\@\xspace}

\definecolor{RoyalBlue}{rgb}{0.25, 0.41, 0.88}
\newcommand{\cellhi}{\cellcolor{RoyalBlue!15}}

\newcommand{\cev}[1]{\reflectbox{\ensuremath{\vec{\reflectbox{\ensuremath{#1}}}}}}

\newenvironment{talign*}
 {\csname align*\endcsname}
 {\endalign}

\newcommand{\pbase}{p^{\text{base}}}
\newcommand{\pbasepbc}{\bar{p}^{\text{base}}}

\renewcommand{\P}{\mathcal{P}}

\newcommand{\T}{\mathbb{T}}
\newcommand{\Z}{\mathbb{Z}}
\newcommand{\R}{\mathbb{R}}
\newcommand{\E}{\mathbb{E}}

\newcommand{\fX}{\bm{X}}
\usepackage{xspace}
\newcommand*{\eg}{{\it e.g.}\@\xspace}
\newcommand*{\ie}{{\it i.e.}\@\xspace}

\def\dt{{dt}}
\def\rd{{d}}
\newcommand{\norm}[1]{\|#1\|}

\newcommand{\pr}[1]{\left(#1\right)}
\DeclareMathOperator*{\argmin}{arg\,min}

\newcommand{\KL}{D_{\text{KL}}}
\newcommand{\dif}{\operatorname{d\!}{}}

\DeclarePairedDelimiterX{\infdivx}[2]{(}{)}{%
  #1\;\delimsize\|\;#2%
}
\newcommand{\kldiv}{D_{\text{KL}}\infdivx}

\definecolor{cadetblue}{rgb}{0.37, 0.62, 0.63}
\definecolor{mygray}{gray}{0.95}
\definecolor{mediumpurple}{rgb}{0.46, 0.44, 0.68}
\definecolor{periwinkle}{rgb}{0.8, 0.8, 1.0}
\definecolor{purplemountainmajesty}{rgb}{0.59, 0.47, 0.71}
\newcommand{\graybox}[1]{%
\vspace{-1em} 
\begin{center}			
\colorbox{mygray} {		
\begin{minipage}{0.987\linewidth} 	
\centering
\vspace{-1em}   
{#1}    
\end{minipage}}			
\end{center}
}

\makeatletter
\DeclareRobustCommand{\cev}[1]{%
  {\mathpalette\do@cev{#1}}%
}
\newcommand{\do@cev}[2]{%
  \vbox{\offinterlineskip
    \sbox\z@{$\m@th#1 x$}%
    \ialign{##\cr
      \hidewidth\reflectbox{$\m@th#1\vec{}\mkern4mu$}\hidewidth\cr
      \noalign{\kern-\ht\z@}
      $\m@th#1#2$\cr
    }%
  }%
}
\makeatother

\usepackage[noabbrev,nameinlink]{cleveref}
\theoremstyle{plain}  
\newtheorem{theorem}{Theorem}[section]  
\newtheorem{lem}[theorem]{Lemma}
\newtheorem{prop}[theorem]{Proposition}

\theoremstyle{definition}  
\newtheorem{definition}[theorem]{Definition}

\title{Adjoint Sampling: Highly Scalable Diffusion Samplers via Adjoint Matching}

\author[2,\dagger,*]{Aaron Havens}
\author[1,*]{Benjamin Kurt Miller}
\author[1,3,*]{Bing Yan}
\author[4]{Carles Domingo-Enrich}
\author[1]{Anuroop Sriram}
\author[1]{Brandon Wood}
\author[1]{Daniel Levine}
\author[2]{Bin Hu}
\author[1]{Brandon Amos}
\author[1]{Brian Karrer}
\author[1,*]{Xiang Fu}
\author[1,*]{Guan-Horng Liu}
\author[1,*]{Ricky T. Q. Chen}

\affiliation[1]{FAIR at Meta}
\affiliation[2]{University of Illinois}
\affiliation[3]{New York University}
\affiliation[4]{Microsoft Research New England}

\contribution[*]{Core contributors}
\contribution[\dagger]{Work done during internship at FAIR}

\abstract{
We introduce \emph{Adjoint Sampling}, a highly scalable and efficient algorithm for learning diffusion processes that sample from unnormalized densities, or energy functions. It is the first on-policy approach that allows significantly more gradient updates than the number of energy evaluations and model samples, allowing us to scale to much larger problem settings than previously explored by similar methods.
Our framework is theoretically grounded in stochastic optimal control and shares the same theoretical guarantees as Adjoint Matching, being able to train without the need for corrective measures that push samples towards the target distribution.
We show how to incorporate key symmetries, as well as periodic boundary conditions, for modeling molecules in both cartesian and torsional coordinates.
We demonstrate the effectiveness of our approach through extensive experiments on classical energy functions, and further scale up to neural network-based energy models where we perform amortized conformer generation across many molecular systems.
To encourage further research in developing highly scalable sampling methods, we plan to open source these challenging benchmarks, where successful methods can directly impact progress in computational chemistry.
}

\metadata[Code \& benchmark]{\url{https://github.com/facebookresearch/adjoint_sampling}}
\metadata[Models]{\url{https://huggingface.co/facebook/adjoint_sampling}}

\begin{document}

\maketitle

\section{Introduction}
Sampling from complex, high-dimensional distributions underlies many important problems in computational science, with applications spanning molecular modeling, Bayesian inference , and generative modeling. In particular, we are interested in sampling from the target distribution with only access to its unnormalized energy function $E$, which defines the Boltzmann distribution
\begin{equation}\label{eq:p_star}
    \mu(x) = \frac{\exp\left( -\frac{1}{\tau} E(x) \right)}{Z},
\end{equation}
where $Z = \int_{\mathbb{R}^d} \exp \left(-\tfrac{1}{\tau}E(x)\right) \dif{x} < \infty$ is the unknown normalization constant. The Boltzmann distribution describes the equilibrium state of many physical systems, where $E(x)$ denotes the energy of a configuration $x$, and $\tau > 0$ is a temperature parameter. Efficiently sampling from such distributions remains challenging, especially for high-dimensional systems with intricate energy landscapes. Additionally many energy functions are extremely computationally expensive, e.g. requiring physics simulations.

Traditional approaches, such as Markov Chain Monte Carlo (MCMC) and Sequential Monte Carlo (SMC) using well-designed Markov Chains~\citep{neal2001annealed,neal2011mcmc,del2006sequential}, provide asymptotically unbiased samples but often suffer from slow mixing and poor scalability to high-dimensional settings. This necessitates the design of better transition densities and smarter proposal distributions. Recent works try to address this by
augmenting sampling with learned proposal distribution~\citep{albergo2019flow, arbel2021annealed, gabrie2022adaptive} via normalizing flows~\citep{chen2018neural,rezende2015variational}.\looseness=-1

It may seem natural to look towards the recent explosion of diffusion and flow-based generative models~\citep{song2019generative,ho2020denoising, lipman2023flow, albergo2023stochastic}. 
However, a na\"ive adaptation of these data-driven generative modeling frameworks requires access to ground truth data. 
This limitation is particularly significant in applications such as molecular simulations and physics-based inference, where direct access to samples from the target distribution is often unavailable. 
As a result, prior attempts often require an augmentation with sequential Monte Carlo or importance-sampling \citep{phillips2024particle,de2024target,akhounditerated}, making these methods highly inefficient in terms of energy function evaluations.\looseness=-1

The connection between sampling and diffusion processes was established by~\citet{tzen2019theoretical} through using classical results of stochastic optimal control (SOC) and Schrodinger-Bridge problems~\citep{pavon1989stochastic, daipra91, follmer2005entropy, chen2016relation}. Using this formulation, \citet{zhang2022path} show that by directly parameterizing the drift of the controlled process, one could solve an SOC problem given unnormalized density for sampling tasks.
This concept is further generalize by~\citet{berner2023optimal} and \citet{richter2024improved}; however, all of these method require computationally expensive simulation of the diffusion process per gradient update. Furthermore, they require at least one---sometimes many---energy evaluations per gradient update.

To overcome these challenges, we introduce \emph{Adjoint Sampling}, a novel and extremely efficient variational inference framework based on stochastic control of diffusion processes, which we apply at much larger scale than previous methods. 
Our method is built on top of Adjoint Matching \citep{domingoenrich2024adjoint}, a recent method developed for solving general stochastic control problems which we specialize and improve for efficiently learning to sample.

\begin{figure}[t!]
    \centering
    \includegraphics[width=1.0\linewidth]{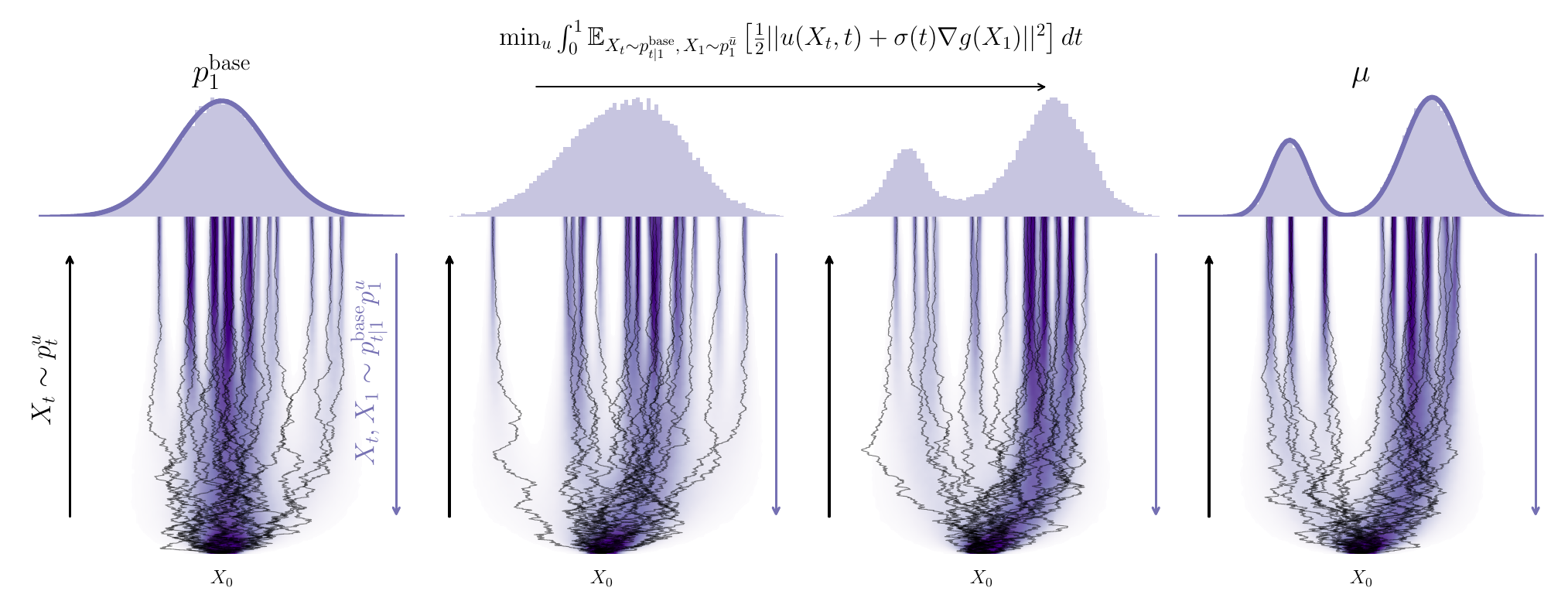}
    \caption{
    Starting from the uncontrolled diffusion process $\pbase_1$ (left-most panel), Adjoint Sampling uses the Reciprocal Projection {\color{mediumpurple}{$(X_t, X_1) \sim \pbase_{t|1}p^u_1$}} of the current controlled SDE to approximate the joint trajectory distribution $p^u_t$, allowing us to take several gradient steps on the RAM objective per evaluated sample and energy gradient $(X_1$, $\nabla g(X_1))$. After several iterations, Adjoint Sampling converges to the target Boltzmann $\mu$ (right-most panel).}
    \label{fig:enter-label}
\end{figure}

Our contributions are as follows:
\begin{itemize}
    \item \textbf{Efficiency}: Our method is the first on-policy approach to allow far more gradient updates per model sample and energy evaluation. This is extremely important for scaling up to difficult amortized settings.
    \item \textbf{Theoretically Grounded}: Our method is grounded in stochastic control, following the recent work of Adjoint Matching \citep{domingoenrich2024adjoint}. We improve upon Adjoint Matching by proposing a new objective that implicitly projects the model onto a set of optimal controls.
    We do not require corrective measures such as importance sampling or sequential Monte Carlo for our theoretical guarantees.
    \item  \textbf{Structure}: Graph and Euclidean symmetries can be easily incorporated. We also adopt the method to handle periodic boundary conditions for modeling torsion angle representations of molecular systems.
    \item  \textbf{New Benchmarks}: We introduce amortized molecule sampling benchmarks that challenge new methods to be applicable at scale. Being successful at these benchmarks directly drives progress in chemistry.
\end{itemize}

\section{Preliminaries}

In this section, we briefly introduce an optimization problem over a class of stochastic processes that samples from $\mu$ at the optimal solution at the fixed time $t=1$. This control perspective of sample generation was observed by~\citet{tzen2019theoretical}, with similar formulations found in \citet{zhang2022path, berner2023optimal}. A more in-depth introduction to stochastic control for sampling can be found in \Cref{app:SOC}.

\subsection{Optimizing Diffusion Processes for Sampling}
Consider the stochastic differential equation (SDE)
\begin{align}\label{eq:controlled_sde}
    \dif{X_t} = 
    \sigma(t) u(X_t, t) \dif{t} + \sigma(t) \dif{B_t},\quad X_0 = 0,
\end{align}
where $\sigma:[0,1]\rightarrow \mathbb{R}$ is a scalar noise function, $u: \mathbb{R}^d \times [0,1] \rightarrow \mathbb{R}^d$
is a learnable \emph{control} and $(B_t)_t$ is a $d$-dimensional Brownian motion. Under mild conditions on $\sigma$ and $u$, this controlled process~\eqref{eq:controlled_sde} uniquely defines time-marginals $p^u_{t}(X_t)$ for all $t\in \mathbb{R}$ and we seek a drift $u$ such that $p^u_{1}(X_1) = \mu(X_1)$. We denote $\fX = \{X_t : 0\leq t \leq 1 \}$ a sample trajectory and $p^u(\fX)$ and $\pbase(\fX)$ the distributions over trajectories\footnote{
    Formally, path measures of stochastic processes and their KL divergences are defined by Radon-Nikodyn derivative, \eg $\frac{\rd \mathbb{P}^u}{\rd \mathbb{P}^\text{base}}(\fX) = \exp(\int^1_0 \frac{1}{2} \lVert u_t \rVert^2 \dt + \int^1_0 u_t \cdot \rd B_t )$ with $\fX \sim \mathbb{P}^u$.
} generated by the controlled process~\eqref{eq:controlled_sde} and base process $\dif{X_t} = \sigma(t) \dif{B_t}$ (i.e. setting $u\equiv 0$ in \eqref{eq:controlled_sde}), respectively.
The drift that transports the Dirac distribution to the target density $\mu$ at time $t=1$ is not unique; however, we can choose a particular target path density that is a \emph{Schr\"odinger bridge}~\citep{pavon1989stochastic, chen2016relation}  between the Dirac and the target distribution given by
\begin{equation}\label{eq:brownian_bridge}
    p^*(\fX) = \pbase(\fX|X_1)\mu(X_1),
\end{equation}
where $\pbase(\fX|X_1)$ is the posterior distribution of the base process conditioned on arriving at $X_1$. This Schr\"odinger bridge is, among other processes that coincides with $\mu$ at $t=1$, the one that deviates the least from the base process in KL divergence. 
Notably, the path KL to $p^*$ over trajectories with respect to the control $u$ can be expressed analytically by factorizing $p^*(\fX) = \pbase(\fX)\tfrac{\mu(X_1)}{\pbase_1(X_1)}$ and invoking Girsanov's Theorem~\citep{protter2005stochastic},
\begin{align}
\label{eq:KL_regularized_main}
    \mathcal{L}_\text{SOC}(u) = \kldiv{p^{u}(\fX)}{p^*(\fX)} 
    &= \kldiv{p^u(\fX)}{\pbase(\fX)} + \mathbb{E}_{p^u} \log\left(\frac{\pbase_{1}(X_1)}{\mu(X_1)} \right)
    \\
    &=\label{eq:soc_cost}
       \mathbb{E}_{p^u} \left[\int_0^1 \tfrac{1}{2}\lVert u(X_t, t) \rVert^2 \dif{t} + \log\left(\frac{\pbase_1(X_1)}{\mu(X_1)} \right)\right]. 
\end{align}
To this end, we propose to minimize the objective~\eqref{eq:soc_cost}, which, intuitively, minimizes the control energy required to transport to $\mu$. 
In fact, this criteria corresponds exactly with a minimum-energy stochastic optimal control (SOC) problem  with terminal cost function given by $\log \tfrac{\pbase_1(x)}{\mu(x)}$. 
This particular SOC problem is heavily studied in literature and has a unique optimal solution $u^*$ satisfying~\citep{kappen2005path, tzen2019neural} (See \Cref{app:SOC} for more details)
\begin{align}\label{eq:KL_minimized}
    \kldiv{p^{u^*}(\fX)}{p^*(\fX)} =-\log \mathbb{E}_{\pbase} \left[\frac{\mu(X_1)}{\pbase_1(X_1)}\right] = 0.
\end{align}
Therefore $p^*$ can indeed be achieved by $u^*$ and our problem is then reduced to optimizing the SOC objective \eqref{eq:soc_cost} as efficiently as possible. This is precisely what our proposed method \emph{Adjoint Sampling} addresses.

\section{Adjoint Sampling}
Now that we have equated the sampling problem as the solution of an SOC problem, we will now propose a highly efficient method to solve SOC problems of the form
\begin{subequations}\label{eq:soc} 
\begin{eqnarray}
    \label{eq:soc_terminal_cost} 
    \min_{u} \mathbb{E}_{\fX \sim p^u}\left[ \int_0^1 \tfrac{1}{2} \lVert u(X_t, t) \rVert^2 \dif{t} + g(X_1)\right]\\
    \label{eq:soc_terminal_cost_2}
    \text{ s.t. } \dif{X_t} = 
    \sigma(t) u(X_t, t) \dif{t} + \sigma(t) \dif{B_t},\quad X_0 = 0 
\end{eqnarray}
\end{subequations}
which we will use to learn a stochastic process that samples from a target distribution given only an unnormalized energy function~\eqref{eq:p_star} with
\begin{equation}
    g(x) = \log \pbase_1(x) + \tfrac{1}{\tau}E(x),
\end{equation}
which is \eqref{eq:soc_cost} without the normalization constant as it does not affect the optimal solution.
Na\"ively, one could simulate the SDE \eqref{eq:soc_terminal_cost_2}, and then differentiate through the objective \eqref{eq:soc_terminal_cost}, often referred to as the adjoint method~\citep{BrysonHo69}, and is the method of choice for prior works~\citep{chen2018neural,zhang2022path}. However, this is very slow as it requires two simulations per iteration, one for the state $X_t$ and another one for backpropagating through the SDE.

\subsection{Adjoint Matching} \label{subsec:AM}

Adjoint Matching \citep{domingoenrich2024adjoint} is an algorithm designed to solve a more general family of stochastic optimal control problems, being able to make use of base processes with arbitrary drift $\dif{X_t} = b(X_t,t)\dif{t} + \sigma(t)\dif{B_t}$. Unlike standard adjoint methods~\citep{BrysonHo69} that use gradient-based approaches and differentiate through the objective \eqref{eq:soc_terminal_cost}, Adjoint Matching turns the problem into a moving regression formulation.
In particular, it solves for the fixed point $u$ such that $u(x, t) = -\sigma(t) \nabla J(u; x, t)$ where $J$ is the value function. We discuss Adjoint Matching in detail in \Cref{sec:additional_prelims}. For our objective, the Adjoint Matching loss is 
\begin{align}\label{eq:lean_adjoint_matching_prelim}
\mathcal{L}_{\mathrm{AM}}(u) 
:= \mathbb{E}_{p^{\bar u}}\Bigg[\frac{1}{2} \int_0^{1} \big\| & u(X_t 
,t)
+ \sigma(t)^{\top} \tilde{a}(t;\bm{X}) \big\|^2 \, \mathrm{d}t \Bigg], 
\qquad \fX \sim p^{\bar{u}}, \quad \bar{u} = \texttt{stopgrad}(u), \\
\label{eq:lean_adjoint_prelim_1}
\text{s.t. }\quad \frac{\mathrm{d}}{\mathrm{d}t} \tilde{a}(t;\bm{X}) 
&= - \tilde{a}(t;\bm{X})^{\top} \nabla b (X_t,t),\quad
\tilde{a}(1;\bm{X}) = \nabla g(X_1).
\end{align}
where $\tilde{a}$ is referred to as the \emph{lean adjoint} state, and $\bar{u} = \texttt{stopgrad}(u)$ denotes a stop gradient operation of $u$,
\textit{i.e.}, although $\fX \sim p^{\bar{u}}$ is sampled according to the controlled process. 

Our first key observation is that the original Adjoint Matching algorithm can be significantly simplified in the case of our sampling formulation \eqref{eq:soc}. In particular if we set set $b\equiv 0$, the lean adjoint state equation $\tfrac{\dif}{\dif{t}}\tilde a(t;\fX)=0, \quad \tilde a(1,\fX) = \nabla g(X_1)$ admits the unique analytical solution $\tilde a(t;\fX)= \nabla g(X_1)$ for all $t \in [0,1]$. This leads to a drastically simplified regression loss requiring no additional simulation of the lean adjoint state.
\begin{align}\label{eq:adjoint_matching_terminal_cost}
    \mathcal{L}_{\text{AM}}(u) = \mathbb{E}_{\fX \sim p^{\bar{u}}} \left[  \int_0^1 \frac{1}{2} \lVert u(X_t, t) + \sigma(t) \nabla g(X_1) \rVert^2 \dif{t} \right] 
\end{align}
Roughly speaking, Adjoint Matching offers a simple interpretation: for each intermediate state $X_t$, simply regress the control onto the negative gradient of the terminal cost $- \nabla g(X_1)$ for all possible $X_1$ that can be arrived at from $X_t$. Since this is a moving target, this will, over the course of optimization, slowly shift the process towards regions with small terminal cost $g$.

\subsection{Reciprocal Adjoint Matching} \label{sec:reciprocal_adjoint_matching}

Adjoint Matching still requires two computationally expensive operations at every iteration: (i) simulation of the controlled process and (ii) evaluation of the terminal cost.
Simulation of an SDE is a slow iterative procedure, and the terminal cost can also be costly. 
For instance, in our formulation, the terminal cost involves the energy model $E$ which in many real-world scenarios is extremely computationally expensive. 
Ideally, neither of these two operations should be done frequently.
In this section, we propose modifications to the Adjoint Matching algorithm to make it highly scalable, removing both requirements.

Firstly, notice that the training loss only depends on sampling $(X_t, X_1) \sim p^u_{t,1}$, not the full trajectory. Hence \eqref{eq:adjoint_matching_terminal_cost} is equivalent to sampling pairs $(X_t, X_1)$ from the joint distribution defined by the stochastic process $p^u$:
\begin{align}\label{eq:am}
\begin{split}
    \mathcal{L}_{\text{AM}}(u) = &\int_0^1 \mathbb{E}_{(X_t, X_1) \sim p^{\bar{u}}_{t,1}} \left[\frac{1}{2}\lVert u(X_t, t) + \sigma(t) \nabla g(X_1) \rVert^2 \right] \dif{t}.
\end{split}
\end{align}
This does not yet yield any efficiency gains as the only way to sample $(X_t, X_1) \sim p^u$ is through simulating the controlled process. 
Our key insight is to make use of the knowledge that at the optimal solution $u^*$ we have that 
$p^{u^*}_{t,1}(X_t, X_1) = \pbase_{t|1}(X_t | X_1)\mu(X_1)$ due to \eqref{eq:brownian_bridge}. 
Although the current control $u$ does not necessarily satisfy this property, we can project the path measure---known as a Reciprocal projection which we make formal in \Cref{sec:theory}---onto the Schr\"odinger bridge that generates $p^{\bar u}_{1}$ given by $\pbase_{t|1}(X_t| X_1)p^{\bar u}(X_1)$.
Based on this, we propose the \textit{Reciprocal Adjoint Matching} (RAM) objective:
\graybox{
\begin{align}\label{eq:ram}
\begin{split}
    \mathcal{L}_{\text{RAM}}(u) =
    \int_0^1 \lambda(t) \mathbb{E}_{X_t \sim \pbase_{t|1}, X_1 \sim p^{\bar u}_1} \left[\frac{1}{2}\lVert u(X_t, t) 
    + \sigma(t) \nabla g(X_1) \rVert^2 \right] \dif{t}.
\end{split}
\end{align}}
That is, we sample $X_1$ according to the controlled process, then sample $X_t$ conditioned on $X_1$ using the posterior distribution defined by the base process.
Since the base process is an SDE with zero drift (equivalent to \eqref{eq:controlled_sde} with $u \equiv 0$), conditionals $\pbase_{t|1}(X_t | X_1)$ are known in closed form and can be easily sampled.
This also has the effect of de-correlating samples across time because the $X_t$ samples for different values of $t$ are now conditionally independent given $X_1$, whereas in \eqref{eq:am} the $X_t$ are sampled from the same trajectory.
We additionally apply a time scaling of $\lambda(t) = \frac{1}{\sigma(t)^2}$ which does not affect the optimal solution but improves numerical stability.

We note that \eqref{eq:ram} is related to the training objectives that appear in PDDS \citep{phillips2024particle} and TSM \citep{de2024target}, where the same formula appears inside the expectation; however, the expectations differ and makes a \emph{significant} difference. 
In particular, PDDS and TSM take an expectation with respect to the optimal control, whereas we take an expectation with respect to the current control. 
That is, they require sampling from the target distribution, which is intractable. Unlike PDDS and TSM, which is a simple regression formulation with a fixed regression target, our formulation has a \emph{moving regression target} and is designed to solve for a fixed point (see \Cref{sec:additional_prelims}). 

We can further increase computational efficiency by noting that we can fix the regression target and delay updating it (which we justify theoretically in \Cref{sec:theory}.
To design a \emph{highly efficient} algorithm, we further decouple $p(X_1)$ from the regression problem of learning $u$, delaying updates to $p(X_1)$ and performing multiple iterations to train $u$.
Practically, this leads to the following alternating algorithm which we refer to as \emph{Adjoint Sampling}:\looseness=-1
\begin{enumerate}
    \item \label{enum:alg1} Using the current control $u_i$, construct a buffer $\mathcal{B} = (X_1^{(i)}, \nabla g^{(i)})$ with samples $\{X_1^{(i)}\} \stackrel{iid}{\sim} p^u_1(X_1)$ and $\nabla g^{(i)} = \nabla g(X_1^{(i)})$.
    \item \label{enum:alg2} Obtain updated control $u_{i+1}$ by optimizing \eqref{eq:ram} using samples $\{X_1^{(i)}, \nabla g^{(i)}\} \sim \mathcal{B}$.
\end{enumerate}
We theoretically justify this alternating scheme in the next section as implicitly performing a projection onto a more optimal control that aids in decreasing the SOC objective.
Note that optimizing~\eqref{eq:ram} in Step \ref{enum:alg2} above is extremely cheap due to the analytical form of the forward base process.
Both $\pbase_1$ and the posterior distribution $\pbase_{t|1}$ can be designed as closed-form Gaussian distributions which we detail in \Cref{sec:base_processes}.
This means that \textit{many gradient updates can be carried out without needing to simulate the controlled process or evaluation of the energy model}. We provide detailed pseudo-code in \Cref{alg:adjoint_sampling}.

\begin{algorithm}
\caption{Adjoint Sampling}
\label{alg:adjoint_sampling}
\begin{algorithmic}[1]
        \STATE \textbf{Input:} Terminal Cost: $ g=\log \pbase + \frac{1}{\tau}E$, base process: $p_t^{base}$ given by $d X_t = \sigma(t) d B_t$, outer-loop batch size: $n$, inner-loop batch size $m$, SDE drift network: $u_\theta$, Replay buffer $\mathcal{B}$.
       
        \STATE $\mathcal{B} \leftarrow \varnothing$\
        \WHILE{Outer-Loop}
        \STATE {\color{gray}\# Euler-Maruyama with \textbf{no gradient}}
            \STATE $\{X_1^{(i)}\}_{i=1}^n \sim p^{\bar{u}}_{1}$, \qquad  $\bar{u} = \texttt{stopgrad}(u_\theta)$
            \STATE {\color{gray}\# gradient of energy is evaluated \textbf{once} per sample}
            \STATE $\nabla g^{(i)} \leftarrow \nabla g(X_1^{(i)})$ 
            \STATE $\mathcal{B} \leftarrow \mathcal{B}\cup\{(X_1^{(i)}, \nabla g^{(i)})\}_{i=1}^n$
            \WHILE{Inner-Loop}
            \STATE $\{(X_1^{(j)}, \nabla g^{(j)})\}_{j=1}^{m}  \sim \mathcal{U}(\mathcal{B})$
            \STATE $t^{(j)} \sim \mathcal{U}([0,1]),\quad X_t^{(j)} \sim p^{base}_{t^{(j)}|1 }(x | X_1^{(j)})$
            \STATE $\mathcal{L}_{\text{RAM}}^{(j)} \leftarrow \frac{\lambda(t)}{2}\Big\Vert u_\theta(X_t^{(j)}, t^{(j)}) +\sigma(t)\nabla g^{(j)}  \Big\Vert^2$ 
            \STATE $\theta \leftarrow \texttt{optimizer\_step}(\theta,\nabla_\theta \frac{1}{m}\sum_j\mathcal{L}_{\text{RAM}}^{(j)})$
            \ENDWHILE
        \ENDWHILE
    \STATE \textbf{Output:} SDE sampler drift $u_\theta$.
    \end{algorithmic}
\end{algorithm}

\subsection{Adjoint Sampling Theory}\label{sec:theory}

In this section, we provide theoretical justification to our proposed \Cref{alg:adjoint_sampling}. 
Firstly, let us define a projection $\P$ of a control $u$ onto the SOC solution where $p^u_1$ is the target distribution:
\begin{equation}\label{eq:mark_projection}
    \P(u) = \argmin_v \kldiv{p^{v}(\fX)}{ \pbase(\fX | X_1)p_1^u(X_1)}.
\end{equation}
This is a projection step onto the Schr\"odinger bridge that shares the same terminal distribution $p^u_{1}(X_1)$. In connection to existing literature, this can be understood as a combination of the Reciprocal and Markovian projections of \citet{shi2024diffusion}.
In terms of minimizing the objective, the resulting control process is consistently as effective as, if not superior to, the original control. We formalize this in the following proposition:
\begin{prop}
    \label{prop:projection}
    After projection \eqref{eq:mark_projection}, Reciprocal Adjoint Matching is equivalent to Adjoint Matching,
    \begin{equation}
        \mathcal{L}_\text{RAM}(\P(u)) = \mathcal{L}_\text{AM}(\P(u)),
    \end{equation}
    and furthermore, this projection improves upon on the SOC objective,
    \begin{equation}
        J(u) \ge J(\P(u)), \qquad \text{where } J(u) := \mathbb{E}_{\fX \sim p^u}\left[ \int_0^1 \tfrac{1}{2} \lVert u(X_t, t) \rVert^2 \dif{t} + g(X_1)\right].
    \end{equation}
\end{prop}
We defer the proof to \Cref{app:Reciprocal}. The result of \Cref{prop:projection} hints that if we explicitly perform this projection \eqref{eq:mark_projection} at every iteration, then we are effectively performing Adjoint Matching iterations with a more optimal control. However, performing the projection requires additional computation steps, \eg, using algorithm such as Bridge Matching \citep{shi2024diffusion}, which we want to avoid. Fortunately, it turns out that using just the RAM loss is sufficient and will \emph{implicitly} include this exact projection.

Our main theoretical result is the observation that the alternating scheme, where the distribution of $X_1$ samples are fixed (step \ref{enum:alg1}) and we update the control by fully converging the RAM loss (step \ref{enum:alg2}), is \textit{implicitly} performing the projection \eqref{eq:mark_projection} while performing Adjoint Matching on the SOC objective \eqref{eq:soc_cost}. This correspondence is stated informally in the following Theorem.

\begin{theorem}[Theoretical guarantees of Adjoint Sampling (\textit{informal})] \label{thm:adjoint_sampling}
    Starting with any control $u_i$, performing steps \ref{enum:alg1} and \ref{enum:alg2} to obtain $u_{i+1}$ equivalently satisfies
    \begin{equation}
        u_{i+1} = \P(u_i) - \frac{\delta \mathcal{L}_{\text{AM}}}{\delta u} (\P(u_i)),
    \end{equation}
    where $\frac{\delta \mathcal{L}_{\text{AM}}}{\delta u}$ denotes the functional derivative with respect to the control $u$. Moreover, the fixed point where $u = \mathcal{P}(u)$ and $u=u-\frac{\delta \mathcal{L}_{\text{AM}}}{\delta u}(u)$
    is the optimal control $u^*$ to \eqref{eq:soc}.
\end{theorem}
A more precise statement and proof can be found in \Cref{app:Reciprocal}. 
This result provides the theoretical justification for our proposed algorithm (\Cref{alg:adjoint_sampling}). 
In practice, we differ slightly and use a replay buffer that contains samples from multiple prior steps, and we do not perform step \ref{enum:alg2} until convergence. We find this helps smoothen the optimization and improve computational efficiency.

\subsection{Geometric Extensions}

Symmetries are essential to efficiently sample from energy-based models defining physical systems. Molecules are symmetric with respect to atom permutations, rotations, translations, and parity (reflection). We reduce dimensionality and improve data efficiency by enforcing symmetries in our model of the controlled process. Our model is either parameterized by E(n) Equivariant Graph Neural Networks (EGNNs)~\citep{satorras2021n} or Tensor Field Networks~\citep{thomas2018tensorfieldnetworks}, implemented in \eTnn~\citep{e3nn_software, jing2022torsional}, which incorporate graph features while respecting symmetry constraints 
\citep{kondor2018clebschgordannets, weiler20183dsteerablecnns, miller2020relevance, geiger2022e3nn}. We experiment with two controlled processes: one sampling atomic positions using EGNN, and the other sampling \emph{torsion} angles with \eTnn. See \Cref{app:conformers_and_torsions} for details.
\subsubsection{Graph-Conditioned $SE(3)$-Invariant Sampling}
Consider the graph $\mathcal{G}(\mathcal{V},\mathcal{E})$. Each node $v_i \in \mathcal{V}$ and edge $e_{ij} \in \mathcal{E}$ has feature attributes relating the topological structure of the data (\eg, atom types and bond orderings), including a spatial coordinate $x_i \in \mathbb{R}^d$. We can condition an equivariant drift $u_\theta$ on the graph features such that it is equivariant to the spatial symmetries given by the group $g \in G:=\text{Aut}(\mathcal{G}) \times SO(d)$ acting via
$g \cdot x = (P \otimes R) x$. (\ie $u_\theta(g \cdot X_t, t; \mathcal{G}) = g \cdot u_\theta(X_t, t; \mathcal{G})$). Here $P \in \text{Aut}(\mathcal{G})$ are graph automorphisms --- represented as a permutation matrix that reorders the graph nodes while preserving its structure --- and $R \in SO(d)$ is an $d$-dimensional rotation matrix. 
\paragraph{Translation Invariance}
To enforce translation invariance, the system is restricted to the zero center-of-mass subspace $\mathcal{X}^\text{CoM} = \{ x \mid \sum_{i=1}^k x^i = 0 \} \subset \mathbb{R}^{kd}$. This is achieved by projecting the particle positions $y \in \mathbb{R}^{kd}$ to the zero CoM subspace using the projection operator:
\begin{align}
x = \mathcal{A} y, \quad \mathcal{A} = \left( I_k - \frac{1}{k} \mathbf{1}_k \mathbf{1}_k^\top \right)\otimes I_d,
\end{align}
where $I_k$ is the identity matrix, and $\mathbf{1}_k$ is a vector of ones. With this we can define a zero CoM process.
\begin{align*}
        \dif{X_t} = \sigma(t) \mathcal{A} u_\theta(X_t, t;\mathcal{G})\dif{t} + \sigma(t) \mathcal{A}\dif{B_t}, \quad X_0 = 0
\end{align*}
Now $\pbase$ is a singular Gaussian $\pbase_t(x) = \mathcal{N}(x ; 0, \nu_t \mathcal{A}\mathcal{A}^\top)$, which can be sampled by first sampling the isotropic Gaussian $\mathcal{N}(x ; 0, \nu_t)$ and then projecting via $\mathcal{A}$. 
We have $\pbase_t(x) \propto \mathcal{N}(x ; 0, \nu_t)$ for all $x\in \mathcal{X}^\text{CoM}$ and the posterior $\pbase_{t|1}$ also projects onto $\mathcal{X}^\text{CoM}$.

\vspace{-0.75em}
\paragraph{Geometric Adjoint Sampling}
By ensuring the drift is $G$-equivariant and zero CoM, we ensure the model distributions are $G$-invariant for all $t \in [0,1]$~\citep{kohler2020equivariant, xu2022geodiff}.
Putting this all together, our RAM loss~\eqref{eq:ram} can be modified to support the proposed symmetries.
\vspace{-0.5em}
\begin{align}
\mathcal{L}_{\text{GeoRAM}}(\theta) =
    \int_0^1 \lambda(t) \mathbb{E}_{\pbase_{t|1}p^{\bar u}_1} \left[\frac{1}{2}\lVert \mathcal{A}(u_\theta(X_t, t;\mathcal{G}) 
    + \sigma(t) \nabla g(X_1)) \rVert^2 \right] \dif{t}
\end{align}
Assuming that energy $E$ and $\pbase_1$ are $G$-invariant, the target $-\sigma(t)\mathcal{A} \nabla g(X_1)$ for our learned drift will be $G$-equivariant and zero CoM.

\subsubsection{Periodic boundary conditions}\label{subsec:flat_torus}
In many cases we may want to model a state space with periodic boundary conditions. Concretely, we consider a state space that is the flat tori in $n$ dimensions, denoted $\mathbb{T}^n = \mathbb{R}^n / \mathbb{Z}^n$. It is a quotient space resulting from identifying any point $x = (x^1, \dots, x^i, \dots, x^n)$ with $(x^1, \dots, x^i + 1, \dots, x^n)$ for all $i \in [n]$. The derivations for the SOC objective in $\mathbb{R}^n$ \eqref{eq:soc_cost} can directly be extended to this quotient space and also general Riemannian manifolds~\citep{de2022riemannian,thornton2022riemannian,huang2022riemannian}. Let us denote by $\pbasepbc$ the distribution of the SDE modeling the base process, \ie, \eqref{eq:controlled_sde} with $u_t$ = 0, that lives on $\mathbb{T}^n$. The RAM objective \eqref{eq:ram} requires computing $\nabla \log \pbasepbc_1(\cdot)$ as part of the terminal cost, and sampling $X_t$ from $\pbasepbc_{t|1}(\cdot | X_1)$. Simulating the base process in $\mathbb{T}^n$ produces a factorized wrapped Gaussian distribution,
\begin{align}\label{eq:pbase_pbc}
\pbasepbc_t(x) = \prod_{i=1}^n \sum_{k = -\infty}^\infty \pbase_t(x^i + k),
\end{align}
where $\pbasepbc$ is the distribution of the base process in $\R$. In practice, we can easily compute this up to high numerical precision by truncating the summation in \eqref{eq:pbase_pbc}.
Furthermore, we can sample from the backwards transition distribution independently for each dimension $i \in [n]$:
\begin{align}
    k^i \sim p(k^i), \qquad p(k^i) \propto \pbase_1(X_1^i + k^i), \qquad
    X_t^i = Y_t^i \bmod 1.0, \qquad Y_t^i \sim \pbase_{t|1}(Y_t^i | X_1^i + k^i).
\end{align}
In practice, we consider only values of $k$ within some truncated set $\{-k_{\text{trunc}},\dots, k_{\text{trunc}}\}$ with $k_{\text{trunc}}$ sufficiently large to cover the region where $\pbase_1(X_1^i + k)$ is practically non-zero. Details for the derivation can be found in \Cref{sec:flat_tori_derivations}.

\begin{table*}[t]
\caption{Results for the synthetic energy function experiments. We report a geometric $\mathcal{W}_2$ metric based on on~\citet{klein2024equivariant} and 1D energy histogram $E(\cdot)\,\mathcal{W}_2$ metric (visualized in~\Cref{fig:energy_histograms}) with respect to ground truth MCMC samples. We also report a path-measure ESS when applicable. See~\Cref{app:w2_metric} for more details. $^\dagger$The values reported are per sample (\ie, divided by the batch size) and according to the LJ-55 experiment hyperparameters.
}
\centering
\resizebox{\textwidth}{!}{%
\renewcommand{\arraystretch}{1.2}
\setlength{\tabcolsep}{3pt}
\begin{tabular}{@{} l ccc ccc ccc rr @{}}
\toprule
 & \multicolumn{3}{c}{DW-4 $(d=8)$}      & \multicolumn{3}{c}{LJ-13 $(d=39)$}    & \multicolumn{3}{c}{LJ-55 $(d=165)$} & 
\multirow{2}{*}{\small \shortstack{\# $E(\cdot)$ evals \\ per gradient \\ update$^\dagger$}}
&
\multirow{2}{*}{\small \shortstack{\# $u_\theta (\cdot)$ evals \\ per gradient \\ update$^\dagger$}}  
\\ \cmidrule(lr){2-4} \cmidrule(lr){5-7} \cmidrule(lr){8-10} 
Method & path-ESS $\uparrow$ & $\mathcal{W}_2$ $\downarrow$ & $E(\cdot)$ $\mathcal{W}_2$ $\downarrow$ 
& path-ESS $\uparrow$& $\mathcal{W}_2$ $\downarrow$ & $E(\cdot)$ $\mathcal{W}_2$ $\downarrow$  
& path-ESS $\uparrow$& $\mathcal{W}_2$ $\downarrow$ & $E(\cdot)$ $\mathcal{W}_2$ $\downarrow$ \\ \hline
PIS {\tiny\citep{zhang2022path}} &
0.462{\color{gray}\tiny$\pm$0.081} & 0.68{\color{gray}\tiny$\pm$0.23} & 0.65{\color{gray}\tiny$\pm$0.25} &
0.012{\color{gray}\tiny$\pm$0.011}& 1.93{\color{gray}\tiny$\pm$0.07} & 18.02{\color{gray}\tiny$\pm$1.12} &
0.001{\color{gray}\tiny$\pm$0.000}& 4.79{\color{gray}\tiny$\pm$0.45}  & 228.70{\color{gray}\tiny$\pm$131.27}  &
1 & 1000 \\
DDS {\tiny\citep{vargas2023denoising}} &
0.461{\color{gray}\tiny$\pm$0.076}& 0.92{\color{gray}\tiny$\pm$0.11} & 0.90{\color{gray}\tiny$\pm$0.37} &
0.010{\color{gray}\tiny$\pm$0.011}& 1.99{\color{gray}\tiny$\pm$0.13} & 24.61{\color{gray}\tiny$\pm$8.99} &
0.001{\color{gray}\tiny$\pm$0.000}& 4.60{\color{gray}\tiny$\pm$0.09}  & 173.09{\color{gray}\tiny$\pm$18.01}  &
1 & 1000 \\
LogVariance {\tiny\citep{richter2023improved}} &
0.025{\color{gray}\tiny$\pm$0.042} & 1.04{\color{gray}\tiny$\pm$0.29} & 1.89{\color{gray}\tiny$\pm$0.89} &
{\color{gray}---} & {\color{gray}---} & {\color{gray}---} &
{\color{gray}---} & {\color{gray}---}  & {\color{gray}---}  &
1 & 1000 \\
iDEM {\tiny\citep{akhounditerated}} &
{\color{gray}---}& 0.70{\color{gray}\tiny$\pm$0.06} & \cellhi 0.55{\color{gray}\tiny$\pm$0.14} &
{\color{gray}---} & \cellhi 1.61{\color{gray}\tiny$\pm$0.01} & 30.78{\color{gray}\tiny$\pm$24.46} &
{\color{gray}---}& 4.69{\color{gray}\tiny$\pm$1.52} & 93.53{\color{gray}\tiny$\pm$16.31} &
512 & \cellhi 3 \\
iDEM w/ 1 MC sample  &
{\color{gray}---}& 1.21{\color{gray}\tiny$\pm$0.02} & 2.70{\color{gray}\tiny$\pm$0.18} &
{\color{gray}---}& 2.03{\color{gray}\tiny$\pm$0.02} & 22.41{\color{gray}\tiny$\pm$0.18} &
{\color{gray}---}& 5.79{\color{gray}\tiny$\pm$1.60} & 1e32{\color{gray}\tiny$\pm$1e32} &
1 & \cellhi 3 \\
Adjoint Sampling w/o RP (\textbf{Ablation})&
0.448{\color{gray}\tiny$\pm$0.110}& 0.63{\color{gray}\tiny$\pm$0.11} & 1.03{\color{gray}\tiny$\pm$0.23} &
0.159{\color{gray}\tiny$\pm$0.068}& 1.68{\color{gray}\tiny$\pm$0.01} &  2.91{\color{gray}\tiny$\pm$1.39} &
\cellhi 0.094{\color{gray}\tiny$\pm$0.025}& \cellhi 4.50{\color{gray}\tiny$\pm$0.10} & 94.48{\color{gray}\tiny$\pm$76.12} &
\cellhi 0.002 & \cellhi 3 \\
Adjoint Sampling (\textbf{Ours})&
\cellhi 0.627{\color{gray}\tiny$\pm$0.037}& \cellhi 0.62{\color{gray}\tiny$\pm$0.06} & \cellhi 0.55{\color{gray}\tiny$\pm$0.12} &
\cellhi 0.220{\color{gray}\tiny$\pm$0.041}& 1.67{\color{gray}\tiny$\pm$0.01} & \cellhi 2.40{\color{gray}\tiny$\pm$1.25} &
0.066{\color{gray}\tiny$\pm$0.037}& \cellhi 4.50{\color{gray}\tiny$\pm$0.05} & \cellhi 58.04{\color{gray}\tiny$\pm$20.98} &
\cellhi 0.002 & \cellhi 3 \\ 
\bottomrule
\end{tabular}
}
\label{tab:synthetic}
\end{table*}

\section{Related Work}
\paragraph{Learning Augmented MCMC}
Markov Chain Monte Carlo (MCMC) and Sequential Monte Carlo (SMC) methods have been the standard for sampling from complex distributions using a well-designed Markov-chain. 
Due to prohibitively long mixing times and poor scaling in high-dimensions, existing work have combined MCMC and SMC techniques with deep learning.
\citet{albergo2019flow}, \citet{arbel2021annealed} and \citet{gabrie2022adaptive} learn better MCMC proposal distributions with variational inference via normalizing flows~\citep{chen2018neural}. \citet{matthews2022continual} proposed learning proposal distributions for improving SMC using stochastic normalizing flow~\citep{wu2020stochastic}. \citet{albergo2024nets} and \citet{holderrieth2025leaps} starts from MCMC procedures and optimize a learnable component by minimizing the Kolmogorov equations.

\paragraph{MCMC-reliant Diffusion Samplers} There are many works that learn diffusion processes but rely on an auxiliary sampling mechanism to obtain the correct signal for training.
More recent advances leverage score-based diffusion models \citep{song2019generative, ho2020denoising} in unnormalized sampling tasks to improve sample efficiency and scalability.
\citet{phillips2024particle} and \citet{de2024target} demonstrate the effectiveness of simple regression objectives, such as score matching, while
\citet{akhounditerated} propose iterated Denoising Energy Matching (iDEM), an offline, yet biased, algorithm for learning the score from a replay buffer, thereby overcoming scalability issues of simulation. \citet{phillips2024particle} and \citet{de2024target} use a similar loss to ours but relies on sampling from the target distribution, using sequential Monte Carlo (SMC) to obtain samples for training. 
In general, these approaches either require access to ground truth samples via an auxiliary sampling algorithm such as MCMC, or rely on importance-weighted estimation or resampling, all of which require extensive evaluation of the energy function and it is not clear whether these improve upon simple data-driven learning algorithms (\Cref{app:bridge_matching}).

\paragraph{SOC-based Diffusion Samplers}
In contrast, stochastic optimal control (SOC) based samplers reframe sampling tasks as an optimization problem.
\citet{zhang2022path} and \citet{vargas2023denoising}
showcase directly optimizing for controlled processes that match the desired target distribution.
This framework has been generalized in many ways ~\citet{berner2023optimal,richter2024improved,vargas2024transport,chen2024sequential}.
However, all of these methods are hindered by their computational requirements, including computationally expensive differentiation through the sampling procedure, computation of higher-order derivatives in constructing the training objectives, or the need for importance sampling (\ie multiple energy evaluations). 
Adjoint Sampling overcomes these challenges by being able to make significantly more gradient updates per generated sample and energy evaluation.

\paragraph{Off-policy Methods}
Off-policy methods \citep{malkin2023trajectory,richter2023improved,akhounditerated,hua2024simulation} are those which do not need to use samples from the current model. As such, off-policy methods do not inherently make use of the gradient of the energy function, as they do not differentiate through the model sample. To alleviate this, it is typical to parameterize the model using the gradient of the energy function, and it has been found that performance strongly relies on this trick~\citep{he2025no}. As our setting concerns computationally expensive energy functions, this parameterization is no longer feasible.
Furthermore, being off-policy does not imply being more efficient. Many off-policy methods still require the full trajectory to evaluate their loss, only being able to use sub-trajectories if the time-marginals are either additionally learned~\citep{bengio2023gflownet,lahlou2023theory,zhang2024diffusion} or prescribed~\citep{albergo2024nets,chen2024sequential,holderrieth2025leaps}. 
In contrast, Adjoint Sampling is an on-policy training method, explicitly uses the gradient of the energy function, but benefits significantly from requiring only $(X_t, X_1)$ pairs, being able to use a replay buffer and Reciprocal projections to train at a significantly reduced computational cost.

\paragraph{Molecule Conformer Generative Models}
To our knowledge, no deep learning-based sampling method has been applied at scale to generate molecular conformers directly from energy without data. 
Boltzmann generators use importance sampling to improve performance using energy information, but require offline molecular dynamics data to learn the initial models \citep{noe2019boltzmann}. Using an annealed importance sampling method, Boltzmann generators have been able to to sample small, single particle systems without data, but their results are not amortized and require many energy evaluations \citep{midgley2024se}.
Several flow-based generative modeling works have learned to generate conformers and Boltzmann distributions by training on ground truth data~\citep{kohler2020equivariant, jing2022torsional, xu2022geodiff, hassanflow, kohler2023rigid, klein2024timewarp, diez2024generation}. 
These works did not learn solely from existing energy models.

\section{Experiments}

\begin{figure}
    \centering
    \includegraphics[width=0.95\linewidth]{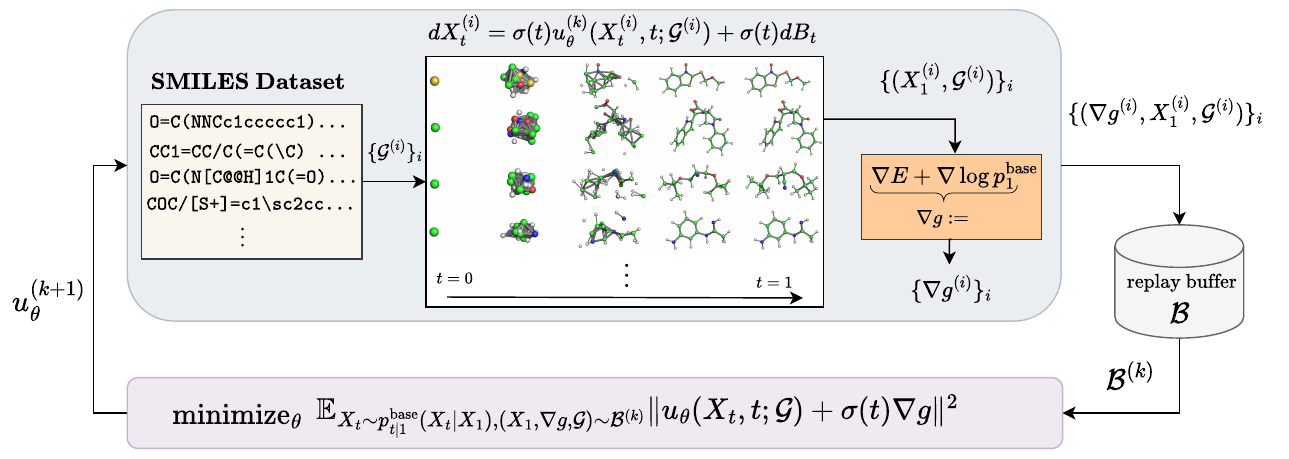}
    \caption{The iterations of Adjoint Sampling for amortized molecular conformer generation. SMILES strings condition the stochastic differential equation (SDE) on a specific molecular graph $\mathcal{G}^{(i)}$, where the final states and energy gradients and conditioning $(\nabla g^{(i)},X^{(i)}_1, \mathcal{G}^{(i)})$ are stored into a replay buffer. The model is trained by minimizing the RAM loss using the sample buffer, progressively transforming the samples into realistic molecular conformations.}
    \label{fig:enter-label}
\end{figure}

We now evaluate Adjoint Sampling on several energy functions over multi-particle systems. We compare to previous works~\citep{zhang2022path, akhounditerated} on synthetic energy benchmarks, whose energy is known analytically and are very cheap to evaluate. 
We then propose a challenging \emph{molecular conformer generation benchmark}, where one trains amortized models to sample conformers for a large dataset of organic molecules. For this task, we used the eSEN energy model~\citep{fu2025learning}.

\paragraph{Reciprocal Projection ablation}
We include an ablation of Adjoint Sampling, where instead of using the Reciprocal projection (RP) to sample $X_t$ given $X_1$, we simply store sample pairs $(X_1, X_t)$ in the buffer and train on the Adjoint Matching objective \eqref{eq:am}. We call this \emph{Adjoint Sampling w/o RP}, which helps demonstrate the effectiveness of the Reciprocal projection. Ablation results for conformer generation are differed to~\Cref{app:adjoint_ablation}~\Cref{tab:adjoint_sampling_ablation}.

\subsection{Synthetic Energy Functions}
Several synthetic energy functions of n-particle bodies proposed and benchmarked by previous works~\citep{kohler2020equivariant, midgley2022flow,klein2024equivariant, akhounditerated}. We consider three energy functions: A 2D 4-particle Double-Well Potential (DW-4), a 3D 13-particle Lennard-Jones potential (LJ-13) and finally a a 55-particle Lennard-Jones energy (LJ-55). See \Cref{app:synthetic_energy} for details.

\paragraph{Baselines}
Along with the Adjoint Sampling ablation, we compare against the most recent state-of-the-art sampler iDEM~\citep{akhounditerated} who claims to be first learned sampler to scale successfully to LJ-55. We also compare against PIS~\citep{zhang2022path} and DDS~\citep{vargas2023denoising}, being the most comparable method to Adjoint Sampling in its SOC / reverse KL formulation. Lastly, for completeness we present an offline variant of the PIS loss referred to as the LogVariance~\citep{richter2023improved}, which is actually gradient-free.
All methods use the EGNN~\citep{satorras2021n} architecture for all experiment settings, where PIS and DDS use fewer number of layers making back-propogation through the SDE tractable.

\begin{table*}[ht]
\caption{Recall and precision metrics for large scale amortized conformer generation. Coverage values are for thresholds of 1.25\AA. Standard deviations are computed across molecules in the test set. Results split by number of rotatable bonds can be found in \Cref{fig:recall_comparison_across_rotb}.}
\centering
\renewcommand{\arraystretch}{1.3}
\resizebox{\textwidth}{!}{%
\begin{tabular}{@{} l l cccc cccc }
    \toprule
    & & \multicolumn{4}{c}{{SPICE}} & \multicolumn{4}{c}{{GEOM-DRUGS}} \\
    \cmidrule(lr){3-6} \cmidrule(lr){7-10}
    & & \multicolumn{2}{c}{{Recall} } & \multicolumn{2}{c}{{Precision}} & \multicolumn{2}{c}{{Recall}} & \multicolumn{2}{c}{{Precision}} \\
    & Method & Cov. $\uparrow$ & AMR $\downarrow$ & Cov. $\uparrow$ & AMR $\downarrow$ & Cov. $\uparrow$ & AMR $\downarrow$ &
    Cov. $\uparrow$ & AMR $\downarrow$ \\
    \midrule
    & RDKit ETKDG 
& 72.74{\color{gray}\tiny$\pm$33.18} & 1.04{\color{gray}\tiny$\pm$0.52} & 69.68{\color{gray}\tiny$\pm$37.11} & 1.14{\color{gray}\tiny$\pm$0.64} & 63.51{\color{gray}\tiny$\pm$34.74} & 1.15{\color{gray}\tiny$\pm$0.61} & \cellhi 69.77{\color{gray}\tiny$\pm$38.23} & \cellhi 1.09{\color{gray}\tiny$\pm$0.66} \\
    & Torsional AdjSampling 
& 85.06{\color{gray}\tiny$\pm$24.61} &  0.86{\color{gray}\tiny$\pm$0.30} & \cellhi 70.42{\color{gray}\tiny$\pm$34.54} & \cellhi 1.06{\color{gray}\tiny$\pm$0.54} &  72.91{\color{gray}\tiny$\pm$31.17} & \cellhi 0.98{\color{gray}\tiny$\pm$0.40} & 67.85{\color{gray}\tiny$\pm$36.01} & \cellhi 1.09{\color{gray}\tiny$\pm$0.55} \\
    & Cartesian AdjSampling 
& 82.22{\color{gray}\tiny$\pm$25.72} & 0.96{\color{gray}\tiny$\pm$0.26} & 49.13{\color{gray}\tiny$\pm$33.01} & 1.26{\color{gray}\tiny$\pm$0.38} & 60.93{\color{gray}\tiny$\pm$35.15} & 1.20{\color{gray}\tiny$\pm$0.43} & 28.44{\color{gray}\tiny$\pm$27.77} & 1.86{\color{gray}\tiny$\pm$0.64} \\
    & Cartesian AdjSampling (+pretrain) 
& \cellhi 89.42{\color{gray}\tiny$\pm$17.48} & \cellhi 0.84{\color{gray}\tiny$\pm$0.24} & 65.93{\color{gray}\tiny$\pm$29.53} & 1.12{\color{gray}\tiny$\pm$0.34} & \cellhi 72.98{\color{gray}\tiny$\pm$30.82} & 1.02{\color{gray}\tiny$\pm$0.34} & 45.14{\color{gray}\tiny$\pm$31.46} & 1.47{\color{gray}\tiny$\pm$0.52} \\
    \midrule
    \parbox[t]{2mm}{\multirow{4}{*}{\rotatebox[origin=c]{90}{\color{cadetblue}w/ relaxation}}} 
        & RDKit ETKDG 
    & 81.61{\color{gray}\tiny$\pm$27.58} & 0.79{\color{gray}\tiny$\pm$0.44} & 74.51{\color{gray}\tiny$\pm$35.07} & 0.97{\color{gray}\tiny$\pm$0.64} & 71.72{\color{gray}\tiny$\pm$29.73} & 0.93{\color{gray}\tiny$\pm$0.53} & \cellhi 75.37{\color{gray}\tiny$\pm$32.76} & \cellhi 0.89{\color{gray}\tiny$\pm$0.60} \\
        & Torsional AdjSampling 
    & 88.25{\color{gray}\tiny$\pm$21.17} & 0.72{\color{gray}\tiny$\pm$0.33} & 74.66{\color{gray}\tiny$\pm$32.96} & 0.94{\color{gray}\tiny$\pm$0.59} & 76.62{\color{gray}\tiny$\pm$27.73} & 0.87{\color{gray}\tiny$\pm$0.40} & 71.85{\color{gray}\tiny$\pm$32.34} & 0.97{\color{gray}\tiny$\pm$0.54} \\
        & Cartesian AdjSampling 
    & 94.10{\color{gray}\tiny$\pm$15.67} & 0.68{\color{gray}\tiny$\pm$0.28} & 62.80{\color{gray}\tiny$\pm$29.63} & 1.02{\color{gray}\tiny$\pm$0.37} & 79.08{\color{gray}\tiny$\pm$29.44} & 0.89{\color{gray}\tiny$\pm$0.45} & 47.88{\color{gray}\tiny$\pm$30.92} & 1.40{\color{gray}\tiny$\pm$0.61} \\
        & Cartesian AdjSampling (+pretrain) 
    & \cellhi 96.65{\color{gray}\tiny$\pm$\phantom{0}7.51} & \cellhi 0.60{\color{gray}\tiny$\pm$0.23} & \cellhi 77.20{\color{gray}\tiny$\pm$25.76} & \cellhi 0.92{\color{gray}\tiny$\pm$0.35} & \cellhi 87.01{\color{gray}\tiny$\pm$22.79} & \cellhi 0.76{\color{gray}\tiny$\pm$0.34} & 61.96{\color{gray}\tiny$\pm$31.55} & 1.10{\color{gray}\tiny$\pm$0.46} \\
    \bottomrule
\end{tabular}
}
\label{tab:conformer_generation}
\end{table*}

\begin{figure*}[h]
\centering
\rotatebox[origin=l]{90}{\,Cartesian}
\includegraphics[width=0.16\linewidth]{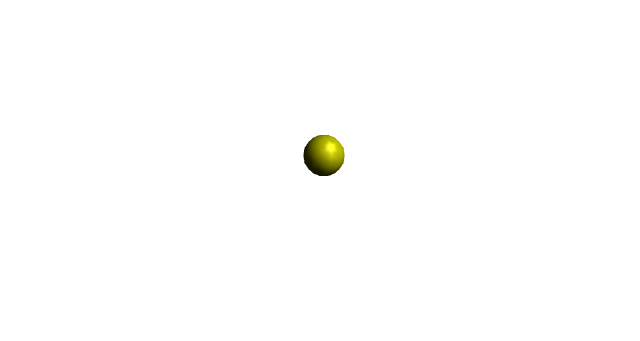}%
\includegraphics[width=0.16\linewidth]{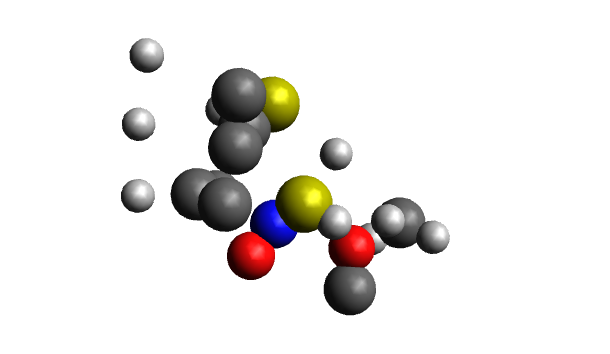}%
\includegraphics[width=0.16\linewidth]{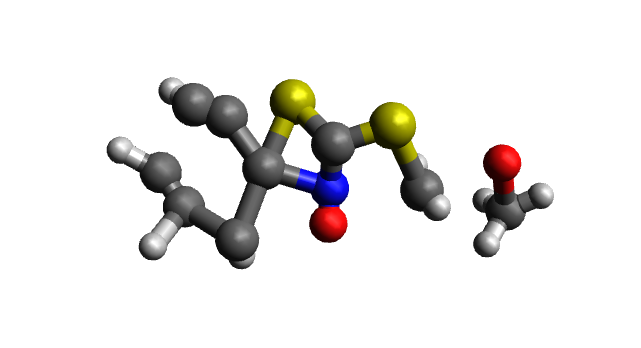}%
\includegraphics[width=0.16\linewidth]{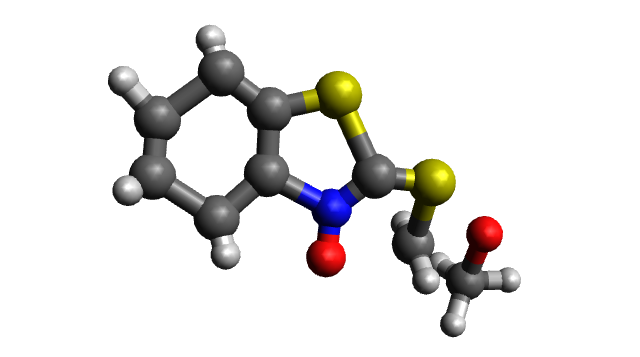}%
\includegraphics[width=0.16\linewidth]{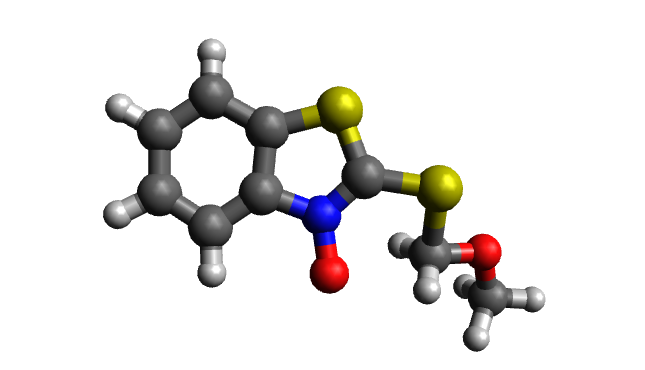}%
\includegraphics[width=0.16\linewidth]{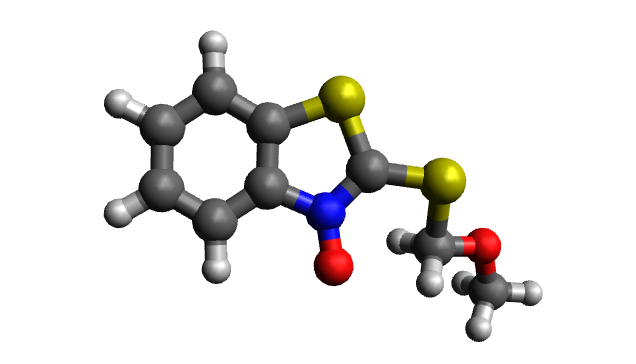}\\
\rotatebox[origin=l]{90}{\;Torsional}
\begin{subfigure}[t]{0.16\linewidth}
\includegraphics[width=\linewidth]{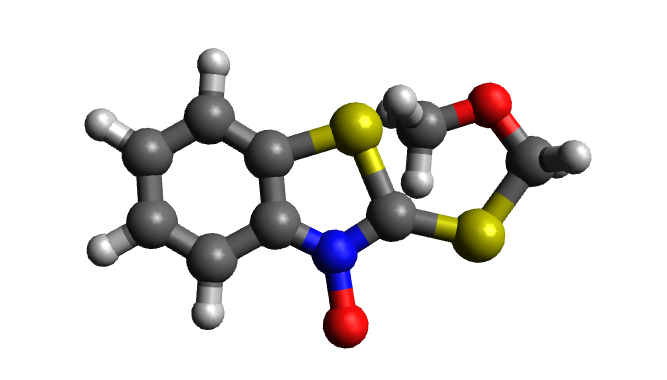}
\caption*{$X_0$}
\end{subfigure}%
\begin{subfigure}[t]{0.64\linewidth}
\includegraphics[width=0.25\linewidth]{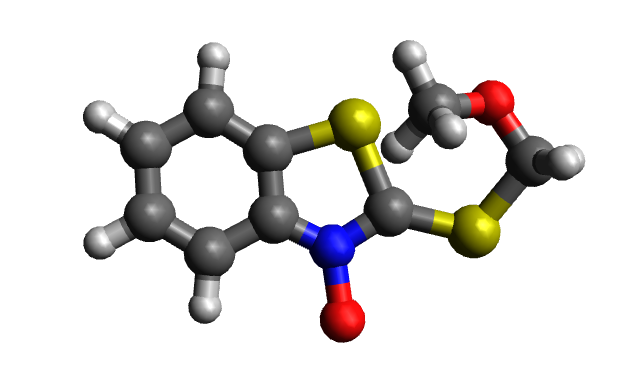}%
\includegraphics[width=0.25\linewidth]{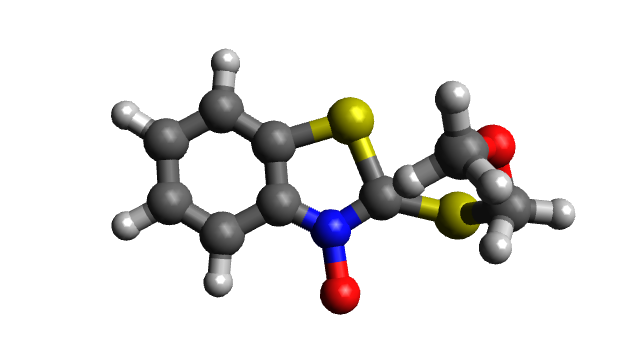}%
\includegraphics[width=0.25\linewidth]{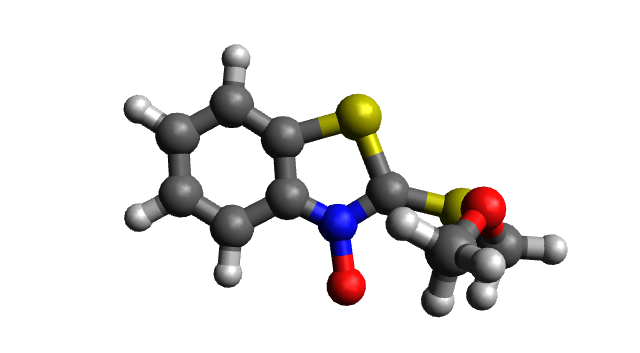}%
\includegraphics[width=0.25\linewidth]{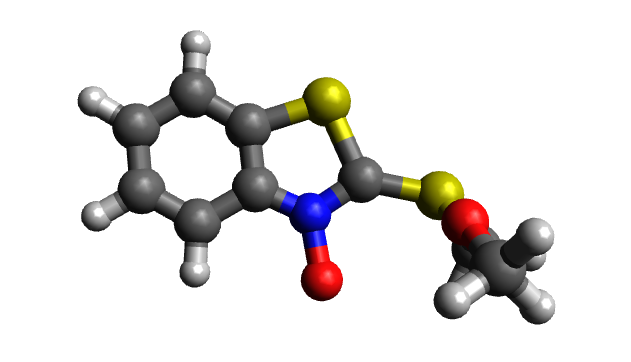}
\vspace{-1em}
\caption*{
\begin{tikzpicture}
  \draw[->] (-5,0) -- node[below] {Generation process} (5,0);
\end{tikzpicture}
}
\end{subfigure}%
\begin{subfigure}[t]{0.16\linewidth}
\includegraphics[width=\linewidth]{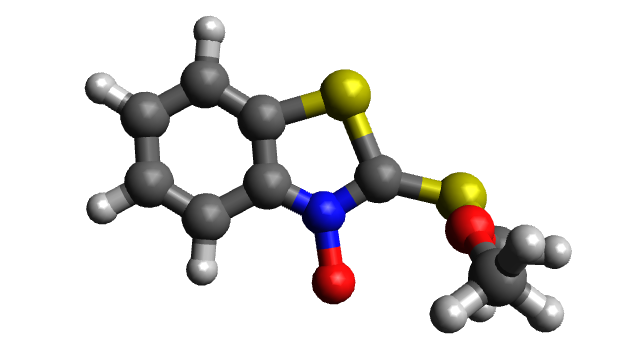}
\caption*{$X_1$}
\end{subfigure}
\caption{The figure depicts two sampled trajectories from trained Adjoint Sampling models that use either the cartesian or the torsional representations. They target conformations of the held-out SMILES string \texttt{COCSc1sc2ccccc2[n+]1[O-]}. The left frame $X_0$ comes from the initial Dirac distribution and the right frame $X_1$ is a sampled conformer.
}\label{fig:generation_process}
\end{figure*}
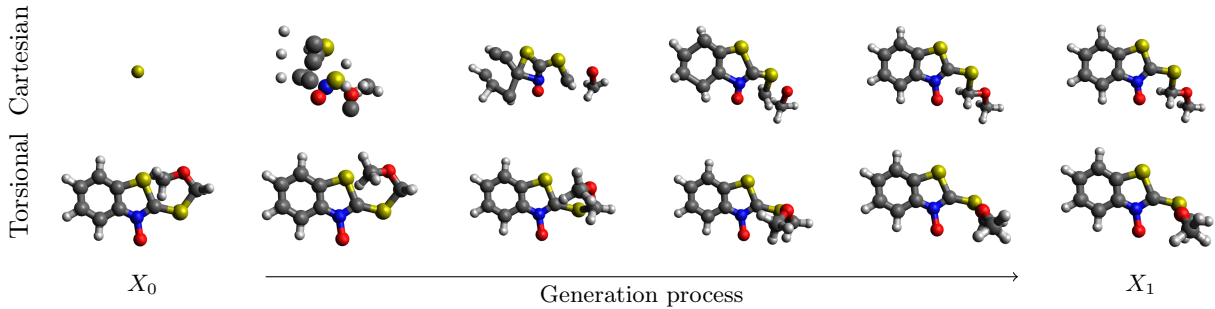

\begin{figure*}[ht]
\centering
    \begin{subfigure}[b]{0.5\linewidth}
    \centering
    \caption*{SPICE}
    \includegraphics[width=0.9\linewidth]{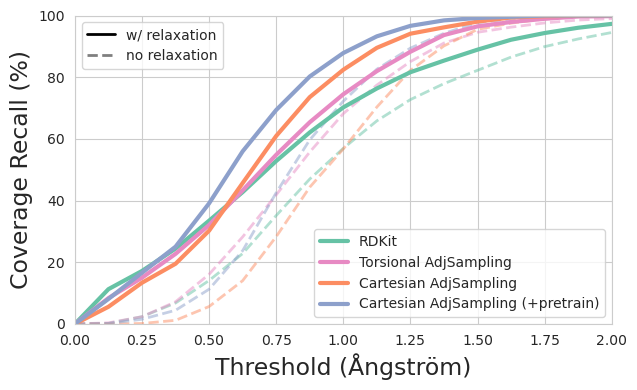}
    \end{subfigure}%
    \begin{subfigure}[b]{0.5\linewidth}
    \centering
    \caption*{GEOM-DRUGS}
    \includegraphics[width=0.9\linewidth]{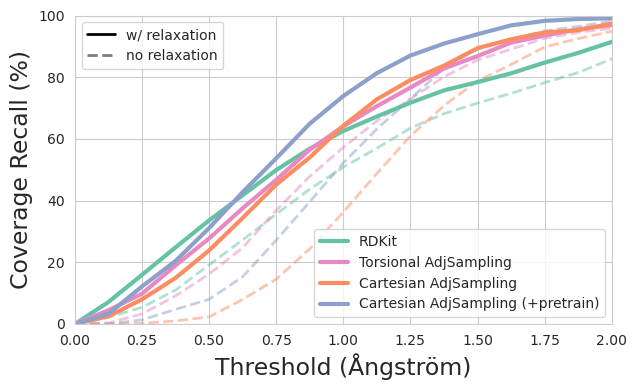}
    \end{subfigure}
    \vspace{-2em}
    \caption{Recall coverage versus RMSD threshold for Adjoint Sampling variants and RDKit. We show performance both with and without relaxation.}
\label{fig:recall_plot}
\end{figure*}
\vspace{-1em}
\paragraph{Evaluation}
In Table~\ref{tab:synthetic} we report our metrics including $\mathcal{W}_2$: the geometric 2-Wasserstien distance (taking into account symmetries by first optimizing a rigid transformation \citep{klein2024equivariant} between model generated samples and the ground truth MCMC data and $E(\cdot) \mathcal{W}_2$: The 1D $\mathcal{W}_2$ distance between the distribution of energy values produced by generated samples and ground-truth and path effective sample size (ESS): Roughly measuring the proportion of informative sample trajectories with respect to each algorithms target SDE process. Since iDEM is not built on a reverse KL objective, the path-ESS metric is not meaningful and is hence omitted. Please see~\Cref{app:w2_metric} for more details. To demonstrate the energy sample efficiency of Adjoint Sampling we report the number of energy function evaluations per gradient update.

\paragraph{Results}
It can be seen clearly in Table~\ref{tab:synthetic} that Adjoint Sampling is on par with iDEM in terms of particle $\mathcal{W}$-2 distances, however iDEM struggle to remain in low
energy regions as indicated by the much larger $E(\cdot) \mathcal{W}_2$ on the LJ-13 and LJ-55 experiments (visualized as energy histograms in~\Cref{fig:energy_histograms}). Moreover, \emph{iDEM uses $\mathcal{O}(100,000)$ more energy evaluations than Adjoint Sampling per gradient update} since it relies on Monte Carlo (MC) estimation of the ground truth score function. 
This is what makes scaling iDEM to more complex energy functions intractable. To test a more efficient version of iDEM, we also run iDEM using only a single MC sample for estimating their objective, however it proved to yield overall unstable metrics, most likely due to its biased MC score target estimator. 
In regards to SOC-based methods, PIS and DDS struggle to perform beyond LJ13, while the offline method LogVariance can not scale beyond DW-4. These results indicate that Adjoint Sampling is the preferred method for large and complex energy models.

\subsection{Sampling Conformers from an Energy Function}

We provide a new benchmark to encourage research into highly scalable sampling methods capable of finding \emph{conformers} from a molecular energy model. 
These are local minima on the molecule's potential energy surface, \ie a fixed molecular topology has a set of conformers determined by locally stable configurations of rotations around its bonds. Depending on the flexibility, a molecule can have one or many different conformers.

Sampling conformations based on quantum mechanical energy functions such as density functional theory (DFT) is an important but challenging task. Due to the high cost of DFT calculations, we use an energy model that was pretrained on the SPICE-MACE-OFF dataset~\citep{eastman2023spice, kovacs2023mace} called \emph{eSEN}~\citep{fu2025learning}. It has shown very high accuracy in predicting the DFT energy. eSEN is a transferable message passing neural network model that can predict the energy of a given molecular conformation for molecules containing \{H, C, N, O, F, P, S, Cl, Br, I\}. It featurizes a molecule with a sparse graph representation, encoding the chemical element in the node features and the Cartesian displacements in the edge features. The SPICE-MACE-OFF dataset comprehensively samples the energy surface in high and low energy regions, enabling eSEN to be generally applicable to many conformations of small molecules.

Using the Adjoint Sampling framework, we can sample atomic coordinates of a specific molecule by conditioning on its connectivity graph. These atomic coordinates are proposed conformations. The graph conditional method is appealing as it enables amortization, \ie learning a sampler for molecules consisting of arbitrary combinations of organic elements and bonds. We design two approaches:

\begin{enumerate}
\item \textbf{Cartesian Adjoint Sampling} (\Cref{subsec:cartesian_sampling}): We sample all 3D Cartesian coordinates for each atom, conditioned on elements (nodes) and bond type (edges).
\item \textbf{Torsional Adjoint Sampling} (\Cref{app:torsion_sampling}): We sample \emph{only} the torsion angles in the molecule, conditioned on elements, bond lengths, and bond angles (3-tuples).
\end{enumerate}
In addition, we evaluate a \emph{pretraining} step on RDKit samples using Bridge Matching~\citep{shi2024diffusion} which aids in initial exploration of low energy regions~(see \Cref{app:bridge_matching} for more details).
In our experiments, we scale the eSEN energy by the Boltzmann constant and a low temperature. In Cartesian Adjoint Sampling, we introduce additional energy terms to regularize bond lengths, ensuring we are targeting the correct molecule according to its connectivity. The terms do not appear in Torsional Adjoint Sampling; they are unnecessary since the torsional SDE cannot break bonds by design.

\paragraph{Datasets}
We consider two datasets of molecular structures: SPICE~\citep{eastman2023spice} and GEOM-DRUGS~\citep{axelrod2022geom}. 
SPICE enumerates over 23,000 diverse drug-like molecules and their properties; it is designed to be utilized for machine learning. All of our models are trained on SPICE; however, we utilize \emph{merely} the SMILES strings to extract the molecular topology. This is possible because Adjoint Sampling does not require any data during training, \ie, we do \emph{not} use the atomic configurations!
We evaluate generalization capabilities by computing precision and recall on a subset of GEOM-DRUGS, see \Cref{app:geom_drugs_info}. 
This is one the largest datasets and most relevant to pharmaceutical drug design.

\looseness=-1000
We split the SPICE molecules into a train set and a 80 molecule test set, allowing us to validate our sampler's ability to extrapolate to unseen molecules. We produce reference conformers for evaluation purposes using RDKit \citep{landrum2013rdkit}; CREST~\citep{pracht2020automated}, which is considered the gold standard in molecular chemistry; and ORCA~\citep{orca}.
Because this procedure is computational expensive, we do this only for the SPICE test set. GEOM-DRUGS provides conformers for us. See \Cref{app:data_prep} for details about the datasets.

\paragraph{Baselines} We compare our approach to RDKit ETKDG \citep{riniker2015better}, a chemistry-based conformer generation method. RDKit combines a rules-based distance matrix method, an iterative refinement algorithm, a 3D embedding, and experimental torsional preferences to generate conformers, making it a suitable domain-specific baseline.

\paragraph{Evaluation}
Our assessment focuses on reference conformer coverage \emph{recall}, which measures the percentage of reference CREST conformers recovered by the generated samples, and coverage \emph{precision}, which evaluates how often generated conformers closely match a low-energy reference structure. In particular, recall is a very important metric for molecular design because exploration is typically challenging and good mechanisms exist to filter candidates. The number of reference conformers varies significantly for molecules under consideration, ranging from tens to thousands.
We also report Absolute Mean RMSD (AMR) as an error metric, with lower values indicating better structural fidelity. See~\Cref{app:rmsd_metrics} for more details. Results are presented both with and without relaxation, where post-generation optimization helps refine molecular conformers in Table~\ref{tab:conformer_generation}. While samples without relaxation are a test of the sampling algorithm's performance, testing with relaxation is ultimately what brings the most value to computational chemistry being a standard refinement process (see~\Cref{app:conformers_and_torsions}).

\paragraph{Results} As seen in \Cref{tab:conformer_generation}, across both SPICE and GEOM-DRUGS, Adjoint Sampling (both Cartesian and Torsional) outperform the RDKit baseline in Recall. On generalizing to the GEOM-DRUGS dataset, we show worse precision than RDKit but significantly improved recall. As depicted in \Cref{fig:recall_plot}, Cartesian Adjoint Sampling outperforms RDKit on almost all threshold values, with and without relaxation. While relaxation improves coverage across all methods, it particularly benefits the performance of Adjoint Sampling, making the gap between RDKit-ETKDG even larger. We believe this comes from the fact that Adjoint Sampling naturally explores more of the configuration space by being initialized as a noisy stochastic process.

Without relaxation or pretraining, the Torsional AdjSampling variant performs significantly better than Cartesian Variant in precision and recall on both datasets. However, after pretraining Cartesian AdjSampling surpasses the Torsional variant and even further exceeds in performance after relaxation, (being pretrained or not).
This may be due to the Torsional representation evolving on a constrained space as seen in \Cref{fig:generation_process}, where the Cartesian representation evolves with the energy regularizer~(see \Cref{subsec:cartesian_sampling}) but is unconstrained. This state representation will affect how the controlled process explores the configuration space, suggesting that the unconstrained Cartesian representation samples are more diverse. 


Finally, in \Cref{fig:recall_comparison_across_rotb} we also compare across different number of rotatable bonds---the number of conformers increases drastically with the number of rotatable bonds (see \cref{fig:spice_drugs_confs})---where we see that the gap between our model and RDKit increases as the difficulty increases.

\section{Conclusion}
We introduce Adjoint Sampling, a highly scalable approach for learning diffusion-based samplers from energy functions. 
Using the Reciprocal Adjoint Matching objective and a replay buffer, our method enables efficient training by allowing many gradient updates with few energy evaluations and model samples. 
Based on continuous-time diffusion processes, our framework naturally integrates symmetries and periodic boundary conditions, making it effective for molecular modeling and conformer generation. 
We achieve state-of-the-art performance on synthetic energy functions and are the first to scale up to much more difficult conformer generation. We open-source our benchmarks to encourage further advancements in highly scalable sampling methods.\looseness=-1

\bibliographystyle{assets/plainnat}
\bibliography{paper}

\begin{thebibliography}{90}
\providecommand{\natexlab}[1]{#1}
\providecommand{\url}[1]{\texttt{#1}}
\expandafter\ifx\csname urlstyle\endcsname\relax
  \providecommand{\doi}[1]{doi: #1}\else
  \providecommand{\doi}{doi: \begingroup \urlstyle{rm}\Url}\fi

\bibitem[Akhound-Sadegh et~al.(2024)Akhound-Sadegh, Rector-Brooks, Bose, Mittal, Lemos, Liu, Sendera, Ravanbakhsh, Gidel, Bengio, et~al.]{akhounditerated}
Tara Akhound-Sadegh, Jarrid Rector-Brooks, Joey Bose, Sarthak Mittal, Pablo Lemos, Cheng-Hao Liu, Marcin Sendera, Siamak Ravanbakhsh, Gauthier Gidel, Yoshua Bengio, et~al.
\newblock Iterated denoising energy matching for sampling from boltzmann densities.
\newblock In \emph{Forty-first International Conference on Machine Learning}, 2024.

\bibitem[Albergo and Vanden-Eijnden(2024)]{albergo2024nets}
Michael~S Albergo and Eric Vanden-Eijnden.
\newblock Nets: A non-equilibrium transport sampler.
\newblock \emph{arXiv preprint arXiv:2410.02711}, 2024.

\bibitem[Albergo et~al.(2019)Albergo, Kanwar, and Shanahan]{albergo2019flow}
Michael~S Albergo, Gurtej Kanwar, and Phiala~E Shanahan.
\newblock Flow-based generative models for markov chain monte carlo in lattice field theory.
\newblock \emph{Physical Review D}, 100\penalty0 (3):\penalty0 034515, 2019.

\bibitem[Albergo et~al.(2023)Albergo, Boffi, and Vanden-Eijnden]{albergo2023stochastic}
Michael~S Albergo, Nicholas~M Boffi, and Eric Vanden-Eijnden.
\newblock Stochastic interpolants: A unifying framework for flows and diffusions.
\newblock \emph{arXiv preprint arXiv:2303.08797}, 2023.

\bibitem[Arbel et~al.(2021)Arbel, Matthews, and Doucet]{arbel2021annealed}
Michael Arbel, Alex Matthews, and Arnaud Doucet.
\newblock Annealed flow transport monte carlo.
\newblock In \emph{International Conference on Machine Learning}, pages 318--330. PMLR, 2021.

\bibitem[Axelrod and Gomez-Bombarelli(2022)]{axelrod2022geom}
Simon Axelrod and Rafael Gomez-Bombarelli.
\newblock Geom, energy-annotated molecular conformations for property prediction and molecular generation.
\newblock \emph{Scientific Data}, 9\penalty0 (1):\penalty0 185, 2022.

\bibitem[Batatia et~al.(2022)Batatia, Kovacs, Simm, Ortner, and Cs{\'a}nyi]{batatia2022mace}
Ilyes Batatia, David~P Kovacs, Gregor Simm, Christoph Ortner, and G{\'a}bor Cs{\'a}nyi.
\newblock Mace: Higher order equivariant message passing neural networks for fast and accurate force fields.
\newblock \emph{Advances in neural information processing systems}, 35:\penalty0 11423--11436, 2022.

\bibitem[Bellman(1957)]{bellman1957}
Richard Bellman.
\newblock \emph{Dynamic programming}.
\newblock Princeton Landmarks in Mathematics. Princeton University Press, Princeton, NJ, 2010., 1957.

\bibitem[Bengio et~al.(2023)Bengio, Lahlou, Deleu, Hu, Tiwari, and Bengio]{bengio2023gflownet}
Yoshua Bengio, Salem Lahlou, Tristan Deleu, Edward~J Hu, Mo~Tiwari, and Emmanuel Bengio.
\newblock Gflownet foundations.
\newblock \emph{Journal of Machine Learning Research}, 24\penalty0 (210):\penalty0 1--55, 2023.

\bibitem[Berner et~al.(2023)Berner, Richter, and Ullrich]{berner2023optimal}
Julius Berner, Lorenz Richter, and Karen Ullrich.
\newblock An optimal control perspective on diffusion-based generative modeling.
\newblock \emph{arXiv preprint arXiv:2211.01364}, 2023.

\bibitem[Bryson and Ho(1969)]{BrysonHo69}
A.~E. Bryson and Y.~C. Ho.
\newblock \emph{Applied Optimal Control}.
\newblock Blaisdell, 1969.

\bibitem[Chen et~al.(2024{\natexlab{a}})Chen, Richter, Berner, Blessing, Neumann, and Anandkumar]{chen2024sequential}
Junhua Chen, Lorenz Richter, Julius Berner, Denis Blessing, Gerhard Neumann, and Anima Anandkumar.
\newblock Sequential controlled langevin diffusions.
\newblock \emph{arXiv preprint arXiv:2412.07081}, 2024{\natexlab{a}}.

\bibitem[Chen et~al.(2018)Chen, Rubanova, Bettencourt, and Duvenaud]{chen2018neural}
Ricky T.~Q. Chen, Yulia Rubanova, Jesse Bettencourt, and David~K Duvenaud.
\newblock Neural ordinary differential equations.
\newblock In \emph{Advances in Neural Information Processing Systems}, volume~31. Curran Associates, Inc., 2018.

\bibitem[Chen et~al.(2024{\natexlab{b}})Chen, Goldstein, Hua, Albergo, Boffi, and Vanden-Eijnden]{chen2024probabilistic}
Yifan Chen, Mark Goldstein, Mengjian Hua, Michael~S Albergo, Nicholas~M Boffi, and Eric Vanden-Eijnden.
\newblock Probabilistic forecasting with stochastic interpolants and follmer processes.
\newblock \emph{arXiv preprint arXiv:2403.13724}, 2024{\natexlab{b}}.

\bibitem[Chen et~al.(2016)Chen, Georgiou, and Pavon]{chen2016relation}
Yongxin Chen, Tryphon~T Georgiou, and Michele Pavon.
\newblock On the relation between optimal transport and schr{\"o}dinger bridges: A stochastic control viewpoint.
\newblock \emph{Journal of Optimization Theory and Applications}, 169:\penalty0 671--691, 2016.

\bibitem[Dai~Pra(1991)]{daipra91}
Paolo Dai~Pra.
\newblock A stochastic control approach to reciprocal diffusion processes.
\newblock \emph{Applied mathematics and Optimization}, 23\penalty0 (1):\penalty0 313--329, 1991.

\bibitem[De~Bortoli et~al.(2022)De~Bortoli, Mathieu, Hutchinson, Thornton, Teh, and Doucet]{de2022riemannian}
Valentin De~Bortoli, Emile Mathieu, Michael Hutchinson, James Thornton, Yee~Whye Teh, and Arnaud Doucet.
\newblock Riemannian score-based generative modelling.
\newblock \emph{Advances in Neural Information Processing Systems}, 35:\penalty0 2406--2422, 2022.

\bibitem[De~Bortoli et~al.(2024)De~Bortoli, Hutchinson, Wirnsberger, and Doucet]{de2024target}
Valentin De~Bortoli, Michael Hutchinson, Peter Wirnsberger, and Arnaud Doucet.
\newblock Target score matching.
\newblock \emph{arXiv preprint arXiv:2402.08667}, 2024.

\bibitem[Del~Moral et~al.(2006)Del~Moral, Doucet, and Jasra]{del2006sequential}
Pierre Del~Moral, Arnaud Doucet, and Ajay Jasra.
\newblock Sequential monte carlo samplers.
\newblock \emph{Journal of the Royal Statistical Society Series B: Statistical Methodology}, 68\penalty0 (3):\penalty0 411--436, 2006.

\bibitem[Diez et~al.(2024)Diez, Atance, Engkvist, and Olsson]{diez2024generation}
Juan~Viguera Diez, Sara~Romeo Atance, Ola Engkvist, and Simon Olsson.
\newblock Generation of conformational ensembles of small molecules via surrogate model-assisted molecular dynamics.
\newblock \emph{Machine Learning: Science and Technology}, 5\penalty0 (2):\penalty0 025010, 2024.

\bibitem[Domingo-Enrich et~al.(2024)Domingo-Enrich, Drozdzal, Karrer, and Chen]{domingoenrich2024adjoint}
Carles Domingo-Enrich, Michal Drozdzal, Brian Karrer, and Ricky T.~Q. Chen.
\newblock Adjoint matching: Fine-tuning flow and diffusion generative models with memoryless stochastic optimal control, 2024.
\newblock \url{https://arxiv.org/abs/2409.08861}.

\bibitem[Eastman et~al.(2023)Eastman, Behara, Dotson, Galvelis, Herr, Horton, Mao, Chodera, Pritchard, Wang, et~al.]{eastman2023spice}
Peter Eastman, Pavan~Kumar Behara, David~L Dotson, Raimondas Galvelis, John~E Herr, Josh~T Horton, Yuezhi Mao, John~D Chodera, Benjamin~P Pritchard, Yuanqing Wang, et~al.
\newblock Spice, a dataset of drug-like molecules and peptides for training machine learning potentials.
\newblock \emph{Scientific Data}, 10\penalty0 (1):\penalty0 11, 2023.

\bibitem[Fleming and Rishel(2012)]{fleming2012deterministic}
W.H. Fleming and R.W. Rishel.
\newblock \emph{Deterministic and Stochastic Optimal Control}.
\newblock Stochastic Modelling and Applied Probability. Springer New York, 2012.

\bibitem[F{\"o}llmer(2005)]{follmer2005entropy}
Hans F{\"o}llmer.
\newblock An entropy approach to the time reversal of diffusion processes.
\newblock In \emph{Stochastic Differential Systems Filtering and Control: Proceedings of the IFIP-WG 7/1 Working Conference Marseille-Luminy, France, March 12--17, 1984}, pages 156--163. Springer, 2005.

\bibitem[Friede et~al.(2024)Friede, H{\"o}lzer, Ehlert, and Grimme]{friede2024dxtb}
Marvin Friede, Christian H{\"o}lzer, Sebastian Ehlert, and Stefan Grimme.
\newblock dxtb—an efficient and fully differentiable framework for extended tight-binding.
\newblock \emph{The Journal of Chemical Physics}, 161\penalty0 (6), 2024.

\bibitem[Fu et~al.(2025)Fu, Wood, Barroso-Luque, Levine, Gao, Dzamba, and Zitnick]{fu2025learning}
Xiang Fu, Brandon~M Wood, Luis Barroso-Luque, Daniel~S Levine, Meng Gao, Misko Dzamba, and C~Lawrence Zitnick.
\newblock Learning smooth and expressive interatomic potentials for physical property prediction.
\newblock \emph{arXiv preprint arXiv:2502.12147}, 2025.

\bibitem[Gabri{\'{e} } et~al.(2022)Gabri{\'{e} }, Rotskoff, and Vanden-Eijnden]{gabrie2022adaptive}
Marylou Gabri{\'{e} }, Grant~M. Rotskoff, and Eric Vanden-Eijnden.
\newblock Adaptive monte carlo augmented with normalizing flows.
\newblock \emph{Proceedings of the National Academy of Sciences}, 119\penalty0 (10), mar 2022.

\bibitem[Ganea et~al.(2021)Ganea, Pattanaik, Coley, Barzilay, Jensen, Green, and Jaakkola]{ganea2021geomol}
Octavian Ganea, Lagnajit Pattanaik, Connor Coley, Regina Barzilay, Klavs Jensen, William Green, and Tommi Jaakkola.
\newblock Geomol: Torsional geometric generation of molecular 3d conformer ensembles.
\newblock \emph{Advances in Neural Information Processing Systems}, 34:\penalty0 13757--13769, 2021.

\bibitem[Geiger and Smidt(2022)]{geiger2022e3nn}
Mario Geiger and Tess Smidt.
\newblock e3nn: Euclidean neural networks, 2022.
\newblock \url{https://arxiv.org/abs/2207.09453}.

\bibitem[Geiger et~al.(2022)Geiger, Smidt, M., Miller, Boomsma, Dice, Lapchevskyi, Weiler, Tyszkiewicz, Batzner, Madisetti, Uhrin, Frellsen, Jung, Sanborn, Wen, Rackers, Rød, and Bailey]{e3nn_software}
Mario Geiger, Tess Smidt, Alby M., Benjamin~Kurt Miller, Wouter Boomsma, Bradley Dice, Kostiantyn Lapchevskyi, Maurice Weiler, Michał Tyszkiewicz, Simon Batzner, Dylan Madisetti, Martin Uhrin, Jes Frellsen, Nuri Jung, Sophia Sanborn, Mingjian Wen, Josh Rackers, Marcel Rød, and Michael Bailey.
\newblock Euclidean neural networks: e3nn, April 2022.
\newblock \url{https://doi.org/10.5281/zenodo.6459381}.

\bibitem[Grimme et~al.(2017)Grimme, Bannwarth, and Shushkov]{grimme2017robust}
Stefan Grimme, Christoph Bannwarth, and Philip Shushkov.
\newblock A robust and accurate tight-binding quantum chemical method for structures, vibrational frequencies, and noncovalent interactions of large molecular systems parametrized for all spd-block elements (z= 1--86).
\newblock \emph{Journal of chemical theory and computation}, 13\penalty0 (5):\penalty0 1989--2009, 2017.

\bibitem[Hassan et~al.(2024)Hassan, Shenoy, Lee, Stark, Thaler, and Beaini]{hassanflow}
Majdi Hassan, Nikhil Shenoy, Jungyoon Lee, Hannes Stark, Stephan Thaler, and Dominique Beaini.
\newblock Et-flow: Equivariant flow-matching for molecular conformer generation.
\newblock In \emph{The Thirty-eighth Annual Conference on Neural Information Processing Systems}, 2024.

\bibitem[He et~al.(2025)He, Du, Vargas, Zhang, Padhy, OuYang, Gomes, and Hern{\'a}ndez-Lobato]{he2025no}
Jiajun He, Yuanqi Du, Francisco Vargas, Dinghuai Zhang, Shreyas Padhy, RuiKang OuYang, Carla Gomes, and Jos{\'e}~Miguel Hern{\'a}ndez-Lobato.
\newblock No trick, no treat: Pursuits and challenges towards simulation-free training of neural samplers.
\newblock \emph{arXiv preprint arXiv:2502.06685}, 2025.

\bibitem[Ho et~al.(2020)Ho, Jain, and Abbeel]{ho2020denoising}
Jonathan Ho, Ajay Jain, and Pieter Abbeel.
\newblock Denoising diffusion probabilistic models.
\newblock In \emph{Advances in Neural Information Processing Systems}, volume~33. Curran Associates, Inc., 2020.

\bibitem[Holderrieth et~al.(2025)Holderrieth, Albergo, and Jaakkola]{holderrieth2025leaps}
Peter Holderrieth, Michael~S Albergo, and Tommi Jaakkola.
\newblock {LEAPS}: A discrete neural sampler via locally equivariant networks.
\newblock \emph{arXiv preprint arXiv:2502.10843}, 2025.

\bibitem[Hua et~al.(2024)Hua, Lauri{\`e}re, and Vanden-Eijnden]{hua2024simulation}
Mengjian Hua, Matthieu Lauri{\`e}re, and Eric Vanden-Eijnden.
\newblock A simulation-free deep learning approach to stochastic optimal control.
\newblock \emph{arXiv preprint arXiv:2410.05163}, 2024.

\bibitem[Huang et~al.(2022)Huang, Aghajohari, Bose, Panangaden, and Courville]{huang2022riemannian}
Chin-Wei Huang, Milad Aghajohari, Joey Bose, Prakash Panangaden, and Aaron~C Courville.
\newblock Riemannian diffusion models.
\newblock \emph{Advances in Neural Information Processing Systems}, 35:\penalty0 2750--2761, 2022.

\bibitem[Jing et~al.(2022)Jing, Corso, Chang, Barzilay, and Jaakkola]{jing2022torsional}
Bowen Jing, Gabriele Corso, Jeffrey Chang, Regina Barzilay, and Tommi Jaakkola.
\newblock Torsional diffusion for molecular conformer generation.
\newblock \emph{Advances in Neural Information Processing Systems}, 35:\penalty0 24240--24253, 2022.

\bibitem[Kappen(2005)]{kappen2005path}
H~J Kappen.
\newblock Path integrals and symmetry breaking for optimal control theory.
\newblock \emph{Journal of Statistical Mechanics: Theory and Experiment}, 2005\penalty0 (11), nov 2005.

\bibitem[Karras et~al.(2022)Karras, Aittala, Aila, and Laine]{karras2022elucidating}
Tero Karras, Miika Aittala, Timo Aila, and Samuli Laine.
\newblock Elucidating the design space of diffusion-based generative models.
\newblock \emph{Advances in neural information processing systems}, 35:\penalty0 26565--26577, 2022.

\bibitem[Klein et~al.(2024{\natexlab{a}})Klein, Foong, Fjelde, Mlodozeniec, Brockschmidt, Nowozin, No{\'e}, and Tomioka]{klein2024timewarp}
Leon Klein, Andrew Foong, Tor Fjelde, Bruno Mlodozeniec, Marc Brockschmidt, Sebastian Nowozin, Frank No{\'e}, and Ryota Tomioka.
\newblock Timewarp: Transferable acceleration of molecular dynamics by learning time-coarsened dynamics.
\newblock \emph{Advances in Neural Information Processing Systems}, 36, 2024{\natexlab{a}}.

\bibitem[Klein et~al.(2024{\natexlab{b}})Klein, Kr{\"a}mer, and No{\'e}]{klein2024equivariant}
Leon Klein, Andreas Kr{\"a}mer, and Frank No{\'e}.
\newblock Equivariant flow matching.
\newblock \emph{Advances in Neural Information Processing Systems}, 36, 2024{\natexlab{b}}.

\bibitem[K{\"o}hler et~al.(2020)K{\"o}hler, Klein, and No{\'e}]{kohler2020equivariant}
Jonas K{\"o}hler, Leon Klein, and Frank No{\'e}.
\newblock Equivariant flows: exact likelihood generative learning for symmetric densities.
\newblock In \emph{International conference on machine learning}, pages 5361--5370. PMLR, 2020.

\bibitem[K{\"o}hler et~al.(2023)K{\"o}hler, Invernizzi, De~Haan, and No{\'e}]{kohler2023rigid}
Jonas K{\"o}hler, Michele Invernizzi, Pim De~Haan, and Frank No{\'e}.
\newblock Rigid body flows for sampling molecular crystal structures.
\newblock In \emph{International Conference on Machine Learning}, pages 17301--17326. PMLR, 2023.

\bibitem[Kondor et~al.(2018)Kondor, Lin, and Trivedi]{kondor2018clebschgordannets}
Risi Kondor, Zhen Lin, and Shubhendu Trivedi.
\newblock Clebsch-gordan nets: a fully fourier space spherical convolutional neural network, 2018.
\newblock \url{https://arxiv.org/abs/1806.09231}.

\bibitem[Kov{\'a}cs et~al.(2023)Kov{\'a}cs, Moore, Browning, Batatia, Horton, Kapil, Witt, Magd{\u{a}}u, Cole, and Cs{\'a}nyi]{kovacs2023mace}
D{\'a}vid~P{\'e}ter Kov{\'a}cs, J~Harry Moore, Nicholas~J Browning, Ilyes Batatia, Joshua~T Horton, Venkat Kapil, William~C Witt, Ioan-Bogdan Magd{\u{a}}u, Daniel~J Cole, and G{\'a}bor Cs{\'a}nyi.
\newblock Mace-off23: Transferable machine learning force fields for organic molecules.
\newblock \emph{arXiv preprint arXiv:2312.15211}, 2023.

\bibitem[Lahlou et~al.(2023)Lahlou, Deleu, Lemos, Zhang, Volokhova, Hern{\'a}ndez-Garc{\i}a, Ezzine, Bengio, and Malkin]{lahlou2023theory}
Salem Lahlou, Tristan Deleu, Pablo Lemos, Dinghuai Zhang, Alexandra Volokhova, Alex Hern{\'a}ndez-Garc{\i}a, L{\'e}na~N{\'e}hale Ezzine, Yoshua Bengio, and Nikolay Malkin.
\newblock A theory of continuous generative flow networks.
\newblock In \emph{International Conference on Machine Learning}, pages 18269--18300. PMLR, 2023.

\bibitem[Landrum(2013)]{landrum2013rdkit}
Greg Landrum.
\newblock Rdkit documentation.
\newblock \emph{Release}, 1\penalty0 (1-79):\penalty0 4, 2013.

\bibitem[Lipman et~al.(2023)Lipman, Chen, Ben-Hamu, Nickel, and Le]{lipman2023flow}
Yaron Lipman, Ricky T.~Q. Chen, Heli Ben-Hamu, Maximilian Nickel, and Matthew Le.
\newblock Flow matching for generative modeling.
\newblock In \emph{The Eleventh International Conference on Learning Representations}, 2023.

\bibitem[Liu et~al.(2023)Liu, Vahdat, Huang, Theodorou, Nie, and Anandkumar]{liu20232}
Guan-Horng Liu, Arash Vahdat, De-An Huang, Evangelos~A Theodorou, Weili Nie, and Anima Anandkumar.
\newblock I$^{2}$sb: Image-to-image schr$\backslash$" odinger bridge.
\newblock \emph{arXiv preprint arXiv:2302.05872}, 2023.

\bibitem[Liu et~al.(2024)Liu, Lipman, Nickel, Karrer, Theodorou, and Chen]{liu2023generalized}
Guan-Horng Liu, Yaron Lipman, Maximilian Nickel, Brian Karrer, Evangelos Theodorou, and Ricky T.~Q. Chen.
\newblock Generalized schr\"odinger bridge matching.
\newblock In \emph{The Twelfth International Conference on Learning Representations}, 2024.

\bibitem[Ma et~al.(2024)Ma, Goldstein, Albergo, Boffi, Vanden-Eijnden, and Xie]{ma2024sit}
Nanye Ma, Mark Goldstein, Michael~S Albergo, Nicholas~M Boffi, Eric Vanden-Eijnden, and Saining Xie.
\newblock Sit: Exploring flow and diffusion-based generative models with scalable interpolant transformers.
\newblock In \emph{European Conference on Computer Vision}, pages 23--40. Springer, 2024.

\bibitem[Malkin et~al.(2023)Malkin, Jain, Bengio, Sun, and Bengio]{malkin2023trajectory}
Nikolay Malkin, Moksh Jain, Emmanuel Bengio, Chen Sun, and Yoshua Bengio.
\newblock Trajectory balance: Improved credit assignment in gflownets.
\newblock \emph{arXiv preprint arXiv:2201.13259}, 2023.

\bibitem[Matthews et~al.(2022)Matthews, Arbel, Rezende, and Doucet]{matthews2022continual}
Alex Matthews, Michael Arbel, Danilo~Jimenez Rezende, and Arnaud Doucet.
\newblock Continual repeated annealed flow transport monte carlo.
\newblock In \emph{International Conference on Machine Learning}, pages 15196--15219. PMLR, 2022.

\bibitem[Midgley et~al.(2024)Midgley, Stimper, Antor{\'a}n, Mathieu, Sch{\"o}lkopf, and Hern{\'a}ndez-Lobato]{midgley2024se}
Laurence Midgley, Vincent Stimper, Javier Antor{\'a}n, Emile Mathieu, Bernhard Sch{\"o}lkopf, and Jos{\'e}~Miguel Hern{\'a}ndez-Lobato.
\newblock {SE}(3) equivariant augmented coupling flows.
\newblock \emph{Advances in Neural Information Processing Systems}, 36, 2024.

\bibitem[Midgley et~al.(2023)Midgley, Stimper, Simm, Sch{\"o}lkopf, and Hern{\'a}ndez-Lobato]{midgley2022flow}
Laurence~Illing Midgley, Vincent Stimper, Gregor~NC Simm, Bernhard Sch{\"o}lkopf, and Jos{\'e}~Miguel Hern{\'a}ndez-Lobato.
\newblock Flow annealed importance sampling bootstrap.
\newblock In \emph{The Twelfth International Conference on Learning Representations: ICLR 2024}, 2023.

\bibitem[Miller et~al.(2020)Miller, Geiger, Smidt, and No{\'e}]{miller2020relevance}
Benjamin~Kurt Miller, Mario Geiger, Tess~E Smidt, and Frank No{\'e}.
\newblock Relevance of rotationally equivariant convolutions for predicting molecular properties.
\newblock \emph{arXiv preprint arXiv:2008.08461}, 2020.

\bibitem[Neal(2001)]{neal2001annealed}
Radford~M Neal.
\newblock Annealed importance sampling.
\newblock \emph{Statistics and computing}, 11:\penalty0 125--139, 2001.

\bibitem[Neal et~al.(2011)]{neal2011mcmc}
Radford~M Neal et~al.
\newblock {{MCMC}} using {{Hamiltonian}} dynamics.
\newblock \emph{Handbook of {{Markov}} {{Chain}} {{Monte}} {{Carlo}}}, 2\penalty0 (11):\penalty0 2, 2011.

\bibitem[Neese(2012)]{orca}
F.~Neese.
\newblock The orca program system.
\newblock \emph{WIRES Comput. Molec. Sci.}, 2\penalty0 (1):\penalty0 73--78, 2012.
\newblock \doi{10.1002/wcms.81}.

\bibitem[No{\'e} et~al.(2019)No{\'e}, Olsson, K{\"o}hler, and Wu]{noe2019boltzmann}
Frank No{\'e}, Simon Olsson, Jonas K{\"o}hler, and Hao Wu.
\newblock Boltzmann generators: Sampling equilibrium states of many-body systems with deep learning.
\newblock \emph{Science}, 365\penalty0 (6457):\penalty0 eaaw1147, 2019.

\bibitem[Pavon(1989)]{pavon1989stochastic}
Michele Pavon.
\newblock Stochastic control and nonequilibrium thermodynamical systems.
\newblock \emph{Applied Mathematics and Optimization}, 19:\penalty0 187--202, 1989.

\bibitem[Peluchetti(2023{\natexlab{a}})]{peluchetti2023diffusion}
Stefano Peluchetti.
\newblock Diffusion bridge mixture transports, schr{\"o}dinger bridge problems and generative modeling.
\newblock \emph{Journal of Machine Learning Research}, 24\penalty0 (374):\penalty0 1--51, 2023{\natexlab{a}}.

\bibitem[Peluchetti(2023{\natexlab{b}})]{peluchetti2023non}
Stefano Peluchetti.
\newblock Non-denoising forward-time diffusions.
\newblock \emph{arXiv preprint arXiv:2312.14589}, 2023{\natexlab{b}}.

\bibitem[Phillips et~al.(2024)Phillips, Dau, Hutchinson, De~Bortoli, Deligiannidis, and Doucet]{phillips2024particle}
Angus Phillips, Hai-Dang Dau, Michael~John Hutchinson, Valentin De~Bortoli, George Deligiannidis, and Arnaud Doucet.
\newblock Particle denoising diffusion sampler.
\newblock \emph{arXiv preprint arXiv:2402.06320}, 2024.

\bibitem[Pracht et~al.(2020)Pracht, Bohle, and Grimme]{pracht2020automated}
Philipp Pracht, Fabian Bohle, and Stefan Grimme.
\newblock Automated exploration of the low-energy chemical space with fast quantum chemical methods.
\newblock \emph{Physical Chemistry Chemical Physics}, 22\penalty0 (14):\penalty0 7169--7192, 2020.

\bibitem[Pracht et~al.(2024)Pracht, Grimme, Bannwarth, Bohle, Ehlert, Feldmann, Gorges, M{\"u}ller, Neudecker, Plett, et~al.]{pracht2024crest}
Philipp Pracht, Stefan Grimme, Christoph Bannwarth, Fabian Bohle, Sebastian Ehlert, Gereon Feldmann, Johannes Gorges, Marcel M{\"u}ller, Tim Neudecker, Christoph Plett, et~al.
\newblock Crest—a program for the exploration of low-energy molecular chemical space.
\newblock \emph{The Journal of Chemical Physics}, 160\penalty0 (11), 2024.

\bibitem[Protter and Protter(2005)]{protter2005stochastic}
Philip~E Protter and Philip~E Protter.
\newblock \emph{Stochastic differential equations}.
\newblock Springer, 2005.

\bibitem[Rezende and Mohamed(2015)]{rezende2015variational}
Danilo Rezende and Shakir Mohamed.
\newblock Variational inference with normalizing flows.
\newblock In \emph{Proceedings of the 32nd International Conference on Machine Learning}, 2015.

\bibitem[Richter and Berner(2023)]{richter2023improved}
Lorenz Richter and Julius Berner.
\newblock Improved sampling via learned diffusions.
\newblock \emph{arXiv preprint arXiv:2307.01198}, 2023.

\bibitem[Richter and Berner(2024)]{richter2024improved}
Lorenz Richter and Julius Berner.
\newblock Improved sampling via learned diffusions.
\newblock In \emph{The Twelfth International Conference on Learning Representations}, 2024.

\bibitem[Riniker and Landrum(2015)]{riniker2015better}
Sereina Riniker and Gregory~A Landrum.
\newblock Better informed distance geometry: using what we know to improve conformation generation.
\newblock \emph{Journal of chemical information and modeling}, 55\penalty0 (12):\penalty0 2562--2574, 2015.

\bibitem[Satorras et~al.(2021)Satorras, Hoogeboom, and Welling]{satorras2021n}
V{\i}ctor~Garcia Satorras, Emiel Hoogeboom, and Max Welling.
\newblock E (n) equivariant graph neural networks.
\newblock In \emph{International conference on machine learning}, pages 9323--9332. PMLR, 2021.

\bibitem[Sethi(2018)]{sethi2018optimal}
S.P. Sethi.
\newblock \emph{Optimal Control Theory: Applications to Management Science and Economics}.
\newblock Springer International Publishing, 2018.

\bibitem[Shi et~al.(2023)Shi, De~Bortoli, Campbell, and Doucet]{shi2024diffusion}
Yuyang Shi, Valentin De~Bortoli, Andrew Campbell, and Arnaud Doucet.
\newblock Diffusion schr{\"o}dinger bridge matching.
\newblock \emph{Advances in Neural Information Processing Systems}, 36, 2023.

\bibitem[Somnath et~al.(2023)Somnath, Pariset, Hsieh, Martinez, Krause, and Bunne]{somnath2023aligned}
Vignesh~Ram Somnath, Matteo Pariset, Ya-Ping Hsieh, Maria~Rodriguez Martinez, Andreas Krause, and Charlotte Bunne.
\newblock Aligned diffusion schr{\"o}dinger bridges.
\newblock In \emph{Uncertainty in Artificial Intelligence}, pages 1985--1995. PMLR, 2023.

\bibitem[Song and Ermon(2019)]{song2019generative}
Yang Song and Stefano Ermon.
\newblock Generative modeling by estimating gradients of the data distribution.
\newblock \emph{arXiv preprint arXiv:1907.05600}, 2019.

\bibitem[Song et~al.(2021)Song, Sohl-Dickstein, Kingma, Kumar, Ermon, and Poole]{song2021scorebased}
Yang Song, Jascha Sohl-Dickstein, Diederik~P. Kingma, Abhishek Kumar, Stefano Ermon, and Ben Poole.
\newblock Score-based generative modeling through stochastic differential equations.
\newblock In \emph{International Conference on Learning Representations (ICLR 2021)}, 2021.

\bibitem[Thomas et~al.(2018)Thomas, Smidt, Kearnes, Yang, Li, Kohlhoff, and Riley]{thomas2018tensorfieldnetworks}
Nathaniel Thomas, Tess Smidt, Steven Kearnes, Lusann Yang, Li~Li, Kai Kohlhoff, and Patrick Riley.
\newblock Tensor field networks: Rotation- and translation-equivariant neural networks for 3d point clouds, 2018.
\newblock \url{https://arxiv.org/abs/1802.08219}.

\bibitem[Thornton et~al.(2022)Thornton, Hutchinson, Mathieu, De~Bortoli, Teh, and Doucet]{thornton2022riemannian}
James Thornton, Michael Hutchinson, Emile Mathieu, Valentin De~Bortoli, Yee~Whye Teh, and Arnaud Doucet.
\newblock Riemannian diffusion schr$\backslash$" odinger bridge.
\newblock \emph{arXiv preprint arXiv:2207.03024}, 2022.

\bibitem[Tzen and Raginsky(2019{\natexlab{a}})]{tzen2019neural}
Belinda Tzen and Maxim Raginsky.
\newblock Neural stochastic differential equations: Deep latent {G}aussian models in the diffusion limit.
\newblock \emph{arXiv:1905.09883}, 2019{\natexlab{a}}.

\bibitem[Tzen and Raginsky(2019{\natexlab{b}})]{tzen2019theoretical}
Belinda Tzen and Maxim Raginsky.
\newblock Theoretical guarantees for sampling and inference in generative models with latent diffusions.
\newblock \emph{arXiv:1903.01608}, 2019{\natexlab{b}}.

\bibitem[Vargas et~al.(2023)Vargas, Grathwohl, and Doucet]{vargas2023denoising}
Francisco Vargas, Will~Sussman Grathwohl, and Arnaud Doucet.
\newblock Denoising diffusion samplers.
\newblock In \emph{The Eleventh International Conference on Learning Representations}, 2023.

\bibitem[Vargas et~al.(2024)Vargas, Padhy, Blessing, and Nusken]{vargas2024transport}
Francisco Vargas, Shreyas Padhy, Denis Blessing, and Nikolas Nusken.
\newblock Transport meets variational inference: Controlled monte carlo diffusions.
\newblock In \emph{The Twelfth International Conference on Learning Representations: ICLR 2024}, 2024.

\bibitem[Veber et~al.(2002)Veber, Johnson, Cheng, Smith, Ward, and Kopple]{veber2002molecular}
Daniel~F Veber, Stephen~R Johnson, Hung-Yuan Cheng, Brian~R Smith, Keith~W Ward, and Kenneth~D Kopple.
\newblock Molecular properties that influence the oral bioavailability of drug candidates.
\newblock \emph{Journal of medicinal chemistry}, 45\penalty0 (12):\penalty0 2615--2623, 2002.

\bibitem[Weiler et~al.(2018)Weiler, Geiger, Welling, Boomsma, and Cohen]{weiler20183dsteerablecnns}
Maurice Weiler, Mario Geiger, Max Welling, Wouter Boomsma, and Taco Cohen.
\newblock 3d steerable cnns: Learning rotationally equivariant features in volumetric data, 2018.
\newblock \url{https://arxiv.org/abs/1807.02547}.

\bibitem[Wu et~al.(2020)Wu, K{\"o}hler, and No{\'e}]{wu2020stochastic}
Hao Wu, Jonas K{\"o}hler, and Frank No{\'e}.
\newblock Stochastic normalizing flows.
\newblock \emph{Advances in Neural Information Processing Systems}, 33:\penalty0 5933--5944, 2020.

\bibitem[Xu et~al.(2022)Xu, Yu, Song, Shi, Ermon, and Tang]{xu2022geodiff}
Minkai Xu, Lantao Yu, Yang Song, Chence Shi, Stefano Ermon, and Jian Tang.
\newblock Geodiff: A geometric diffusion model for molecular conformation generation.
\newblock \emph{arXiv preprint arXiv:2203.02923}, 2022.

\bibitem[Zhang et~al.(2024)Zhang, Chen, Liu, Courville, and Bengio]{zhang2024diffusion}
Dinghuai Zhang, Ricky T.~Q. Chen, Cheng-Hao Liu, Aaron Courville, and Yoshua Bengio.
\newblock Diffusion generative flow samplers: Improving learning signals through partial trajectory optimization.
\newblock In \emph{The Twelfth International Conference on Learning Representations}, 2024.
\newblock \url{https://openreview.net/forum?id=OIsahq1UYC}.

\bibitem[Zhang and Chen(2022)]{zhang2022path}
Qinsheng Zhang and Yongxin Chen.
\newblock Path integral sampler: A stochastic control approach for sampling.
\newblock In \emph{International Conference on Learning Representations}, 2022.

\end{thebibliography}

\appendix

\section{Additional preliminaries} \label{sec:additional_prelims}

\subsection{Stochastic optimal control}\label{app:SOC}

Stochastic optimal control (SOC; \cite{bellman1957,fleming2012deterministic,sethi2018optimal}) considers general optimization problems over stochastic differential equations. Specifically, a class of SOC problems can be expressed as the following optimization problem:
\begin{align} \label{eq:control_problem_def}
    &\min\limits_{u \in \mathcal{U}} \mathbb{E} \big[ \int_0^1 
    \left(\frac{1}{2} \|u(X^u_t,t)\|^2 + f(X^u_t,t) \right) \, \mathrm{d}t + 
    g(X^u_1) \big], \\
    \begin{split}
    \text{s.t.}~ \mathrm{d}X^u_t =  \left( b(X^u_t,t) + \sigma(t) u(X^u_t,t) \right) \, \mathrm{d}t + 
    \sigma(t) \mathrm{d}B_t, \qquad X^u_0 \sim p_0
    \end{split} 
    \label{eq:controlled_SDE}
\end{align}
where in \eqref{eq:controlled_SDE}, $X_t^u \in \R^d$ is the state of the stochastic process, $u : \R^d \times [0,1] \to \R^d$ is commonly referred to as the control, $b : \R^d \times [0,1] \to \R^d$ is a base drift, and $g : [0,1] \to \R^{d \times d}$ is the diffusion coefficient. These jointly define the \emph{controlled process} $\fX^u \sim p^u$ that we are interested in optimizing; often both $b$ and $g$ are fixed and we only optimize over the control $u$. 

In the following proposition show that for the case of $b\equiv 0$, $f\equiv 0$ and $X_0=0$ that the Schrodinger-Bridge $p^*(\fX) = \pbase(\fX | X_1)\mu(X_1)$ can actually be achieved by $p^u(\fX)$.
\begin{prop}\label{prop:min_DKL}
For the stochastic process \eqref{eq:controlled_sde}, there exists a unique minimizer $u^*$ to the following optimization problem
\begin{equation}\label{eq:path_measure_KL_prop}
\min_u \KL\left(p^{u}(\fX) || p^*(\fX) \right) = 0,
\end{equation}
and the optimal controlled distribution $p^{u^*}$ satisfies:
\begin{align}\label{eq:optimal_dist}
    p^{u^*}_{t,1}(X_t, X_1) = p^{*}_{t,1}(X_t, X_1) = \pbase_{t|1}(X_t {\mid} X_1) \mu(X_1),
\end{align}
 where $\pbase_{t|1}(X_t|X_1)$ is the posterior distribution of the base process conditioned on $X_1$.
\end{prop}
\begin{proof}
By using the definition of the KL divergence of path measures, $p^*(\fX) = \pbase(\fX | X_1) \mu(X_1)$ can be factorized as
\begin{align}\label{eq:KL_regularized}
    D_{\text{KL}}\left(p^{u}(\fX) || p^*(\fX) \right)= 
    D_{\text{KL}}\left(p^{u}(\fX) || \pbase(\fX)\frac{\mu(X_1)}{\pbase_1(X_1)} \right)= D_{\text{KL}}\left(p^u(\fX) || \pbase(\fX)\right) + \mathbb{E}_{\fX \sim p^u} \log\left(\frac{\pbase(X_1)}{\mu(X_1)} \right).
\end{align}
By Girsanov's Theorem~\citep{protter2005stochastic} to further evaluate the right-hand side as
\begin{align*}
       D_{\text{KL}}\left(p^u(\fX) || \pbase(\fX)\right) + \mathbb{E}_{\fX \sim p^u} \log\left(\frac{\pbase(X_1)}{p^*(X_1)} \right) = \mathbb{E}_{\fX \sim p^u} \left[\int_0^1 \tfrac{1}{2}\lVert u(X_t, t) \rVert^2 \dif{t} + \log\left(\frac{\pbase(X_1)}{p^*(X_1)} \right)\right].
\end{align*}
This criteria corresponds exactly with a minimum-energy stochastic optimal control (SOC) problem  with stage cost $\tfrac{1}{2}\Vert u(X_t, t)\Vert^2$ and terminal cost function $g(X_1)=\log\left(\frac{\pbase(X_1)}{p^*(X_1)} \right)$.
\begin{align}\label{eq:soc-app}
        &\min_{u} \mathbb{E}_{\fX \sim p^u}\left[ \int_0^1 \tfrac{1}{2} \lVert u(X_t, t) \rVert^2 \dif{t} + g(X_1)\right] \\
    &\text{ s.t. } \dif{X_t} = 
    \sigma(t) u(X_t, t) \dif{t} + \sigma(t) \dif{B_t},\quad X_0 = 0 
\end{align}

We can define the cost-to-go of our problem at a particular state and time $t$ by
\begin{align}
    J(u; x, t) := \mathbb{E}_{\fX \sim p^u}\left[ \int_0^1 \tfrac{1}{2} \lVert u(X_t, t) \rVert^2 \dif{t} + g(X_1) \mid X_t = x\right]
\end{align}
then we know that (under mild conditions on $g$) the \emph{value function} or cost-to-go under the optimal control takes a rather surprising form and that indeed the optimal control $u^*$ is unique and takes the form of a time-varying function of the current state~\citep{kappen2005path}:
\begin{align}
    V(x, t) := \min_{u} J(u; x, t) =  -\log \mathbb{E}_{\pbase}\big[ \exp(-g(X_1)) \mid X_t = x \big],\quad     u^*(x, t) = - \sigma(t) \nabla_x V(x ,t)
\end{align}
With this fact about the optimal solution of the SOC problem, we can upper bound the aforementioned path KL divergence between the optimally controlled process and our target Schrodinger-bridge as
\begin{align}
     D_{\text{KL}}\left(p^{u^*}(\fX) || p^*(\fX) \right) = V(X_0, 0) &= -\log \mathbb{E}_{\fX \sim \pbase}\big[ \exp(-g(X_1)) \big]\\
     &= -\log \int_{\mathbb{R}^d} \mu(x) \dif{x} = 0,
\end{align}
By the data processing inequality we have 
\begin{align}
    D_{\text{KL}}\left(p^{u^*}_{t,1}(X_t, X_1) || p^*_{t,1}(X_t, X_1) \right) \leq  D_{\text{KL}}\left(p^{u^*}(\fX) || p^*(\fX)\right) = 0 ,\quad \forall t \in [0,1]
\end{align}
which implies that $p^{u^*}_{t,1}(X_t,X_1) = \pbase_{1|t}(X_t|X_1)\mu(X_1)$ for all $t\in [0,1]$. This is exactly our claim and completes the proof.
\end{proof}

\subsection{Adjoint Matching}

Adjoint Matching \citep{domingoenrich2024adjoint} is an algorithm designed to solve stochastic optimal control problems of the form \eqref{eq:control_problem_def}. Unlike standard adjoint methods~\citep{BrysonHo69} that differentiate through the objective and perform gradient descent, Adjoint Matching directly tries to find the fixed point
\begin{equation}\label{eq:fixed_point}
    u(x, t) = - \sigma(t) \nabla J(u; x, t) 
\end{equation}
where $J$ is the expected future cost according to the control $u$,
\begin{equation}
    J(u; x, t) = \E_{p^u} \left[ \int_t^1 \tfrac{1}{2}\norm{u(X_s, s)}^2 +f(X_s, s) \mathrm{d}s + g(X_1^u) \Bigg| X_t = x \right].
\end{equation}
Adjoint Matching then solves this fixed point \eqref{eq:fixed_point} by replacing $\nabla_x J$ with a stochastic estimator. This stochastic estimator is the \emph{lean adjoint} \citep{domingoenrich2024adjoint} denoted by $\tilde{\alpha}$. \citet{domingoenrich2024adjoint} then showed that the unique solution to the fixed point 
\begin{equation}
    u(x,t) = -\sigma(t) \E_{\fX \sim p^u} \left[ \tilde{\alpha}(x, t) | X_t = x\right]
\end{equation}
is the optimal control $u^*$. This then motivated the following objective as a means to solve this fixed point problem.
\begin{align}\label{eq:lean_adjoint_matching}
\mathcal{L}_{\mathrm{AM}}(u) 
:= \mathbb{E}_{p^{\bar u}}\Bigg[\frac{1}{2} \int_0^{1} \big\| & u(X_t 
,t)
+ \sigma(t)^{\top} \tilde{a}(t;\bm{X}) \big\|^2 \, \mathrm{d}t \Bigg], 
\qquad \fX \sim p^{\bar{u}}, \quad \bar{u} = \texttt{stopgrad}(u), \\
\label{eq:lean_adjoint_1}
\text{where }\quad \frac{\mathrm{d}}{\mathrm{d}t} \tilde{a}(t;\bm{X}) 
&= - (\tilde{a}(t;\bm{X})^{\top} \nabla_x b (X_t,t) + \nabla_x f(X_t,t)), \\ 
\label{eq:lean_adjoint_2}
\tilde{a}(1;\bm{X}) &= \nabla_{X_1} g(X_1).
\end{align}
where $\tilde{a}$ is referred to as the ``lean adjoint'' state, and $\texttt{stopgrad}(\cdot)$ denotes a stop gradient operation, \textit{i.e.}, although $\fX \sim p^{\bar{u}}$ is sampled according to the controlled process, Adjoint Matching does not differentiate through the sampling procedure.

Note that for the general problem formulation \eqref{eq:control_problem_def}, Adjoint Matching requires two simulations, one to sample a trajectory from the stochastic process $\fX \sim p^u$ and one to solve the lean adjoint state backwards in time \eqref{eq:lean_adjoint_1} from a terminal condition at $t=1$ \eqref{eq:lean_adjoint_2}. 

\paragraph{Adjoint Matching is greatly simplified for simple base processes.}
One of our key observations is that the Adjoint Matching algorithm greatly simplifies when the $b = 0$ and $f=0$, since then \eqref{eq:lean_adjoint_1} is zero. This means that the lean adjoint state $\tilde{a}(t, \fX) = \nabla g(X_1)$ for all $t$. The Adjoint Matching algorithm hence reduces to
\begin{align}
\mathcal{L}_{\mathrm{AM}}(u) 
:= \mathbb{E}_{p^{\bar u}}\Bigg[ \frac{1}{2} \int_0^{1}\big\| & u(X_t,t)
+ \sigma(t)^{\top}  \nabla_{X_1} g(X_1) \big\|^2 \, \mathrm{d}t \Bigg],  
\qquad \fX \sim p^{\bar{u}}, \quad \bar{u} = \texttt{stopgrad}(u).
\end{align}

\section{Base process derivations} \label{sec:base_processes}

\subsection{Euclidean space}

Our base process is modeled by a stochastic differential equation ($X_t \in \R^n$ and $\sigma: \R \rightarrow \R^+$):
\begin{equation}\label{eq:base_process}
    dX_t = \sigma(t) dB_t, \qquad X_0 = 0
\end{equation}corresponding to \eqref{eq:controlled_sde} with the control set to zero, \ie, $u_t = 0$. With this, the forward transition distributions are ($t > s$):
\begin{equation}\label{eq:pbase_forward_transition}
    \pbase_{t|s}(x | X_s) = \mathcal{N}(x | X_s, \nu_{t|s} I), \qquad \text{ where } \nu_{t|s} = \int_{s}^t \sigma(s)^2 ds.
\end{equation}
With $s = 0$ and $X_0 = 0$, we obtain the time marginals of \eqref{eq:base_process}:
\begin{equation}\label{eq:pbase_marginals}
    \pbase_{t}(x) = \mathcal{N}(x | 0, \nu_{t} I), \qquad \text{ where } \nu_{t} = \int_{0}^t \sigma(s)^2 ds.
\end{equation}
Furthermore, the backward transition distributions are ($t > s$): 
\begin{equation}\label{eq:pbase_backward_transition}
    \pbase_{s|t}(x | X_t) = \frac{\pbase_{s}(x) \pbase_{t|s}(X_t | x)}{\pbase_t(X_t)} = \mathcal{N}(x | \alpha_{s|t} X_t, \cev{\nu}_{s|t} I),
\end{equation}
\begin{equation}
    \text{ where } \alpha_{s|t} = \frac{\nu_s}{\nu_s + \nu_{t|s}}, \quad \cev{\nu}_{s|t} = \left( \frac{1}{\nu_s} + \frac{1}{\nu_{t|s}} \right)^{-1} = \alpha_{s|t} \nu_{t|s}
\end{equation}
As such, the Reciprocal Adjoint Matching objective \eqref{eq:ram} becomes tractable when $\nu_{t|s}$ can be derived in closed form.

\paragraph{Constant noise schedule.}
Setting the diffusion coefficient to be a constant $\sigma(t) = \sigma$ results in 
\begin{equation}
    \nu_{t|s} = \sigma^2 (t - s), \qquad \alpha_{s|t} = \tfrac{s}{t}
\end{equation}
and hence the following terminal and backward transition distributions:
\begin{equation}
    \pbase_{1}(X_1) = \mathcal{N}(X_1 ; 0, \sigma^2 I), \qquad \pbase_{t|1}(x | X_1) = \mathcal{N}(x ; t X_1, \sigma^2 (1-t) t I).
\end{equation}

\paragraph{Geometric noise schedule.} \citet{song2021scorebased, karras2022elucidating} propose the following geometric schedule (stated here in reverse-time)
\begin{equation}
    \label{eqn:geometric_noise_schedule_sigma}
    \sigma(t) = \sigma_{\text{min}} \left(\tfrac{\sigma_\text{max}}{\sigma_{\text{min}}}\right)^{1-t} \sqrt{2 \log \tfrac{\sigma_\text{max}}{\sigma_\text{min}}}
\end{equation}
results in
\begin{equation}
    \nu_{t|s} = \sigma_\text{max}^2 \left(\left(\tfrac{\sigma_\text{min}}{\sigma_\text{max}}\right)^{2s} - \left( \tfrac{\sigma_\text{min}}{\sigma_\text{max}} \right)^{2t} \right), \qquad 
    \alpha_{s|t} = \frac{\left(\tfrac{\sigma_\text{min}}{\sigma_\text{max}}\right)^{2s} - 1}{\left(\tfrac{\sigma_\text{min}}{\sigma_\text{max}}\right)^{2t} - 1}
\end{equation}
and hence the following terminal and backward transition distributions:
\begin{equation}
    \pbase_{1}(X_1) = \mathcal{N}(X_1; 0, \sigma_\text{max}^2 \left(1 - \left( \tfrac{\sigma_\text{min}}{\sigma_\text{max}} \right)^{2} \right) I), \qquad 
    \pbase_{t|1}(X_t | X_t) = \mathcal{N}(X_t ; \alpha_{t|1} X_1, \alpha_{t|1} \nu_{1|t} I ).
\end{equation}

\subsection{Flat tori}\label{sec:flat_tori_derivations}

Simulating the base process \eqref{eq:base_process} in $\mathbb{T} = \R / \mathbb{Z}$ produces a wrapped Gaussian distribution $\pbasepbc_{t}$
\begin{equation}\label{eq:pbase_pbc_marginals}
\pbasepbc_{t}(x) = \sum_{k \in \mathbb{Z}} \pbase_{t}(x + k),
\end{equation}
where $\pbase_{t}$ is the time marginal of the base process over $\R$ defined in \eqref{eq:pbase_marginals} and $\mathbb{Z}$ is the set of integers. 
We can equivalently can interpret \eqref{eq:pbase_pbc_marginals} as marginalizing a random variable $k$. That is, we view $\pbase_{t}(x + k)$ as a joint distribution $p_t(x, k)$ over $x \in \T$ and $k \in \Z$,
\begin{equation}
    p_t(x, k) := \underbrace{\frac{\pbase_{t}(x + k)}{\int_{x' \in \T} \pbase_{t}(x' + k)}}_{:= p_t(x | k)} \underbrace{\int_{x' \in \T} \pbase_{t}(x' + k) \vphantom{\frac{\pbase_{t}(x + k)}{\int_{x_t \in \T} \pbase_{t|s}(x_t + k_t | x_s)}} }_{:= p_t(k)}.
\end{equation}
Intuitively, the variable $k$ represents which sub-interval on $\R$ was collapsed to obtain $x$.
Furthermore, we can infer $k$ given $x$ at time $t$:
\begin{equation}
    p_t(k | x) = \frac{ p_t(x, k) }{ \sum_{k'} p_t(x, k') },
\end{equation}
hence arriving at the following backward transition distribution ($s < t$) in which we marginalize $k$:
\begin{equation}\label{eq:pbase_pbc_backwards}
    \pbasepbc_{s|t}(x | X_t) = \sum_{k \in \Z} p_{s|t}(x | k, X_t) p_t(k | X_t)  = \sum_{k \in \mathbb{Z}} \left( \sum_{k'\in \Z} \pbase_{s|t}(x + k' | X_t + k) \right) \frac{\pbase_t(X_t + k)}{\sum_{k' \in \mathbb{Z}} \pbase_t(X_t + k') },
\end{equation}
where $\pbase_{s|t}(x | X_t + k)$ is the backward transition probability in $\R$ defined in \eqref{eq:pbase_backward_transition}. Thus sampling from $\pbasepbc_{t|1}(x | X_1)$ given in \eqref{eq:pbase_pbc_backwards} can be done by the following sampling process 
\begin{enumerate}
    \item Sample $k \sim p(k)$ where $p(k) \propto \pbase_1(X_1 + k)$.
    \item Sample $Y_t \sim \pbase_{t|1}(\cdot | X_1 + k)$.
    \item Project $X_t = Y_t \bmod 1.0$.
\end{enumerate}
For the $n$-dimensional $\T^n$, since the base process is factorized, we simply sample independently for each dimension.

\section{Theory of Reciprocal Adjoint Matching} \label{app:theory}

As a preliminary, let us first define precisely what is meant be \emph{Reciprocal projection}. We take some terminology and results from \citet{shi2024diffusion}, denoting $\mathcal{M}$ the subset of Markov path measures associated with an SDE of the form $\mathrm{d} X_t = v(X_t,t) \mathrm{d}t +
\sigma(t) \, \mathrm{d}B_t$, with $\sigma$, $v$ locally Lipschitz.
\begin{definition}[Reciprocal class, Def.~3 of \citet{shi2024diffusion}]
A Borel probability measure $\Pi \in \mathbf{P}(C)$ is in the Reciprocal class $\mathcal{R}(\mathbb{Q})$ of $\mathbb{Q} \in M$ if $\Pi = \Pi_{0,T} \mathbb{Q}_{(0,T)|0,T}$. 
\end{definition}
The Reciprocal projection of $\mathbb{P} \in \mathbf{P}(C)$ is defined as $\Pi^* = \mathcal{P}_{\mathcal{R}(\mathbb{Q})}(\mathbb{P}) = \mathbb{P}_{0,T} \mathbb{Q}_{(0,T)|0,T}$. By construction, $\Pi^*$ belongs to the Reciprocal class $\mathcal{R}(\mathbb{Q})$. The following characterization shows that the Reciprocal projection is indeed a projection, i.e. $\mathcal{P}_{\mathcal{R}(\mathbb{Q})}^2 = \mathcal{P}_{\mathcal{R}(\mathbb{Q})}$.
\begin{lem}[Reciprocal projection characterization, Prop.~4 of \citet{shi2024diffusion}]
    Let $\mathbb{P} \in \mathbf{P}(C)$, 
    $\Pi^{*} = \mathcal{P}_{\mathcal{R}(\mathbb{Q})}(\mathbb{P})$. 
    Then, $\Pi^{*} = \argmin_{\mathbb{P}} \{\mathrm{KL}(\mathbb{P}|\mathbb{Q}) : \Pi \in \mathcal{R}(\mathbb{Q}) \}$.
\end{lem}
Throughout the paper, we set $\mathbb{Q} = \mathbb{B}_{\nu}$, i.e. $\mathbb{Q}$ is the probability measure of the process $(\int_0^t \sigma(s) \, \mathrm{d}B_s)_{t \in [0,T]}$, which we denote as $\mathcal{P}_{\mathcal{R}(\mathbb{Q})}$.

Although this projection is exactly what is used by the proposed reciprocal adjoint matching objective~\Cref{eq:ram}, we will now proceed to define a similar projection in terms of drifts which will allow us to analyze an idealized version of the adjoint sampling algorithm.

\subsection{Reciprocal projections decrease the SOC objective} \label{app:Reciprocal}

In the case that $p^u = \mathbb{P}$ is generated by some drift function $u$, we can view the reciprocal projection as a projection of drifts $\mathcal{P}$ as proposed in \Cref{thm:adjoint_sampling}, noting that the unique drift generating $\mathcal{P}_{\mathcal{R}}(p^u)$ is given as the solution of the following Schrodinger bridge problem
\begin{align}       
        \mathcal{P}(u) := \argmin_{w}&~\mathbb{E}_{\fX \sim p^w}\left[ \int_0^1 \tfrac{1}{2} \lVert w(X_t, t) \rVert^2 \dif{t} + \log \pr{\frac{\pbase(X_1)}{p^{u}(X_1)}}\right].
        \label{eq:Reciprocal-soc}
\end{align}
This projection of drifts allows us to reason about the RAM loss with respect to both the AM and SOC functionals which are described in~\Cref{prop:projection}.
\begin{proof}[\textbf{Proof of \Cref{prop:projection}}]
Let $v=\mathcal{P}(u)$. Because $\mathcal{P}^2(u) = \mathcal{P}(u)$, the reciprocal projection ``approximation'' introduced in RAM is exactly the distribution induced by $v$ and so $\mathcal{L}_{\text{AM}}(v) = \mathcal{L}_{\text{RAM}}(v)$.

To show that $J(u) \geq J(\mathcal{P}(u))$, we simply apply the definition of $\mathcal{P}(u)$~\eqref{eq:Reciprocal-soc} which yields
\begin{align*}
        J(u) 
        &= \E_{\fX \sim p^u}\left[ \int_0^1 \tfrac{1}{2} \lVert u(X_t, t) \rVert^2 \dif{t} + \pr{\log \frac{\pbase}{\mu}}(X_1)\right] \\
        &= \E_{\fX \sim p^u}\left[ \int_0^1 \tfrac{1}{2} \lVert u(X_t, t) \rVert^2 \dif{t} + \pr{\log \frac{\pbase}{p^u} + \log {\frac{p^u}{\mu}}}(X_1)\right] \\
        &\ge \E_{\fX \sim p^v}\left[ \int_0^1 \tfrac{1}{2} \lVert v(X_t, t) \rVert^2 \dif{t} + \pr{\log \frac{\pbase}{p^u}}(X_1)\right] + \E_{X_1 \sim p^u}\left[ \log {\frac{p^u}{\mu}}(X_1) \right] 
        && \text{\hypersetup{hidelinks}\color{lightgray} by \eqref{eq:Reciprocal-soc}} \\
        &= \E_{\fX \sim p^v}\left[ \int_0^1 \tfrac{1}{2} \lVert v(X_t, t) \rVert^2 \dif{t} + \pr{\log \frac{\pbase}{p^u} + \log {\frac{p^u}{\mu}}}(X_1)\right] 
        && {\color{lightgray} p^v_1(X_1) = p^u_1(X_1) } \\
        &= J(v) = J(\mathcal{P}(u)).
    \end{align*}
\end{proof}

\subsection{Adjoint Sampling Preserves Critical Points of Adjoint Matching} 

\label{subsec:alternating_algorithm}
Consider a slightly more general version of the Reciprocal Adjoint Matching (RAM) objective with arbitrary fixed reference drift $v$,
\begin{align}
    \mathcal{L}(u; v) =\mathbb{E}_{\fX \sim p^v} \left[
    \int_0^1 \frac{1}{2}\lVert u(X_t, t) 
    + \sigma(t) \nabla g(X_1) \rVert^2 \dif{t}\right] \\
    s.t. \quad \dif{X_t} = \sigma(t)u(X_t,t) \dif{t} + \sigma(t)\dif{B_t},\quad X_0 =0
\end{align}
Then the RAM and AM loss can be both be equivalently restated in terms of $\mathcal{L}$
\begin{align*}
    \mathcal{L}_{\text{RAM}}(u) = \mathcal{L}(u; \tilde u),\quad \mathcal{L}_{\text{AM}}(u) = \mathcal{L}(u,\bar u)\quad s.t. \quad \tilde u = \mathcal{P}(\bar u),\quad ,\bar u = \texttt{stopgrad}(u),
\end{align*}
where RAM is shown by first projecting via $\mathcal{P}$, and then setting this as the reference drift with stop gradient applied.

If in our main algorithm we consider sequentially minimizing the RAM loss where starting from some drift $v_i$ we only populate the buffer $\mathcal{B}$ with samples from $p^{v_i}$, we can model adjoint sampling via the following update procedure:
\begin{align}\label{eq:iterative_AS}
    v_{i+1} = \argmin_{u} \mathcal{L}(u; \tilde v_i),\quad s.t. \quad \tilde v_i = \mathcal{P}(v_i)
\end{align}
with this iterative formulation, we can actually re-frame Adjoint Sampling in terms of the stationary condition of the Adjoint matching loss derived in~\cite{domingoenrich2024adjoint}. To show this, we first consider the \emph{first variation} of $\mathcal{L}$ with respect to the first argument which is defined by
\begin{align*}
\delta \mathcal{L}(u; v) (\eta) := \frac{\dif{}}{\dif{\varepsilon}} \mathcal{L}(u+\varepsilon \eta; v)|_{\varepsilon=0}= \mathbb{E}_{\fX \sim p^v}\left[\int^1_0 \langle u(X_t, t) + \sigma(t) \nabla g(X_1), \eta(X_t,t) \rangle  \dif{t}\right]
\end{align*}
the first-order necessary condition of an extremum of this cost functional is that $\delta \mathcal{L}(u;v) \equiv 0$ which is to say
\begin{align}
\delta \mathcal{L}(u; v) (\eta) =
\mathbb{E}_{\fX \sim p^v}\left[\int^1_0 \langle u(X_t, t) + \sigma(t) \mathbb{E}_{p^v}\left[\nabla g(X_1)\mid X_t\right], \eta(X_t,t) \rangle \dif{t} \right]=
0,\quad \forall \eta \in \mathcal{C}(\mathbb{R}^d \times [0,1],\mathbb{R})
\end{align}
We know that if $p^v$ has support over $\mathbb{R}^d$ and $\mathbb{E}_{p^v}\left[\int^1_0 \langle \frac{\delta \mathcal{L}}{\delta u}(X_t,t), \eta(X_t,t) \rangle  \dif{t}\right] = 0$ for all $\eta$, then it must be that $\frac{\delta \mathcal{L}}{\delta u}(u;v)(x,t)=0$ almost everywhere. This is the functional derivative with respect to $u$ which satisfies
\begin{align}
\frac{\delta \mathcal{L}}{\delta u}(u;v) (x,t) := u(x, t) + \sigma(t) \mathbb{E}_{p^{v}}\left[\nabla g(X_1)\mid X_t = x \right] = 0,\quad a.e.
\end{align}
In view of this general necessary condition, we can identify the corresponding necessary condition for $\mathcal{L}_{\text{AM}}(u)$.
\begin{align}\label{eq:AM_first_variation}
    \frac{\delta \mathcal{L}_{\text{AM}}}{\delta u} (u)(x,t) := \frac{\delta \mathcal{L}}{\delta u} (u;\bar u)(x,t) = u(x, t) + \sigma(t) \mathbb{E}_{p^{\bar u}}\left[\nabla g(X_1)\mid X_t = x \right] = 0,\quad a.e.
\end{align}
This coincides with the functional derivative derived in \cite{domingoenrich2024adjoint}, where it was observed that the critical points of Adjoint Matching are exactly when \eqref{eq:AM_first_variation} is zero. Furthermore, it was determined that the only critical point is the optimal control $u^*$ to the SOC problem~\eqref{eq:soc}.
We can now state \Cref{thm:adjoint_sampling} more precisely and provide a proof in the following Theorem.

\begin{theorem}[Adjoint Sampling] \label{thm:adjoint_sampling_formal}

Suppose we have a feedback $u_i$. Then performing the iteration~\eqref{eq:iterative_AS} to obtain $u_{i+1}$ is equivalent to the iteration
\begin{equation}
        u_{i+1} = \P(u_i) - \frac{\delta \mathcal{L}_{AM}}{\delta u}(\P(u_i)),
\end{equation}
Moreover, $u$ is a fixed point satisfying $u = \mathcal{P}(u)$ and $u=u-\frac{\delta \mathcal{L}_{AM}}{\delta u}(u)$ if and only if $u$ is a critical point of $\mathcal{L}_{\text{AM}}$ (i.e. it
    is the optimal control $u^*$ to \eqref{eq:soc}).
\end{theorem}

\begin{proof}
By minimizing $\mathcal{L}(u;\tilde u_i)$ point-wise for a fixed $\tilde u_i = \mathcal{P}(u_i)$, we know the optimal iterates take the form
\begin{align}
    u_{i+1}(x,t) &= -\sigma(t) \mathbb{E}_{p^{\tilde u_i}}\left[ \nabla g(X_1) \mid  X_t = x \right].
\end{align}
Recalling the Adjoint Matching loss $\mathcal{L}_{\text{AM}}(u) = \mathcal{L}(u;\bar u)$ we can identify its functional derivative at $\mathcal{P}(u_i)$ that stems from first variation to be
\begin{align}
    \frac{\delta \mathcal{L}_{AM}}{\delta u}(\P(u_i)) = \mathcal{P}(u_i) + \sigma(t) \mathbb{E}_{p^{\mathcal{P}(u_i)}}\left[ \nabla g(X_1) \mid  X_t = x \right].
\end{align}
Through a simple algebraic manipulation, we can arrive at a new expression for the iterate given by
\begin{align*}
    u_{i+1}(x,t) = \mathcal{P}(u_i)(x,t) - \frac{\delta \mathcal{L}_{AM}}{\delta u}\left(\P(u_i)\right)(x,t).
\end{align*}
We will now show that $v$ is a fixed point of the Adjoint Sampling iteration and $\mathcal{P}(v) = v$ if and only if $v$ is a critical point of $\mathcal{L}_{\text{AM}}$ (and hence $v=u^*$). First of all, if $v$ is a critical point of $\mathcal{L}_{\text{AM}}$, then it must be $v=u^*$~\citep{domingoenrich2024adjoint} and $\frac{\delta \mathcal{L}_{AM}}{\delta u}(v)= 0$. Because $v=u^*$ is the optimal solution of the SOC loss, it is its own projection and so $\mathcal{P}(u^*) - \frac{\delta \mathcal{L}_{AM}}{\delta u}(\P(u_i))(u^*) = u^*$. Therefore $v$ is a fixed point of the Adjoint Sampling iteration. 

On the other hand, if $v$ is a fixed-point of the Adjoint Sampling iteration and $\mathcal{P}(v)=v$, our iteration formula tell us that $\frac{\delta \mathcal{L}_{AM}}{\delta u}(v) = 0$ and so $v$ is a critical point of $\mathcal{L}_{\text{AM}}$ and that $v=u^*$ necessarily.
\end{proof}

There are some important  practical difference between this statement and what we do in practice. In practice, we optimize over a finite size buffer $\mathcal{B}$, where the iterations discussed here in~\Cref{thm:adjoint_sampling_formal} assumes we sample from $p^{u_i}_1$. Furthermore, we do not need to exactly solve the RAM step to convergence in practice and only partially optimize for a fixed number of iterations before collecting new samples for $\mathcal{B}$. This may be desirable as it will have a smoothing affect on our functional gradient descent of~\Cref{thm:adjoint_sampling_formal} which tracks a large unit step-size gradient update.

\section{Bridge Matching and Pretraining}\label{app:bridge_matching}

There may be situations where we have some initial data distribution $p^{\text{data}}$, but want to improve our generated samples even further using an energy function. We can in fact pretrain our control drift $u_\theta$ on this data using \emph{Bridge Matching}~\citep{peluchetti2023diffusion,peluchetti2023non,shi2024diffusion,liu2023generalized,liu20232,somnath2023aligned,ma2024sit,chen2024probabilistic}, which can be similarly formulated as a controlled SDE diffusing from $X_0 = 0$ and switch to Adjoint Sampling from the energy without changing the optimal distribution.

Specifically, given our controlled stochastic process~\Cref{eq:controlled_sde}, we can write down the following Bridge Matching loss to regress onto
\begin{align}
\int^1_0 \lambda(t) \mathbb{E}_{\pbase_{t|1}(X_t|X_1)p^{\text{data}}(X_1)} \lVert \sigma(t) u_\theta(X_t,t) -  \sigma(t)^2 \nabla_{X_t} \log \pbase_{1|t}(X_1 | X_t)\rVert^2 \mathrm{d} t,
\end{align}
where for general base process $\dif{X_t} = \sigma(t)\dif{B_t}$, we have
\begin{align*}
    \pbase_{1|t}(x |X_t) = \mathcal{N}(x | X_t, \nu_{1|t} I),\quad \nu_{1|t} = \int^1_t \sigma(s)^2 \dif{s}.
\end{align*}
Therefore our score target is
\begin{align*}
    \sigma(t)^2 \nabla_{X_t} \log \pbase_{1|t}(X_1|X_t)
    = \sigma(t)^2 \nabla_{X_t} \left( -\frac{1}{2 }\frac{(X_1- X_t)^2}{\nu_{1|t}} + \text{const}\right) 
    =  \sigma(t)^2 \frac{X_1 - X_t}{\nu_{1|t}}.
\end{align*}
Additionally, we apply a scaling $\lambda(t) := \frac{\sqrt{\alpha(t)}}{\sigma(t)}$ to the loss so that the target is $(X_1 - X_t)$ which improves numerical stability, resulting in the following Bridge Matching objective.
\begin{align}
    \mathcal{L}_{\text{BM}}(\theta)=\int^1_0 \mathbb{E}_{\pbase_{t|1}(X_t|X_1)p^{\text{data}}(X_1)} \left\lVert \frac{\nu_{1|t}}{\sigma(t)}u_\theta(X_t,t) -  (X_1 - X_t) \right\rVert^2 \mathrm{d} t
\end{align}

\section{Synthetic Energy Experiment Details}\label{app:synthetic_energy}

\subsection{Double Well Potential (DW-4)}
We use the same double-well potential as in iDEM~\citep{akhounditerated}, which was originally proposed in~\citet{kohler2020equivariant}. DW-4 describes a pair-wise distance potential energy for a system of 4 particles $\{x_1, x_2, x_3, x_4 \}$, where each particle has 2 spatial dimensions $x_i \in \mathbb{R}^2$ $(d=8)$. The potentials analytical form is given by:
\begin{equation}
    E(x) = \frac{1}{\tau}\sum_{ij} a (d_{ij} - d_0) + b (d_{ij} - d_0)^2 + c (d_{ij} - d_0)^4,\quad d_{ij} = \lVert x_i - x_j \rVert_2
\end{equation}
where we set $a=0$, $b=-4$, $c=0.9$ and temperature $\tau = 1$.
\subsection{Lennard-Jones Potential (LJ-13, LJ-55)}
Similarly to DW-4, the Lennard-Jones potential is also based pair-wise distances of a system with $n$ particles, but each each particle has 3 spatial dimensions. Its analytical form is given by
\begin{equation}
    E^{\text{LJ}}(x) = \frac{\epsilon}{\tau} \sum_{ij} \left( \left(\frac{r_m}{d_{ij}}\right)^6 - \left(\frac{r_m}{d_{ij}}\right)^{12} \right),\quad  d_{ij} = \lVert x_i - x_j \rVert_2 
\end{equation}
where \( r_m, \tau, \epsilon \) and \( c \) are physical constants. As in \citet{kohler2020equivariant} and \citet{akhounditerated}, we add the additional a harmonic potential:

\begin{equation}
    E^{\text{osc}}(x) = \frac{1}{2} \sum_i \| x_i - x_{\text{COM}} \|^2
\end{equation}

where \( x_{\text{COM}} \) refers to the center of mass of the system. Therefore, the final energy is then \( E^{\text{Tot}} = E^{\text{LJ}}(x) + c E^{\text{osc}}(x) \), for \( c \) the oscillator scale. As in previous work, we use \( r_m = 1, \tau = 1, \epsilon = 1 \) and \( c = 1.0 \).

\subsection{Architectures and Hyper parameters}
In this section we give a brief overview of the parameters used to train Adjoint Sampling on the synthetic energy functions. For all experiments, we an EGNN architecture similar to~\citet{akhounditerated}. Results for iDEM on all experiments are reproduced using the same configurations found in~\citet{akhounditerated}.

\textbf{DW-4: }
For the DW-4 energy we use an EGNN with 3 layers and 128 hidden features. We train Adjoint Sampling for 1000 outer iterations, generating 512 new samples and energy evaluations per iteration into a buffer of max size 10000. In each iteration we optimize on 500 batches of batch size 512 from the replay buffer, using a learning rate of $3\times 10^{-4}$. We use a geometric noise schedule with $\sigma_{min} = 10^{-4}$ and $\sigma_{max} = 3.0$.

\textbf{LJ-13: }
For the LJ-13 energy we use an EGNN with 5 layers and 128 hidden features. We train Adjoint Sampling for 1000 outer iterations, generating 1024 new samples and energy evaluations per iteration into a buffer of max size 10000. In each iteration we optimize on 500 batches from the replay buffer of batch size 512, using a learning rate of $3\times 10^{-4}$. We use a geometric noise schedule with $\sigma_{min} = 10^{-3}$ and $\sigma_{max} = 3.0$.

\textbf{LJ-55: }
For the LJ-55 energy we use an EGNN with 5 layers and 128 hidden features. We train Adjoint Sampling for 1000 outer iterations, generating 128 new samples and energy evaluations per iteration into a buffer of max size 10000. In each iteration we optimize on 500 batches from the replay buffer of batch size 512, using a learning rate of $3\times 10^{-4}$. We use a geometric noise schedule with $\sigma_{min} = 0.5$ and $\sigma_{max} = 3.0$.

\subsection{Reported Metrics}\label{app:w2_metric}

\paragraph{Geometric $\mathcal{W}_2$}

Because our energy and diffusion process is invariant to rotational and permutations symmetries, we also take these symmetries into account when measuring distance between generated and ground truth point clouds (\eg, from long run MCMC). Note this $\mathcal{W}_2$ metric is different from what is reported in~\citet{akhounditerated}, which uses the euclidean metric.

\begin{equation}
    \mathcal{W}_2(\mu, \nu) = \left( \inf_{\pi} \int \pi(x,y) d(x,y)^2 \,\dif{x}\dif{y} \right)^{\frac{1}{2}}
\end{equation}

where \( \pi \) is the transport plan with marginals constrained to \( \mu \) and \( \nu \) respectively. Here the distance metric $d$ takes into account all point-cloud symmetries, and is obtained by minimizing the squared Euclidean distance over all possible combinations of rotations, reflections ($O(d)$), and permutations ($\mathbb{S}(k)$) for $k$ particles of spatial dimension $d$.

\begin{equation}
    d(x_0, x_1) = \min_{R \in O(d), P \in \mathbb{S}(k)} \| x_0 - (R\otimes P) x_1 \|_2^2.
\end{equation}

However, computing the exact minimal squared distance is computationally infeasible in practice. Therefore, we adopt the approach of~\citet{kohler2020equivariant} and approximate the minimizer by performing a sequential search

\begin{equation}
    d(x_0, x_1) \approx \min_{R \in SO(d)} \| x_0 - (R\otimes \tilde{P}) x_1 \|_2^2, \quad 
    \tilde{P} = \arg\min_{P \in \mathbb{S}(k)} \| x_0 - P x_1 \|_2^2.
\end{equation}

\paragraph{Energy $\mathcal{W}_2$ ($E(\cdot)\,\mathcal{W}_2$)} 
An informative way to assess the quality of the generated samples is to look at their energy distribution with respect to the ground truth distribution obtained from energies of long-run MCMC simulations. This shows how well the generated samples are avoiding high-energy regions. Specifically, we use the Wasserstein-2 ($\mathcal{W}_2$) distance on the 1-dimensional energy distribution on $\mathbb{R}$. This results in the standard euclidean $\mathcal{W}_2$ metric

\begin{equation}
\mathcal{W}_2(E_{\mu}, E_{\nu}) = \left( \inf_{\pi} \int \pi(x,y)|x - y|^2 \dif{x}\dif{y} \right)^{\frac{1}{2}},
\end{equation}

\paragraph{Path effective sample size (path-ESS).}

We use importance weights over paths $\fX \sim p^u$~\citep{zhang2022path},
\begin{equation} \label{eq:importance_weights}
    w(\fX) = \frac{\mathrm{d} p^*(\fX)}{\mathrm{d} p^u(\fX)} =\frac{\mu(X_1)}{\pbase_1(X_1)}\frac{\dif \pbase(\fX)}{\dif p^u(\fX)}= \exp\left( -\frac{1}{2}\int_0^1  \norm{u(X_t, t)}^2 \dif{t} - \int^1_0 u_t(X_t, t)^\top \mathrm{d} B_t - g(X_1) \right).
\end{equation}
As these are unnormalized weights, we use the normalized effective sample size:
\begin{equation}
    \text{ESS} = \frac{1}{n\mathbb{E}_{\fX \sim p^u}[w(\fX)]} \approx \frac{\left(\sum_{i=1}^n w_i\right)^2}{n\sum_{i=1}^n w_i^2 },
\end{equation}
estimated over $n$ samples $\fX_i \stackrel{iid}{\sim} p^u$ with weights $w_i = w(\fX_i)$. Additionally, this ESS is normalized by the number of samples $n$ and takes a value in $[0,1]$.

\section{Sampling Molecular Conformers}
\label{app:conformers_and_torsions}

Conformer generation is a fundamental task in computational chemistry with applications from drug discovery to catalyzing particular reactions of interest. In the context of this paper, conformers represent different stable states, i.e. local minima on the potential energy surface, of the same molecule. The number of conformers a molecule has is generally proportional to the number of rotatable bonds present in the molecule. For example, large molecules with lots of rotatable bonds have many conformers. A bond is considered rotatable if there is a sufficiently low energy barrier to rotation and the bond is attached to non-terminal heavy atoms \citep{veber2002molecular}. We define a specific substructure to match this definition in \Cref{app:torsion_sampling} and use it to determine the torsion angles sampled by our torsion model. However, we apply a slightly different, more restrictive, definition to classify molecules into strata based on the standard definition from RDKit in \Cref{app:data_prep}.

In many molecules and polymers torsion angles are particularly important degrees of freedom for sampling conformations. A molecule has $3k - 6$ degrees of freedom corresponding to relevant atomic motion. This quantity subtracts global translation (3 degrees) and orientation in space (3 degrees), which are both irrelevant to the intrinsic molecule. The remaining degrees of freedom can be broken down into bond lengths (pairs of atoms), bond angles (3-tuples containing two atoms bonded to one common neighbor), and torsion angles (4-tuples containing a central bond and two atoms connected at either end). The energy required to stretch a bond or distort a bond angle is typically much higher than rotating a torsion angle. As a result, torsion angles are the most important degrees of freedom for determining the distribution of possible molecular states, or the Boltzmann distribution, and only allowing torsion angles to rotate is a common approximation used in many algorithms. 

\paragraph{Relaxation}
Conformers represent local minima in the energy landscape; therefore, we are interested in descent algorithms to refine structures generated by Adjoint Sampling or RDKit to reach those minima. We apply the same energy model, eSEN, that we utilized to train our drift estimator on energy to define an optimization target. Doing gradient descent on energy using forces is called relaxation. In our relaxations, we use the LBFGS algorithm and bound the maximum force norm $F_{\max} \coloneqq \max_{i} \lVert \nabla_{x_{i}} E(\tilde{X_{1}})\rVert$ on any of the $k$ atoms as a stopping criteria, where $\tilde{X_{1}}$ is the latest structure in the descent sequence. We stop when $F_{\max} \leq 0.0154$ eV/\AA\xspace or at 200 descent iterations, whatever comes first.

\subsection{Sampling the target molecule in Cartesian coordinates}
\label{subsec:cartesian_sampling}

\paragraph{Bond Structure Energy Regularization}
Although we condition on the molecular graph, our drift control estimator aims to approximate gradient of the energy function. During the evolution of our stochastic trajectories we may wander through energy troughs that have alternative connectivity than that of our target molecule. Since we aim to generate conformers of a specific target molecule, we need to avoid this. We do so by adding a regularizer to the energy. This regularizer assigns atoms particular bond / no bond lengths based on empirically known interatomic radii and repulsive van der Waals radii.
\begin{equation}
    \tilde{E} (x) = \frac{1}{\tau} E(x) + \alpha E^{\text{reg}}(x)
\end{equation}
where $\alpha > 0$ is regularizer scaling constant. The regularizer energy function $E^{\text{reg}} : \mathbb{R}^d \rightarrow [0,\infty)$ is designed to be a ``flat-bottom'' potential having the property that 
\begin{equation}
    E^{\text{reg}}(x) = 0 \Longleftrightarrow x \in \mathcal{S}
\end{equation}
where $\mathcal{S}$ is the desired molecular structure class. In practice, we slightly relax this to be $x \in \mathcal{S} \implies E^{\text{reg}}(x) = 0$, for training stability and due to our inability to exactly characterize $\mathcal{S}$. 

In practice, we determine an upper-bound distance using well-established empirical covalent radii between atoms with bonds, and lower-bound distances using van der Waal radii between atoms without bonds. We then penalize any pairwise distances which violate these bounds taking into account a relaxation factor $\gamma \geq 1$. With the adjacency matrix $A=(a_{ij})$ and atom type attribute vector $H=(h_i)$, we can determinine the molecule structure class $\mathcal{S}(A, H)$. With accompanying interatomic bond limit functions $d_{\text{bond}}(h_i, h_j)$ and $d_{\text{no-bond}}(h_i, h_j)$, we compute the regularizer iterating over all edges as follows:
\begin{align}
    E^{\text{reg}}(x, A, H; \gamma) = &\sum_{i,j \in I_{\text{bond}}}\max\left\{||x_i - x_j||_2-\gamma \cdot d_{\text{bond}}(h_i, h_j),  0\right\}\\
    + &\sum_{k,l\in I_{\text{no-bond}}}\max\left\{\frac{1}{\gamma}\cdot d_{\text{no-bond}}(h_k, h_l)-||x_k - x_l||_2,  0\right\}
\end{align}
where $I_{\text{bond}}$ and $I_{\text{no-bond}}$ are the index pairs of $A$ corresponding to a bond or no-bond between atoms respectively. Assuming accurate interatomic radiis, our regularizer satisfies $x\in \mathcal{S}(A,H) \implies E^{reg}(x,A,H;\gamma) = 0$ for all $\gamma \geq 1$.
\subsection{Torsion Angles and Torques}
\label{app:torsion_sampling}
\begin{figure}[hbt]
    \centering
    \begin{subfigure}[m]{0.45\textwidth}
        \includegraphics[width=\textwidth]{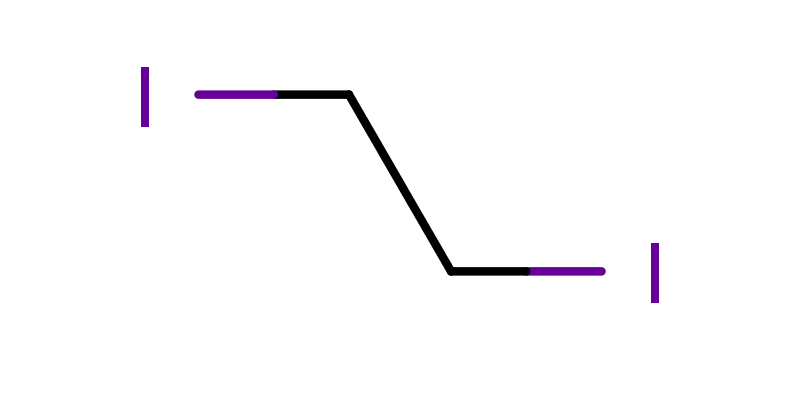}
    \end{subfigure}
    \hfill
    \begin{subfigure}[m]{0.45\textwidth}
        \includegraphics[width=\textwidth]{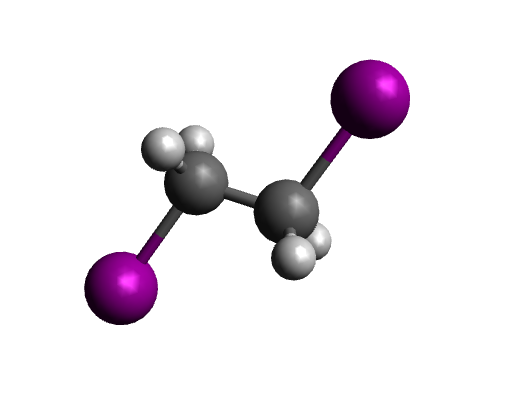}
    \end{subfigure}
    \caption{Multiple representations of the molecule \emph{diiodoethane} with SMILES string ICCI. (Left) Molecular graph representation. (Right) 3D coordinates representation. We chose the torsion angle as I-C-C-I, where I indicates a purple iodine atom and C indicates a gray carbon atom. We identify the torsions with a 4-tuple of indices, selecting the heaviest atoms on either side of the central bond as the first and last members. Ties are broken by arbitrary atomic index. We take the zero dihedral angle to be when first and last members of the 4-tuple are close, the so-called \emph{cis-}isomer. Following \citet{jing2022torsional}, we adjust the dihedral angle by rotating all the atoms at one end of the torsion about the bond axis, leaving the remaining atoms in place. This can be described as a torque pointing along the bond axis. When the torque would asymmetrically act on the molecule, we define the positive direction towards the side of the molecule with more atoms.}
    \label{fig:icci_representations}
\end{figure}

\paragraph{Representing Torsion Angles and Torques}
The orientation around a given bond can be quantified with a torsion angle. Specifically, a torsion angle is defined between four atoms $(a,b,c,d)$, where the torsion angle is the angle between the planes formed by $a,b,c$ and $b,c,d$, which are rotated around the central $bc$ bond. The choice of atoms $a$ and $d$ are typically not unique, thus a different choice of $a$ or $d$ can lead a different torsion angle to describe the same molecular conformation. Our convention is that $a$ and $d$ correspond to the heaviest atoms on either side of the central bond $bc$. We break any ties by arbitrary existing atomic index. Next, we represent torsion angles symbolically\ldots Assuming that atoms $a,b,c,d \in \mathbb{R}^{3}$ are positions in Cartesian space, we define a few vectors that lie along undirected bond axes: 
\begin{align}
    dx_{1} &\coloneqq \overrightarrow{ab} \coloneqq b - a, & dx_{2} &\coloneqq \overrightarrow{bc} \coloneqq c - b, & dx_{3} &\coloneqq \overrightarrow{cd} \coloneqq d - c.
\end{align}
Those vectors allow us to define torsion angle as follows:
\begin{align}
    \label{eqn:torsion_angle}
    \varphi(dx_1, dx_2, dx_3) &\coloneqq \operatorname{atan2}\left(dx_{2} \cdot \left(\left(dx_{1} \times dx_{2}\right) \times \left(dx_{2} \times dx_{3}\right)\right), |dx_{2}| \left(dx_{1} \times dx_{2}\right) \cdot \left(dx_{2} \times dx_{3}\right)\right).
\end{align}
We define the function $\mathcal{T} \colon \mathbb{R}^{4 \times 3} \to \mathbb{R}$ that takes in four column vectors of atomic positions $(a, b, c, d)$ and outputs a torsion angle by \eqref{eqn:torsion_angle}, \ie,  $\mathcal{T} \colon (a,b,c,d) \mapsto \varphi(dx_1(a, b), dx_2(b, c), dx_3(c, d))$.

Following \citet{jing2022torsional}, we can adjust a torsion angle by rotating \emph{all the atoms} on one side of the $bc$ bond about the bond axis $\overrightarrow{\text{bc}}$. This can be described as a torque $\mathfrak{t}$ in the standard \emph{angle-axis representation} where the axis points along $\overrightarrow{\text{bc}}$. A positive sign indicates positive rotation according to the right-hand-rule and the magnitude indicates the rotation's effect on the existing torsion angle. The action of the rotation $R(\mathfrak{t})$ on our system $X_t$ with constituent atoms $x_{i} \in \mathbb{R}^3$, can be written as follows:
\begin{align}
\label{eqn:torsion_rotation}
R(\mathfrak{t}) x_i = 
\begin{cases}
  R'(\mathfrak{t}) (x_i - b) + b  & \text{if } i \in \mathfrak{C}(c) \\
  x_i & \text{otherwise}
\end{cases},
\end{align}
where $R'(\mathfrak{t})$ is the rotation matrix defined by the axis-angle representation of $\mathfrak{t}$, and $\mathfrak{C}(c)$ denotes the connected component of $c$ in an identical graph to the molecule \emph{except} with bond $bc$ \emph{removed}! $\mathfrak{C}(c)$ therefore denotes one ``half,'' or ``side,'' of the molecule when hypothetically divided at bond $bc$.

Since every torque points along $\overrightarrow{\text{bc}}$ by design, we need a convention for ordering torsion angle 4-tuples to determine the direction of rotation. We choose atom $c$ to sit on the side of the $bc$ bond with more atoms \citep{jing2022torsional}. To be more specific, if one broke bond $bc$, there would be two connected components $\mathfrak{C}(b)$ and $\mathfrak{C}(c)$. We choose $c$ such that $\lvert \mathfrak{C}(c) \rvert > \lvert \mathfrak{C}(b) \rvert$. (This is the opposite convention to the definition \citet{jing2022torsional} write in their paper, but the same as in their code.) Given ties, we use the current arbitrary index.

\paragraph{Finding torsional degrees of freedom with RDKit}
Sampling conformations by sampling torsion angles requires identifying the torsional degrees of freedom in our target molecule. We use the package RDKit \citep{landrum2013rdkit} to find and represent the 4-tuples for those degrees of freedom. We look for matches to the following substructure: (1) two atoms must be connected by a bond, (2) neither end of that bond can be aromatic, (3) neither end of that bond is merely connected only to a single hydrogen, and (3) only one end of the bond can be part of a ring. Such a substructure can be encoded in a SMARTS string, specifically \texttt{[!\$(*\#*)\&!D1]-\&!@[!\$(*\#*)\&!D1]}.
This matching can be done on directly on the SMILES string of a molecule.

Substructure matching identifies the torsional degrees of freedom with a list of 4-tuples containing atoms in every position. Recall, this representation is not unique, as described above! Furthermore, this search technique typically finds more torsional degrees of freedom than RDKit returns when querying the number of rotatable bonds. RDKit has a specific definition of rotatable bond as defined implicitly in \texttt{CalcNumRotatableBonds}.

\paragraph{Considerations for torsion angles geometry}
We model a stochastic differential equation in the torsion angle space where we can reason with the base distribution more practically and the dimensionality is reduced. However, this comes with a few additional considerations. For this section, consider a molecule with $d_{\text{rot}}$ torsional degrees of freedom.
Due to the periodic boundary conditions on torsion angles, the space is modeled as a flat torus $\mathbb{D} \coloneqq \left[-\pi,\pi\right)^{d_{\text{rot}}}$ where the extrema are identified with zero displacement due to the torus distance function, or so-called logmap. Let $D_{\text{rot}} = \{1, \ldots, d_{\text{rot}} \}$ denote the index set for each torsion angle.

We use a Dirac source distribution centered at $\varphi=0$ with $\varphi \in \mathbb{D}$ at $t=0$. This choice of origin in $\mathbb{D}$ must have a corresponding unique position in cartesian space. Given atoms $(a_i, b_i, c_i, d_i)$ that are constituents of torsion angle $\varphi_i \in \left[-\pi,\pi\right)$ with $i \in D_{rot}$, we choose $\varphi_i=0$ to be the state in Cartesian space where $a_i,d_i$ are closest in terms of Euclidean distance, \ie,
\begin{equation}
    \mathcal{T}(a_i, b_i, c_i, d_i) = 0 \iff d_i = \argmin \lVert a_i - d_i \rVert_{2}.
\end{equation}
The definition is the same for all $i \in D_{rot}$. We rely on intuition and ignore the details to prove that this choice is unique, and that the distance between $a_i$ and $d_i$ is the only necessary component to specify the torsion angle $\varphi_i$.

We now construct the molecule $X_0$ in cartesian space that has $\varphi_i = 0$ for all $i \in D_{rot}$. We first embed the molecular graph in cartesian space to an arbitrary, approximate conformer using RDKit's \texttt{EmbedMolecule} (or \texttt{EmbedMultipleConfs} with \texttt{numConfs=1}). Once embedded, we can identify atoms that are constituents of torsion angles, $(a_i, b_i, c_i, d_i)$ with $i \in D_{rot}$, using \texttt{GetSubstructMatches} and determine the current torsion angles using $\mathcal{T}$. Given any atomic configuration $X$ created using this method, we can reach $X_0$ by iteratively transforming $X$ using $R(-\mathcal{T}(a_i, b_i, c_i, d_i))X$ for every $i$. Iterative application of this transformation is typical \citep{jing2022torsional} in other work and will bring us to $\varphi_i=0$ for all torsion angles. This is the starting point for Torsional Adjoint Sampling! By setting the torsion angles to zero, but retaining the atom types, molecular topology, bond lengths, and bond angles, we have an initial sample to transform with stochastic control. 

Our SDE evolves in $\mathbb{D}$ starting at $\varphi_0=0$ and ending at $\varphi_1$ with the standard behavior of an SDE on a flat torus. However, our drift estimator of the torque is a function of $X_t$, implying we need to update the $X_t$ positions as $\varphi_t$ evolves under the SDE. Given drift estimate $\hat{\mathfrak{t}}(X_t)$ from our model, we update positions $X_{t+1} = \Pi_{i\in D_{\text{rot}}} R^{i}_{X_t}(\hat{\mathfrak{t}}_i + \epsilon_i) X_{t}$ with $\epsilon_i \sim \mathcal{N}(0, \sigma(t)^2 \Delta t)$. $\Delta t$ the increment between steps and $\sigma(t)$ is determined by the geometric noise schedule in \eqref{eqn:geometric_noise_schedule_sigma} as a function of time.

\paragraph{Regression target} 
In Cartesian Adjoint sampling, the regression target is $\frac{1}{\tau} \nabla_{X} g(X_1)$ where the gradient is taken in Cartesian space. For our problem, we need to compute regression target $\frac{1}{\tau} \nabla_{\varphi} g(X_1)$. These two quantities are related by the Jacobian $\frac{d\varphi}{dX}$, namely $\nabla_{\varphi} g(X_1) = \frac{d\varphi}{dX} \nabla_{X} g(X_1)$. We compute this vector-Jacobian product using automatic differentiation of the function $R(\varphi_1)X_0=X_1$ applied to the $\nabla_{X} g(X_1)$ vector determined by our energy model. To clarify, when we clip the gradient for the torsion model, we clip $\nabla_{\varphi} g(X_1)$.

\paragraph{Torsion space posterior}
Adjoint sampling's innovations include fast sampling of points along the control trajectory using a posterior. In the torsion case, we sample $\varphi_t$ from the posterior $\bar{p}_{t \mid 1}^{base}(\varphi \mid \varphi_1)$. This is done by affine transformation of the coordinates $[0, 1]$ of the flat torus distributions from Section~\ref{sec:flat_tori_derivations} to $[-\pi,\pi)$.

\paragraph{Transformation of torque $\mathfrak{t}$}
When we send $X \mapsto -X$, the torque is transformed $\mathfrak{t} \mapsto -\mathfrak{t}$, such quantities are called pseudovectors and their signed magnitudes are called pseudoscalars. Following \citet{jing2022torsional} and the implementation in \citet{e3nn_software}, our networks predict $d_{\text{rot}}$ pseudoscalars without data augmentation.

\subsection{Coverage Recall and Precision Metrics}\label{app:rmsd_metrics}

\textbf{Root Mean Square Deviation (RMSD)} Similarly to \citet{ganea2021geomol} and \citet{jing2022torsional}, we measure the so-called Average Minimum RMSD (AMR) and Coverage (COV) for Precision (P) and Recall (R). When measuring these metrics, we  generate twice as many conformers as provided by CREST reference conformers. For $K = 2L$ let $\{C^*_l\}_{l \in [1,L]}$ and $\{C_k\}_{k \in [1,K]}$ be the sets of ground truth and generated conformers respectively. In our evaluations, we use $L = \text{max}(L', 128)$, where $L'$ is the number of reference conformers given by CREST, taking the lowest energy conformers as a subset. The RMSD metric finds the best average distance between atoms of molecule with respect to the reference molecule, taking into account all possible symmetries.

\textbf{Coverage Recall}
\begin{align}
    \text{COV-R}(\delta) &:= \frac{1}{L} \left| \left\{ l \in \{1,\ldots,L\} : \exists k \in \{1,\ldots, K\}, \quad \text{RMSD}(C_k, C^*_l) < \delta \right\} \right| \\
    \text{AMR-R} &:= \frac{1}{L} \sum_{l \in \{1,\ldots, L\}} \min_{k \in \{1,\ldots, K\}} \text{RMSD}(C_k, C^*_l)
\end{align}
\textbf{Coverage Precision}
\begin{align}
    \text{COV-P}(\delta) &:= \frac{1}{K} \left| \left\{ k \in \{1,\ldots,K\} : \exists l \in \{1,\ldots, L\}, \quad \text{RMSD}(C_k, C^*_l) < \delta \right\} \right| \\
    \text{AMR-P} &:= \frac{1}{K} \sum_{k \in \{1,\ldots, K\}} \min_{l \in \{1,\ldots, L\}} \text{RMSD}(C_k, C^*_l)
\end{align}
where $\delta > 0$ is the coverage threshold.

\subsection{Hyperparameters and Architecture Details for SPICE and GEOM-DRUGS}

\textbf{Cartesian Adjoint Sampling: }
We use an Equivariant Graph Neural Network (EGNN, \citealt{satorras2021n}) with 12 layers and a hidden feature dimension of 128. The model is trained for $5000$ outer-loop iterations, sampling $100$ batches each iteration from the replay buffer. Each GPU maintains its own buffer with max size of 64000 samples and each iterations we generate $128 \times 8$ new molecules and energy evaluations for the buffer across 8 GPUs. The model is trained with a batch size of 64 per GPU and follows a geometric noise schedule with $\sigma_{\min} = 10^{-3}$ and $\sigma_{\max} = 1$. Weight gradient clipping of $10^{20}$ is applied. Additionally we ues the temperature $\tau=5\times 10^{-3}$ and a regularization constant $\alpha=100.0$ for the regularized energy function. The regression target, the temperature-scaled gradient of the energy function is clipped at $\ell^2$ norm 150. We use a step size of $dt = 0.001$ (\ie $1000$ nfe) for Euler-Maruyama SDE integration during training and evaluation.

\textbf{Torsional Adjoint Sampling: }
We use a different equivariant graph neural networked called \texttt{e3nn} \citep{geiger2022e3nn, e3nn_software} with 6 convolutional layers with 32 scalar features and 8 vector features at each layer. There is additionally a specialized layed called a pseudotorque layer by \citet{jing2022torsional} that we use at output to predict the pseudoscalar torque quantities. Our edge and node embeddings have 32 scalar features. The cutoff radius is $5.0$ and we construct 50 Gaussian radial basis functions to embed distances between atoms.
We use 8 GPUs with $1024 \times 8$ initial molecules and each iteration we generate $128 \times 8$ new molecules and evaluate their energy (gradient), which are then added to the GPU buffers.
The model is trained for $3000$ outer-loop iterations, sampling $100$ batches each iteration from the replay buffer with max size 64000. The model is trained with a batch size of 64 and follows a geometric noise schedule with $\sigma_{\text{min}} = \pi \times 10^{-2}$, $\sigma_{\text{max}} = \pi$,  and a ``number of tiles'' truncation parameter of 6 to approximate $\pbase$ on the flat tori (\Cref{app:torsion_sampling}). Weight gradient clipping of $10^{20}$ is applied. Additionally we use the temperature $\tau= 5\times 10^{-3}$ and the direct energy function without any regularization. The regression target, the temperature-scaled gradient of the energy function w.r.t. torsions is clipped at $\ell^2$ norm 200. We use a step size of $dt =  0.01$ (\ie $100$ nfe) for Euler-Maruyama SDE integration during training and evaluation.

\section{Data Preparation for the Conformation Benchmark}
\label{app:data_prep}

Adjoint Sampling, in this context, takes in SMILES strings and outputs conformations. 
Here we explain the splits of the input SMILES strings and representative coordinates samples of conformations. Additionally, we explain the processes we used to produce, relax, and deduplicate our reference atomic conformations with RDKit \citep{landrum2013rdkit}, CREST \citep{pracht2024crest}, and ORCA \citep{orca} by following an extremely similar procedure to that of GEOM-DRUGS \citep{axelrod2022geom}.

We are releasing the data we used in this paper as a challenging \emph{Conformation Benchmark} with the goal of fostering development of scalable, amortized sampling algorithms. The data has three parts:
\begin{itemize}
    \item A training split of 24,477 SMILES strings from SPICE \citep{eastman2023spice}. These define the molecular topology for the target molecule, but do not have any atomic coordinate information. The molecules have differing levels of flexibility with between 0 and 18 rotatable bonds\footnote{Rotatable bonds are defined by \texttt{RDKit}'s function \texttt{CalcNumRotatableBonds} \citep{landrum2013rdkit}.}.
    \item A test split\footnote{The test split is disjoint from the train split and their union is all of SPICE.} of 80 SMILES strings from SPICE divided into groups of 10, each with between 3 and 10 rotatable bonds. We also release 44,448 DFT annotated and geometry optimized conformations.
    \item A test split of 80 SMILES strings from GEOM-DRUGS divided into groups of 10, each with between 3 and and 10 rotatable bonds. We identify 7024 conformations corresponding to those SMILES strings.
\end{itemize}
The details of are described below in \Cref{app:data_prep_spice}. Since we are repackaging data from SPICE and GEOM-DRUGS, our main contributions include (a) the organization of the benchmark and (b) the computational effort of finding 44,448 conformations from our SPICE test split using a similar procedure to that of GEOM-DRUGS.

One utilizes the benchmark by applying a sampling algorithm to find conformers of a given molecular topology and determining whether the algorithm recovers the samples we provided. In \Cref{fig:recall_plot}, we report this as Coverage Recall \% versus RMSD. We additionally provide tabulated values in \Cref{tab:conformer_generation}.
The conformers we provide should be treated as held-out data generated using an effective and expensive standard method. 
We do not claim to have comprehensively sampled the conformation space for these molecules, but we argue that they are a representative sample given modern methods. Furthermore, since we only release local minima, the sample does not come from a Boltzmann distribution at a fixed temperature, \emph{i.e.} the samples are not drawn from a probability distribution. There is no sense of a conformer being more ``likely'' than another in a probabilistic manner in our data. Instead, the benchmark represents a loosely defined notion of sampling by finding many low energy conformations in an attempt to cover the configuration space. 

\subsection{SPICE Dataset}
\label{app:data_prep_spice}
We used the PubChem subset of the SPICE dataset~\citep{eastman2023spice} to train and evaluate our models. This subset includes small, drug-like molecules comprised of between 18 and 50 atoms with elements Br, C, Cl, F, H, I, N, O, P, and S.
\begin{figure}[htb]
    \centering
    \begin{subfigure}[m]{0.48\textwidth}
        \includegraphics[width=\textwidth]{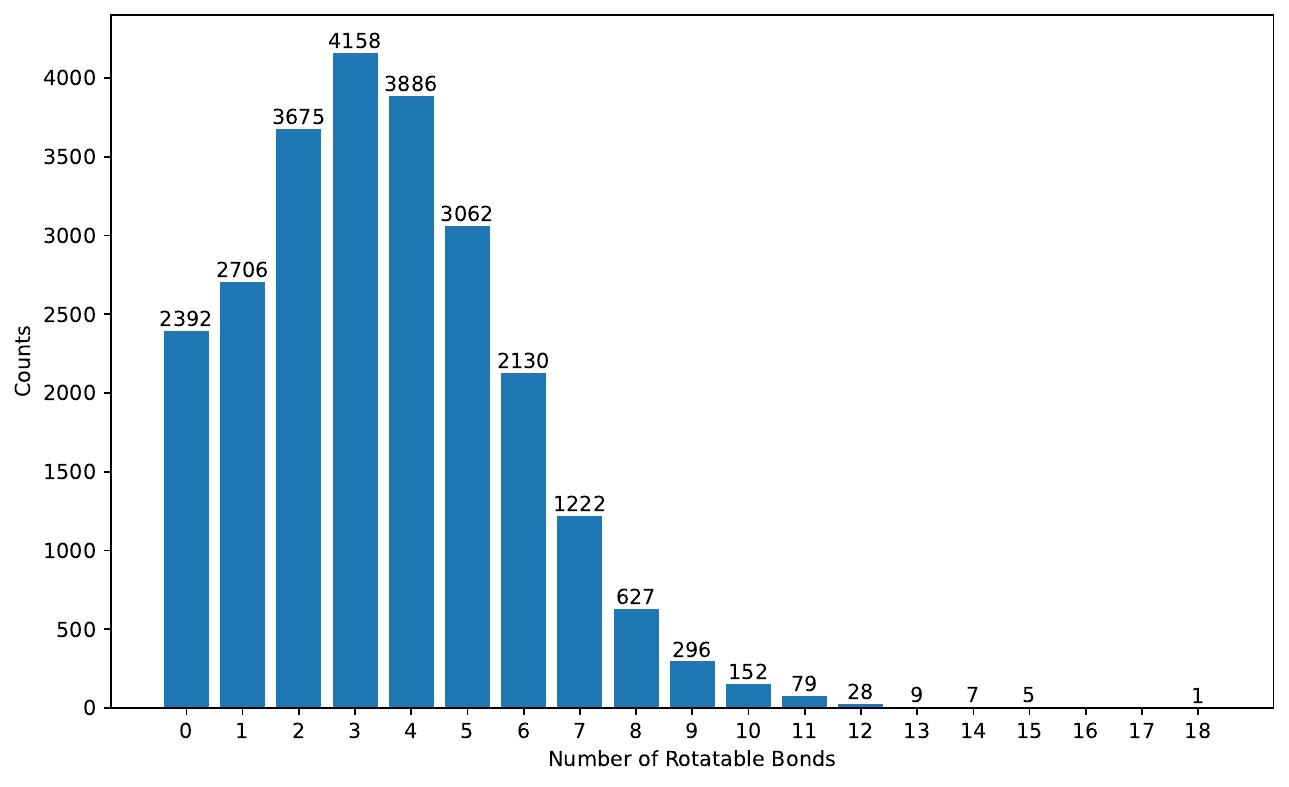}
        \caption*{\texttt{CalcNumRotatableBonds}}
    \end{subfigure}
    \hfill
    \begin{subfigure}[m]{0.48\textwidth}
        \includegraphics[width=\textwidth]{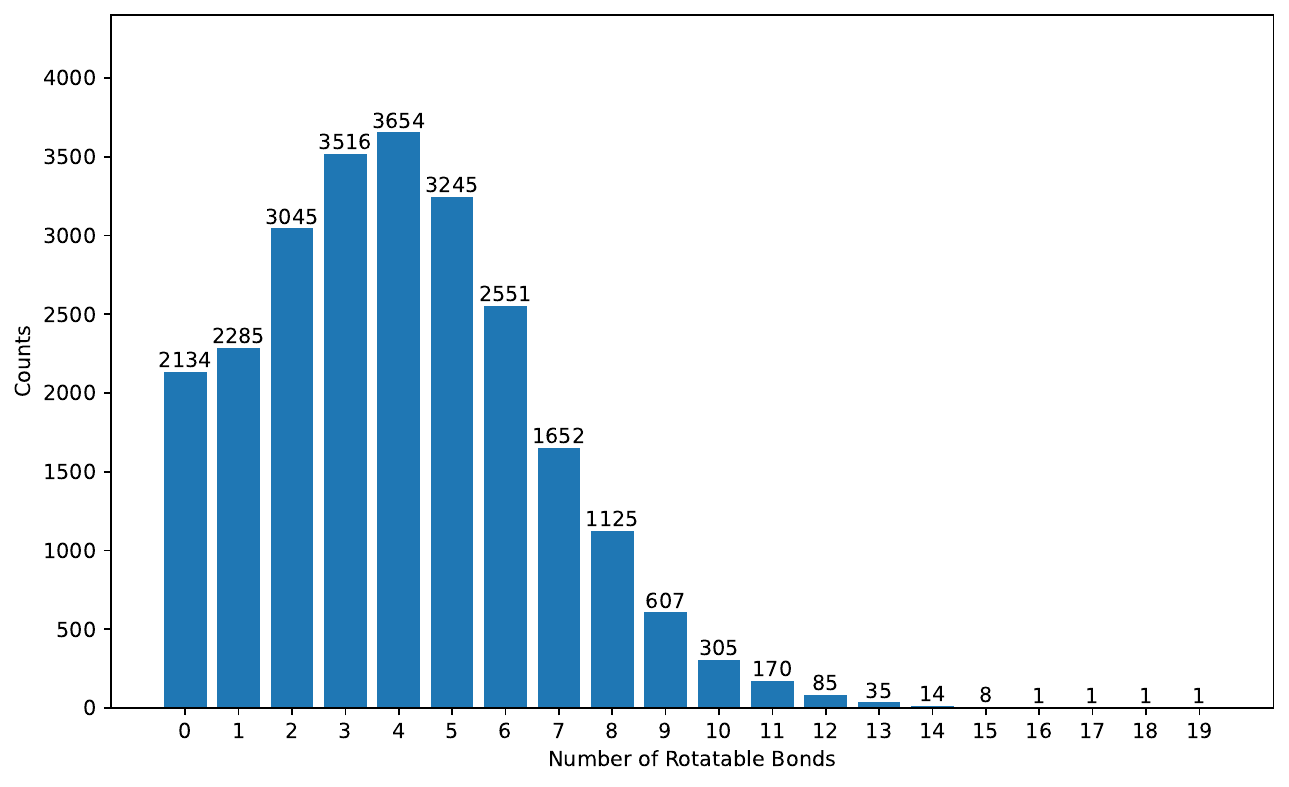}
        \caption*{\texttt{GetSubstructMatches} with SMARTS}
    \end{subfigure}
    \caption{Histograms of rotatable bonds and torsional degrees of freedom in the SPICE training set. Left: Number of rotatable bonds determined using RDKit's \texttt{CalcNumRotatableBonds}. Right: Number of torsional degrees of freedom determined by \texttt{GetSubstructMatches} with SMARTS string \texttt{[!\$(*\#*)\&!D1]-\&!@[!\$(*\#*)\&!D1]}.}
    \label{fig:hist_spice}
\end{figure}
We are interested in conformations of molecules, implying that we care about the rotational degrees of freedom. We identify two different sets describing these degrees of freedom: \emph{number of rotatable bonds} and \emph{torsional degrees of freedom}, or \emph{torsions} for short. In our case, rotatable bonds $\subseteq$ torsional degrees of freedom. We compute the number of rotatable bonds using the RDKit package \texttt{CalcNumRotatableBonds}. We count the number of torsional degrees of freedom using the SMARTS string \texttt{[!\$(*\#*)\&!D1]-\&!@[!\$(*\#*)\&!D1]} and \texttt{GetSubstructMatches}, also from RDKit. We group molecules into classes according to the number of rotatable bonds, but we provide the torsional degrees of freedom to our Adjoint Sampling Torsion model. The distributions of rotatable bonds and the comparison between the two sets of rotational degrees of freedom can be found in \Cref{fig:hist_spice} and \Cref{fig:heatmap}. We find fewer rotatable bonds than torsions matching our substructure query, which is expected. More information about our torsional model is in \Cref{app:conformers_and_torsions}.

We randomly subsampled 10 molecules from each class of molecules with a fixed number of rotatable bonds, in the range of 3-10, to create the test set. We performed pre-processing and CREST~\citep{pracht2024crest} on the resulting 80 molecules. 
We also removed 30 additional molecules from the training set; 10 from each of the classes where the number of rotatable bonds was between 0-2. These do not appear in the test set since they are quite simple. Our procedure follows GEOM-DRUGS~\citep{axelrod2022geom}.

During pre-processing, we first used RDKit to generate 50 conformers with \texttt{EmbedMultipleConfs} using parameters \texttt{pruneRmsThresh=0.01, maxAttempts=5, useRandomCoords=False} corresponding to a pruning threshold of similar conformers of 0.01 \AA, a maximum of five embedding attempts per conformer, coordinate initialization from the eigenvalues of the distance matrix, and a random seed. If no conformers were successfully generated then \texttt{numConfs} was increased to 500. Afterwards, the conformers were optimized in the RDKit default MMFF force field. We deduplicated the conformers using RDKit's \texttt{GetBestRMS}, unless the calculation took longer than 48 hours then we switched to RDKit's \texttt{AlignMol}. (97 used \texttt{GetBestRMS}, 3 used \texttt{AlignMol}.) After removing the duplicate conformers that exhibited an RMSD $< 0.1 \mathring{A}$, the ten conformers with the lowest energy were further optimized with an approximate energy model known as \emph{extended tight binding} (xTB) \citep{grimme2017robust, friede2024dxtb}. The conformer with the lowest xTB energy was used as input to the CREST simulation. We used the default hyperparameters from CREST version 3.0.2 including a 6 kcal/mol cutoff on final conformers. Any SMILES strings containing slash or backslash indicate a cis/trans stereochemistry and were skipped since their SMILES string is not preserved during the above procedure.

The output of the above procedure was a thoroughly sampled potential energy surface by the standards of modern methods; however, the structures are optimized using the xTB level of accuracy. We took the process a step further and performed geometry optimizations using DFT with ORCA on each conformer in the test set. Specifically, we took the output of CREST and ran ORCA relaxation with settings:  
\begin{verbatim}
Opt wB97M-V def2-tzvpd RIJCOSX def2/J NoUseSym DIIS NOSOSCF NormalConv DEFGRID2 
NormalPrint AllPop
%scf Convergence Tight maxiter 300 end
%elprop Dipole true Quadrupole true end
%output Print[P_ReducedOrbPopMO_L] 1 Print[P_ReducedOrbPopMO_M] 1 
Print[P_BondOrder_L] 1 Print[P_BondOrder_M] 1 Print[P_Fockian] 1 Print[P_OrbEn] 2 end
%geom MaxIter 100 end
%scf THRESH 1e-12 TCUT 1e-13 end
\end{verbatim}
Some lines were broken in the latex version of the code snippet above due to space limitations. These results are run at a very high level of DFT theory. Further geometry optimization may be desired at \texttt{DEFGRID3} with \texttt{TightOpt}; however, such an extreme precision was prohibitive given our resources available at the time.

After the geometry optimization, we utilized RDKit's \texttt{GetAllConformerBestRMS} and a threshold of 0.1 \AA\xspace to identify pairwise duplicates and remove them from the dataset. The average number of reference conformers for each number of rotatable bonds is presented in the left side of  \Cref{fig:spice_drugs_confs}.

\looseness=-1000
We investigated the reason that the 80 molecule test set from SPICE generated an order of magnitude more conformers than the 80 molecule GEOM-DRUGS test set. We found, after replicating the GEOM-DRUGS conformers, that the molecules in the SPICE test set are inherently more flexible than the ones from the GEOM-DRUGS test set. GEOM-DRUGS molecules tended to have more rings and hydrogen bonds, leading to fewer conformations. In other words, SPICE tended to have a flatter potentially energy surface. Since we followed the GEOM-DRUGS procedure, we applied the same 6 kcal/mol cutoff in CREST and kept many more conformers from SPICE. For Coverage Recall \%, we capped the reference conformers per molecule at 512.

\begin{figure}[htb]
    \centering
    \includegraphics[width=0.6\textwidth]{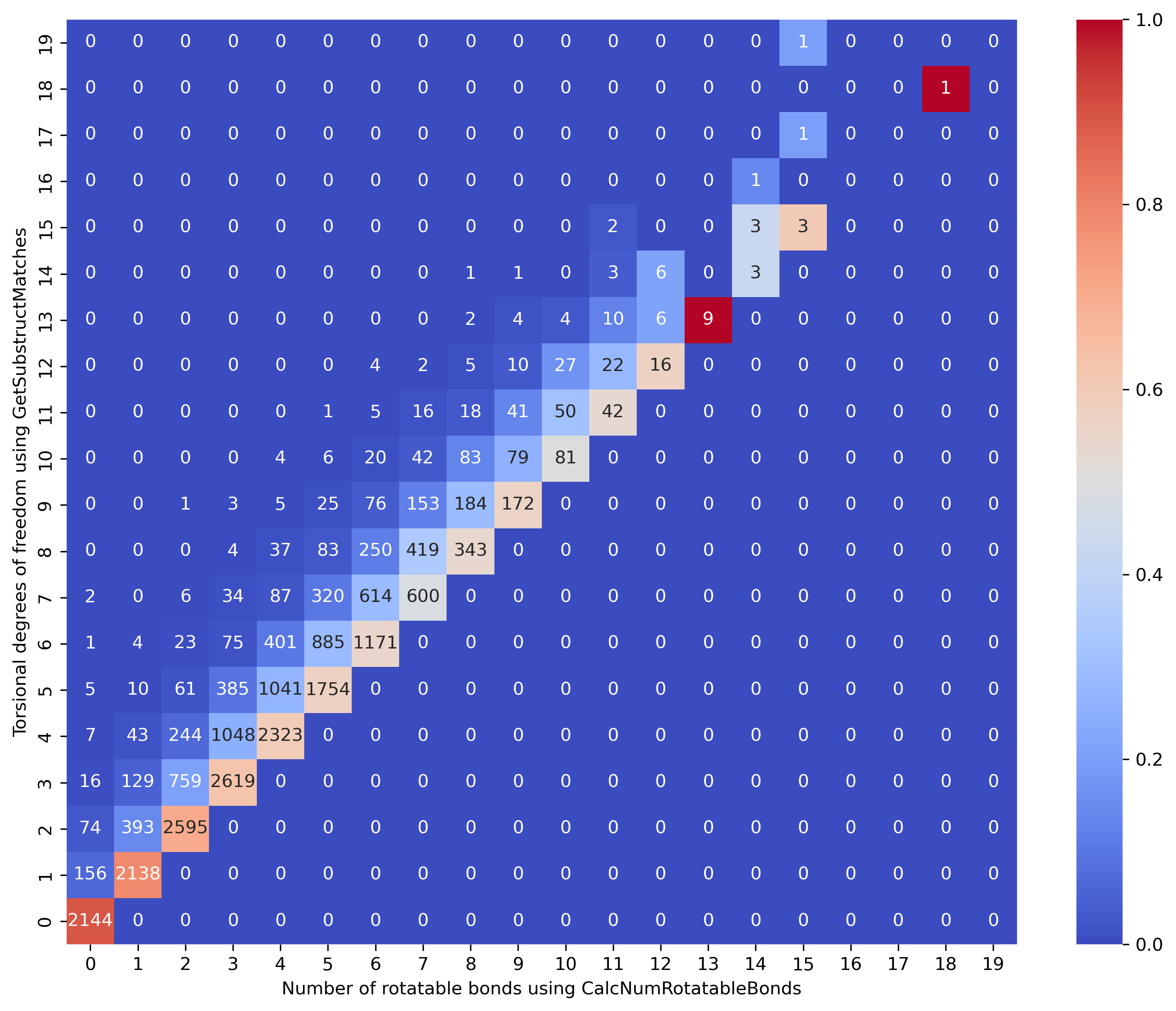}
    \caption{Correlation heatmap between the torsional degrees of freedom (y-axis) versus the number of rotatable bonds (x-axis) in the combined SPICE train and test sets. Torsional degrees of freedom are identified using \texttt{GetSubstructMatches} while rotatable bonds are identified using \texttt{CalcNumRotatableBonds}. There are always more or equal torsional degrees of freedom than rotatable bonds. Colors indicate a conditional probability: for each number of rotatable bonds (x-axis), the color shows the fraction of molecules with a given number of torsional degrees of freedom (y-axis). Numbers in cells show the raw counts.}
    \label{fig:heatmap}
\end{figure}

\begin{figure}[htb]
    \centering
    \begin{subfigure}[m]{\linewidth}
    \begin{subfigure}[m]{0.48\textwidth}
        \includegraphics[width=\textwidth]
        {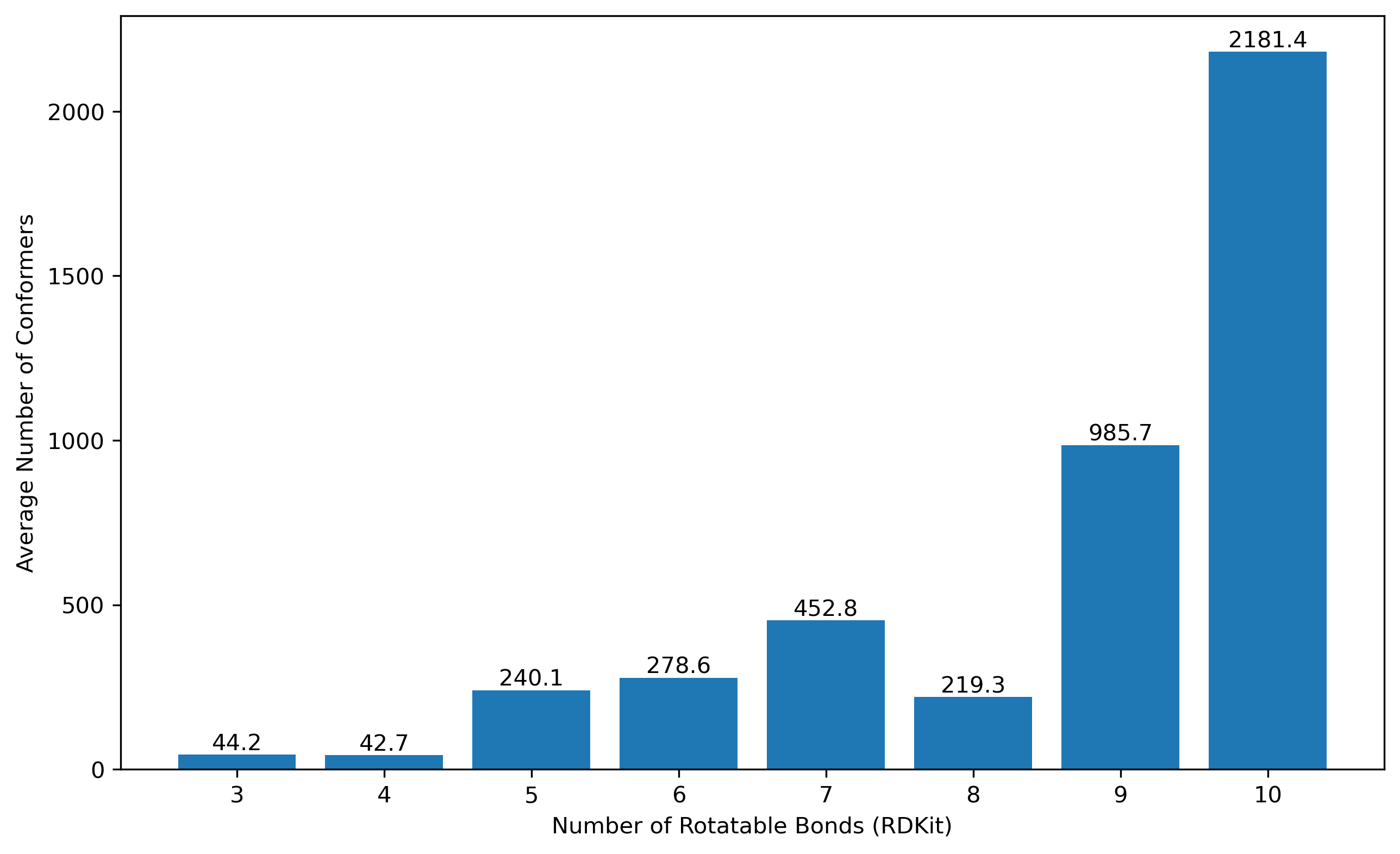}
        \caption*{SPICE}
    \end{subfigure}
    \hfill
    \begin{subfigure}[m]{0.48\textwidth}
        \includegraphics[width=\textwidth]{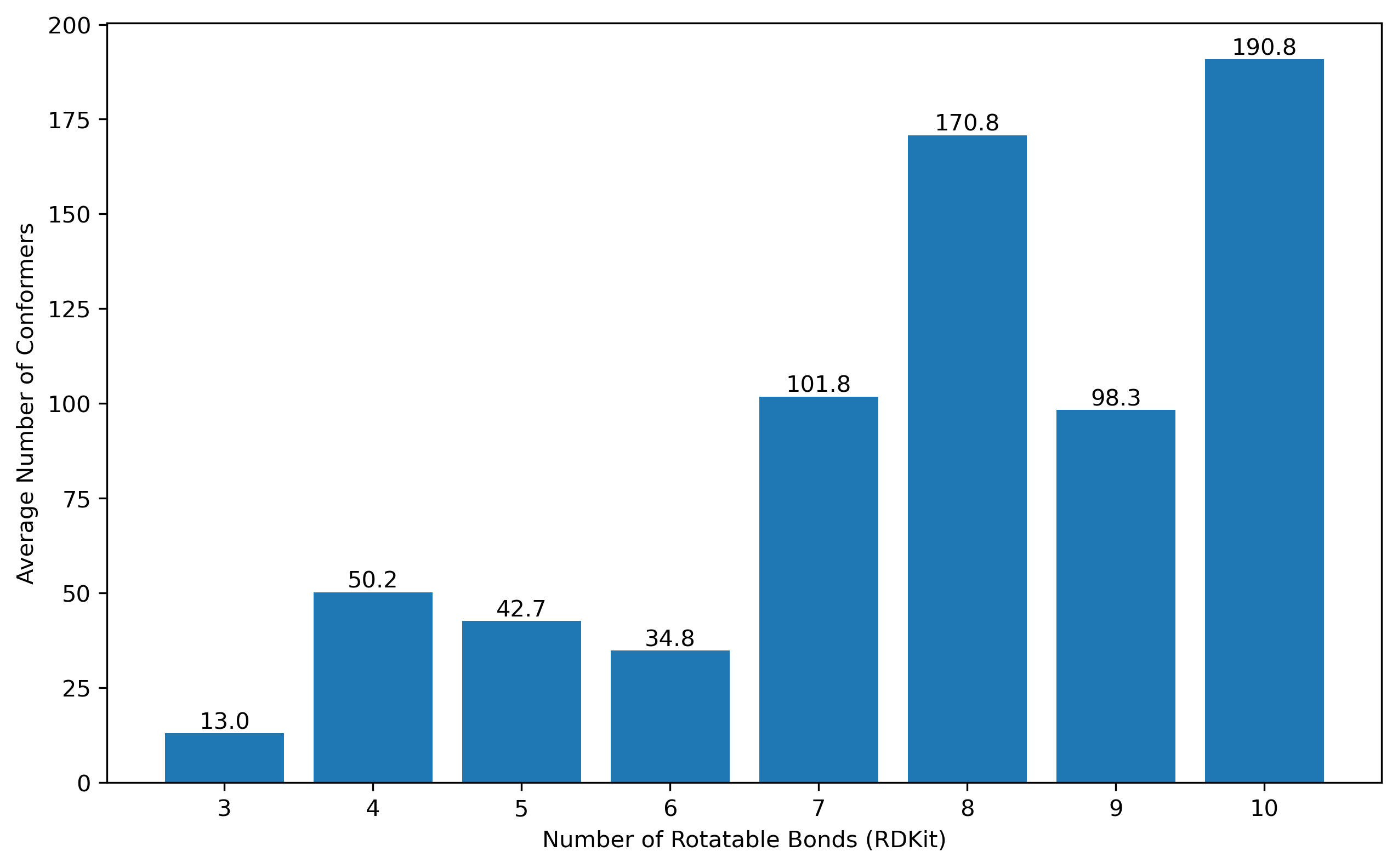}
        \caption*{GEOM-DRUGS}
    \end{subfigure}
    \end{subfigure}
    \caption{Average number of conformers in the SPICE (\emph{left}) and GEOM-DRUGS (\emph{right}) test sets vs number of rotatable bonds calculated by RDKit's function \texttt{CalcNumRotatableBonds}. Note the different y-axes between the figures.}
    \label{fig:spice_drugs_confs}
\end{figure}

\begin{figure}[htb]
    \centering
    \begin{subfigure}[m]{\linewidth}
    \begin{subfigure}[m]{0.50\textwidth}
        \includegraphics[width=\textwidth]{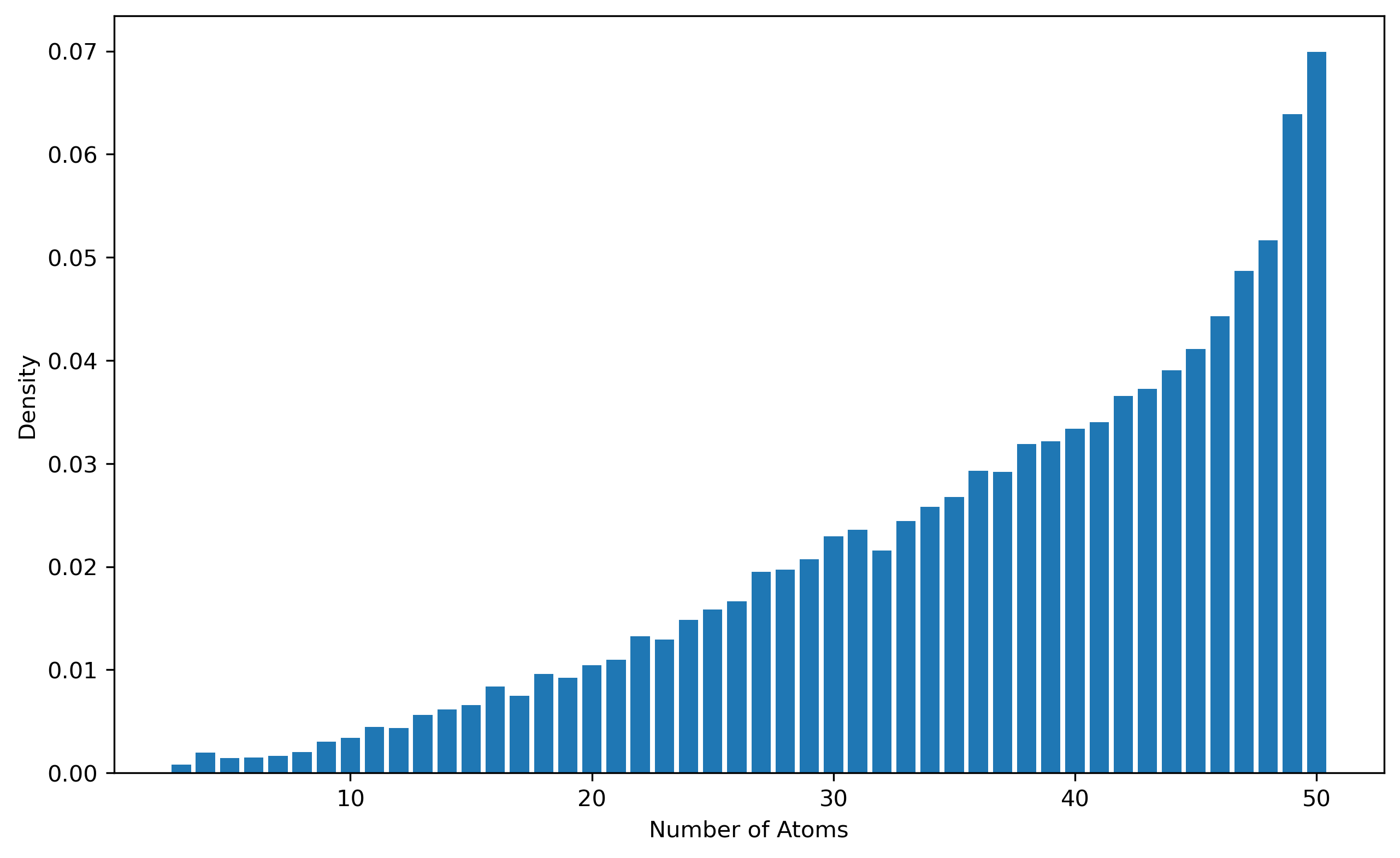}
        \caption*{SPICE: number of atoms.}
    \end{subfigure}
    \hfill
    \begin{subfigure}[m]{0.50\textwidth}
        \includegraphics[width=\textwidth]{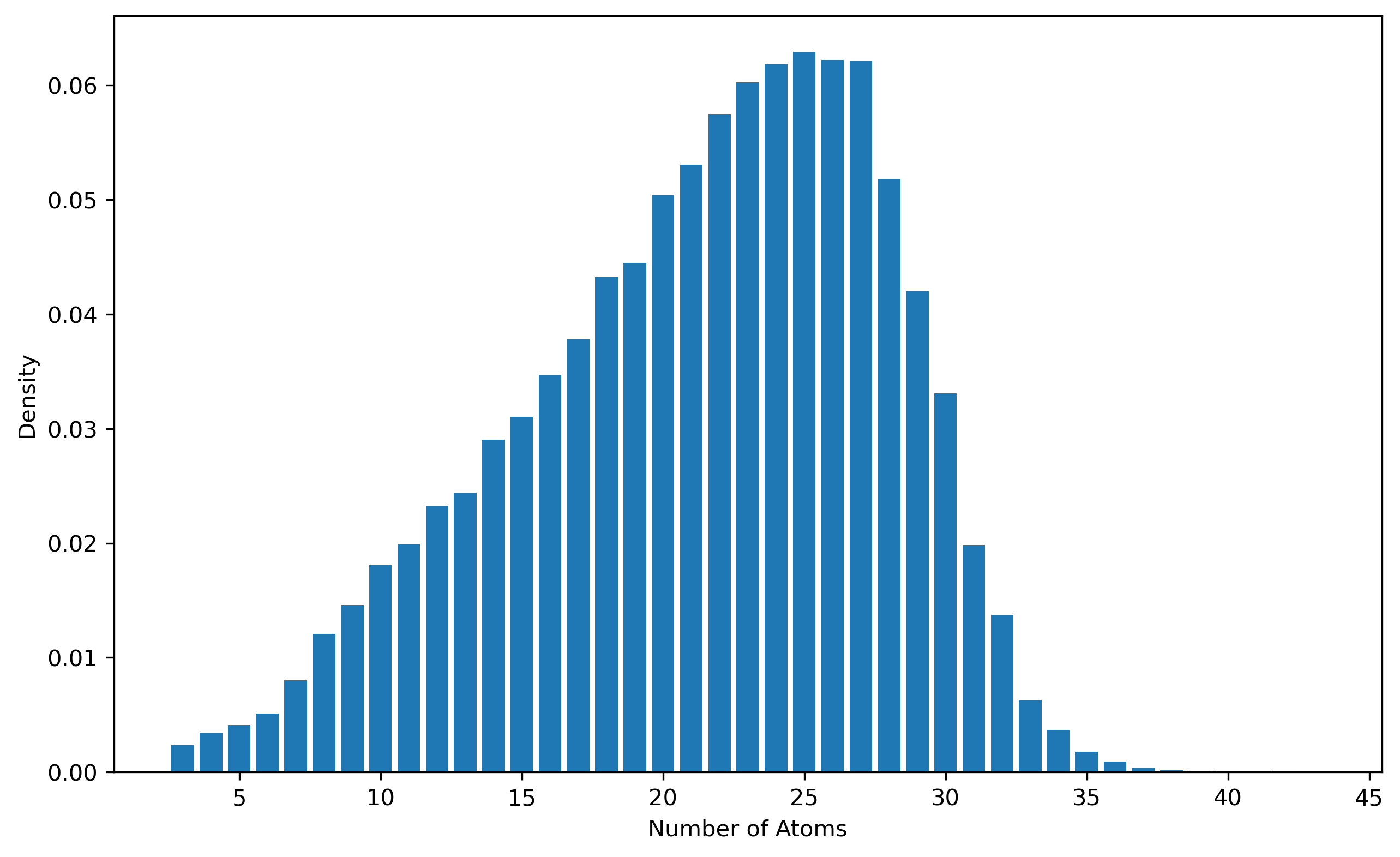}
        \caption*{SPICE: number of heavy atoms (excludes Hydrogen).}
    \end{subfigure}
    \end{subfigure}
    \caption{Atom counts appearing in the SPICE (train + test) dataset: (\emph{left}): Frequency of total atoms. (\emph{right}): Frequency of only heavy atoms (excluding Hydrogen). Our conformer generation predicts positions of \emph{all} atoms.}
    \label{fig:spice_atom_counts}
\end{figure}


\subsection{GEOM-DRUGS Dataset}
\label{app:geom_drugs_info}
We also evaluated our models on the GEOMDRUGS dataset~\citep{axelrod2022geom}. We took the test set from Torsional Diffusion~\citep{jing2022torsional}, and excluded molecules that have elements other than Br, C, Cl, F, H, I, N, O, P, and S. From the remaining molecules, we randomly sampled 10 molecules for each number of rotatable bonds in the range of 3-10 (calculated by RDKit's \texttt{CalcNumRotatableBonds}). The smallest molecule was comprised of 19 atoms and the largest of 65, so slightly larger than SPICE.

We use the conformers in the GEOM-DRUGS dataset as the reference for evaluation. The conformers are represented as RDKit molecule object. To use the same format as the SPICE dataset, we extracted the Cartesian coordinates from the RDKit molecules and wrote them into \texttt{.xyz} files for each molecule.
The average number of reference conformers for each number of rotatable bonds are presented in the right side of \Cref{fig:spice_drugs_confs}.

\clearpage

\section{Additional Experimental Results}

\subsection{Energy Histograms for Synthetic Energy Experiments}\label{app:energy_histograms}
In this section, we provide additional qualitative comparisons between the energy distributions generated by our synthetic energy benchmarks and those obtained via ground-truth MCMC sampling. \Cref{fig:energy_histograms} illustrate histograms for three distinct systems --- DW4, LJ13, and LJ55. These results clearly indicate that Adjoint Sampling is more effective at avoiding high-energy regions, especially in the case of LJ55, which features the highest-dimensional energy surface.

\begin{figure*}[h]\label{fig:energy_histograms}
    \centering
    \begin{subfigure}[t]{0.30\textwidth}
        \centering
        \includegraphics[height=0.95in]{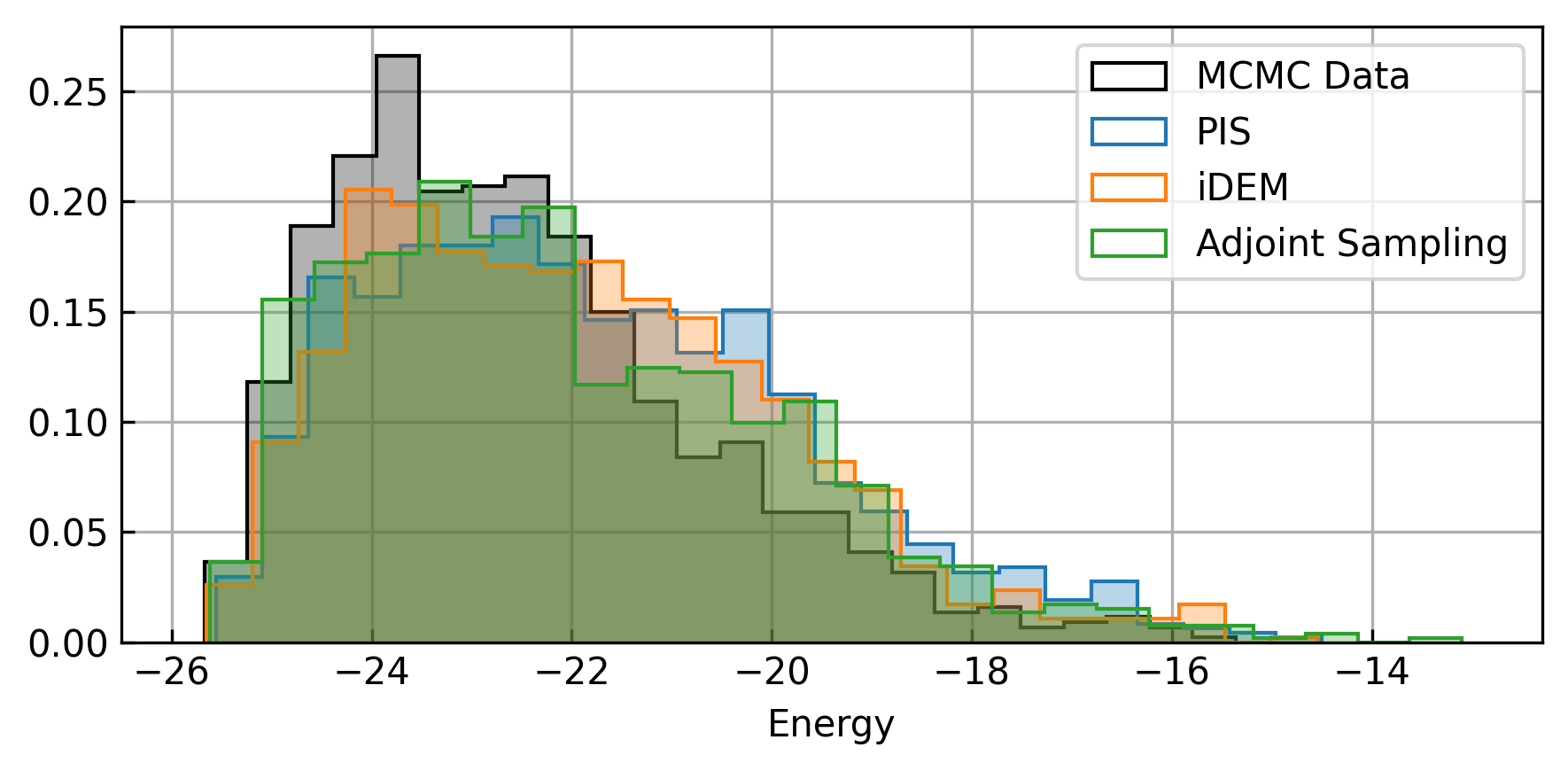}
        \caption{DW4}
    \end{subfigure}%
    ~ 
    \begin{subfigure}[t]{0.30\textwidth}
        \centering
        \includegraphics[height=0.95in]{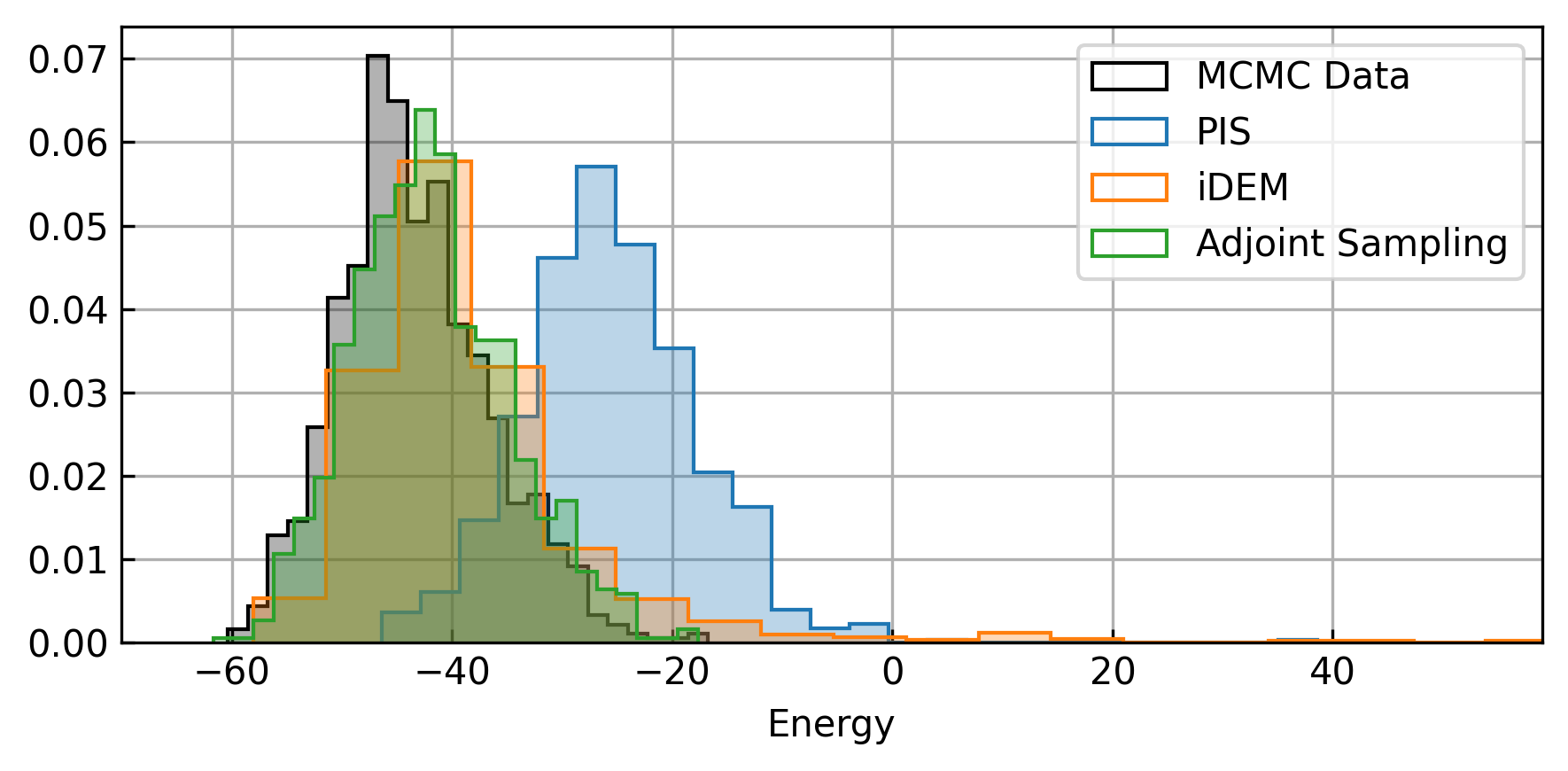}
        \caption{LJ13}
    \end{subfigure}
    ~
    \begin{subfigure}[t]{0.30\textwidth}
        \centering
        \includegraphics[height=0.95in]{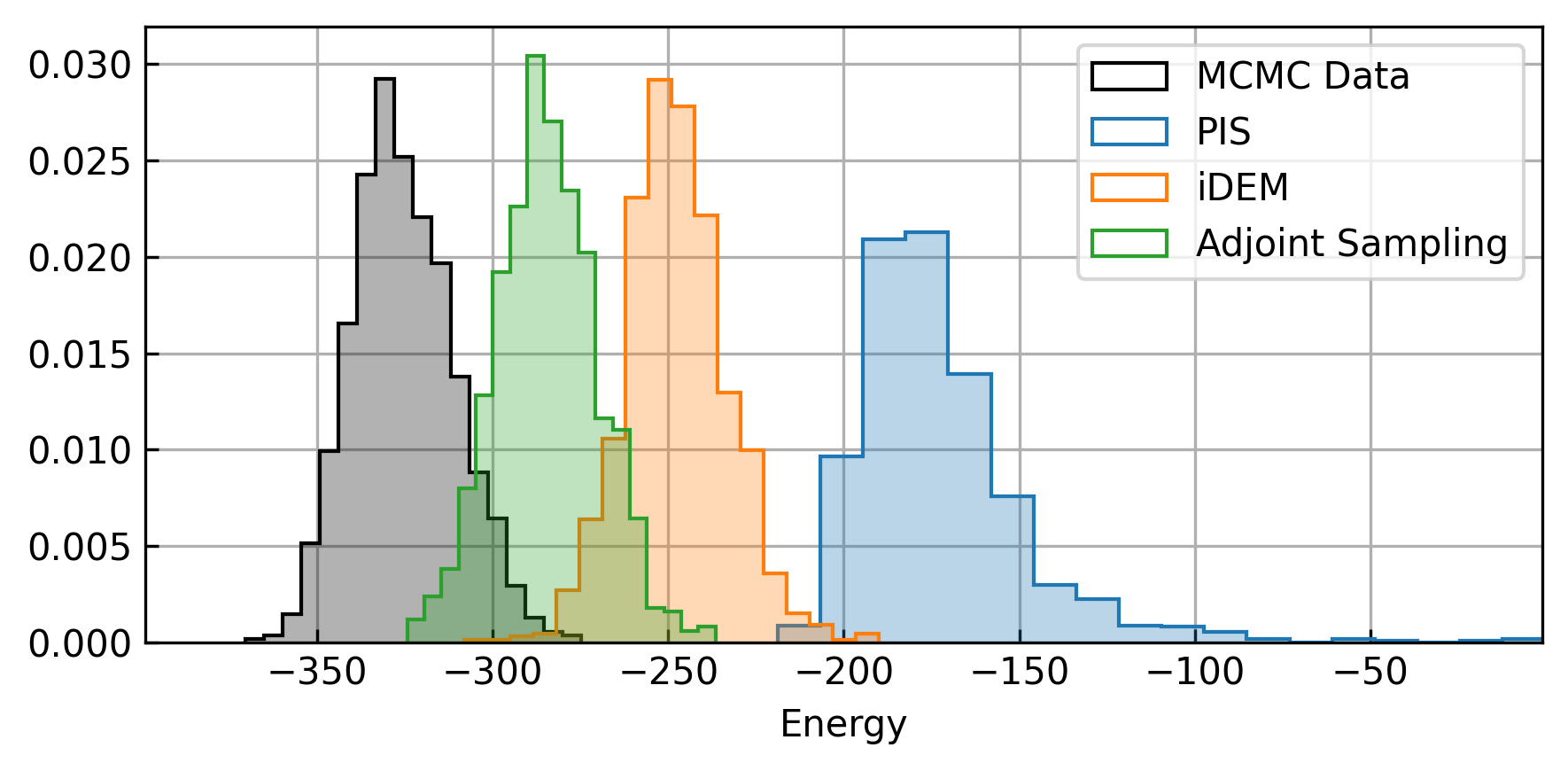}
        \caption{LJ55}
    \end{subfigure}
    \caption{Energy Histograms of the Synthetic energy benchmark overlayed on the ground-truth MCMC generated samples.}
\end{figure*}

\subsection{Runtime Analysis of Molecular Conformer Generation}

This section further examines the computational efficiency of our proposed Adjoint Sampling technique through a detailed runtime analysis of molecular conformer generation task. 

As illustrated in \Cref{fig:runtime}, Adjoint Sampling overcomes two significant computational bottlenecks found in prior approaches, namely the intensive SDE simulation and the costly energy evaluation steps. By mitigating these challenges, our method is able to perform substantially more gradient updates within the same execution time, leading to enhanced overall performance in tasks with expensive energy function evaluation, such as conformer generation.

\begin{figure}[h]
    \centering
    \includegraphics[width=0.8\linewidth]{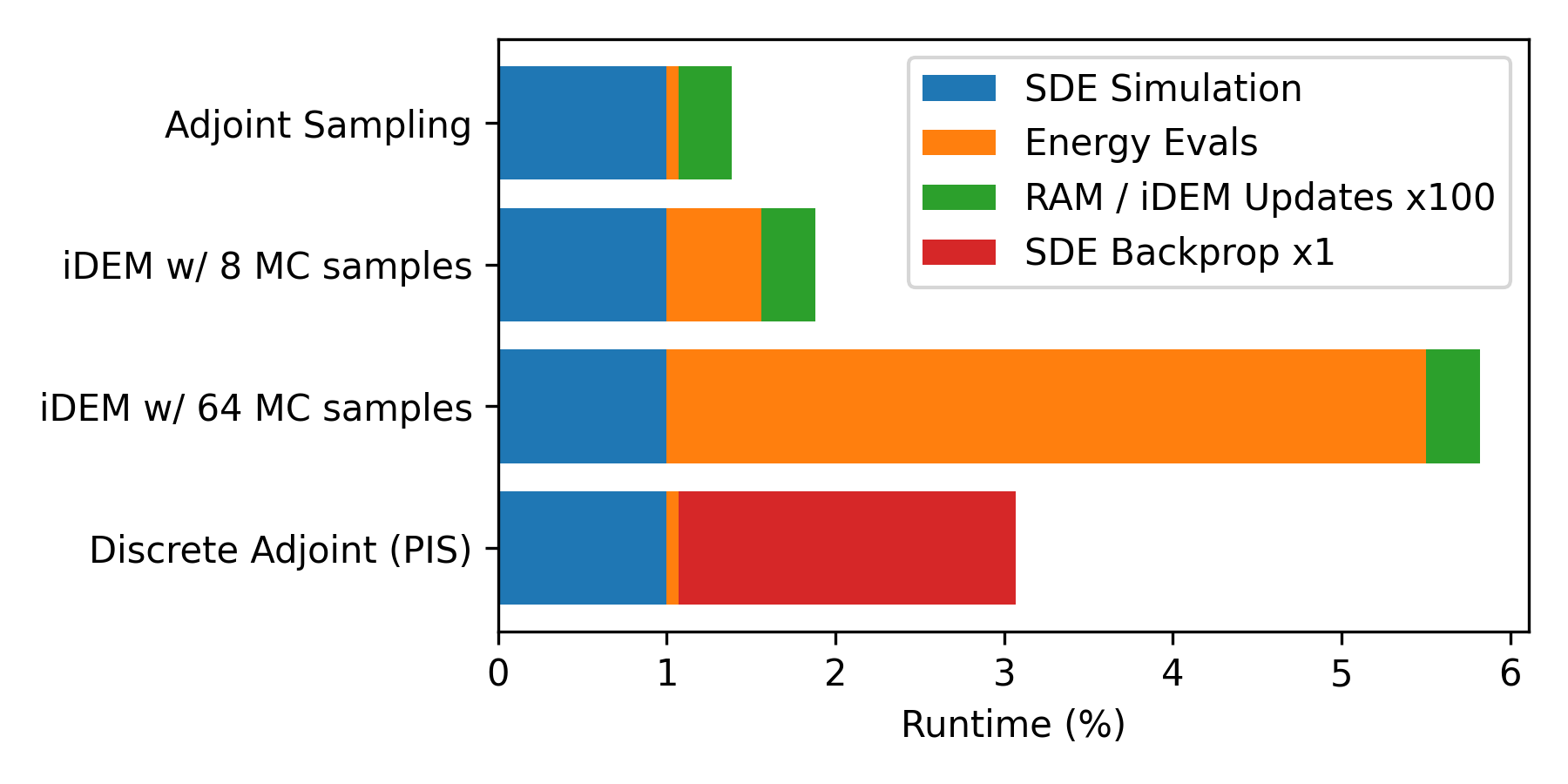}
    \caption{Conformer Generation Run-time: A run-time analysis breakdown of Adjoint Sampling compared to selected baselines at the level of gradient updates. Adjoint Sampling overcomes the two primary bottle-knecks of previous methods: SDE simulation and energy evaluation, allowing Adjoint Sampling to perform  far more gradient updates in the same run-time.}
    \label{fig:runtime}
\end{figure}

\subsection{Choice of energy function}

\begin{table}[h]
\caption{
Test set MAE for eSEN-900k and MACE-OFF-23-M on the SPICE-MACE-OFF test dataset. Energy (\textbf{E}) MAE is in meV/atom. Force (\textbf{F}) MAE is in meV/\AA.
\label{tab:spice_mace_off}
}
\resizebox{\textwidth}{!}{%
\begin{tabular}{lcccccccccccccc}
\toprule
& \multicolumn{2}{c|}{PubChem} & \multicolumn{2}{c|}{DES370K M.} & \multicolumn{2}{c|}{DES370K D.} & \multicolumn{2}{c|}{Dipeptides} & \multicolumn{2}{c|}{Sol. AA} & \multicolumn{2}{c|}{Water} & \multicolumn{2}{c}{QMugs} \\
Model & \textbf{E} & \textbf{F} & \textbf{E} & \textbf{F} & \textbf{E} & \textbf{F} & \textbf{E} & \textbf{F} & \textbf{E} & \textbf{F} & \textbf{E} & \textbf{F} & \textbf{E} & \textbf{F} \\
\midrule
eSEN-900k & 0.66 & 14.37 & 0.50 & 5.32 & 0.52 & 5.62 & 0.37 & 8.52 & 1.41 & 17.16 & 0.71 & 12.66 & 0.44 & 14.56 \\
MACE-OFF-23-M & 0.91 & 20.57 & 0.63 & 9.36 & 0.58 & 9.02 & 0.52 & 14.27 & 1.21 & 23.26 & 0.76 & 15.27 & 0.69 & 23.58 \\
\bottomrule
\end{tabular}
}
\end{table}

The eSEN energy model~\citep{fu2025learning}, used in this work, has been shown to accurately predict DFT energy and forces on the SPICE dataset. We use an eSEN network with 2 layers, $L_{\mathrm{max}}=2, M_{\mathrm{max}}=0$, 64 channels, and a radius cutoff of $4.5$\ \AA. This model has around 900k trainable parameters. To demonstrate the robustness of the adjoint matching algorithm to the choice of the energy function, we conduct additional molecular conformation sampling experiments using adjoint matching on a pretrained MACE-OFF-23 energy model~\citep{batatia2022mace, kovacs2023mace}. The test-set prediction errors of eSEN-900k and the MACE-OFF-23-M~\citep{kovacs2023mace} model are reported in \Cref{tab:spice_mace_off}. The same energy model is used for both relaxation and for optimizing adjoint sampling, where applicable. \Cref{tab:energy_function_choice} shows that adjoint matching with MACE-OFF-M and eSEN achieve similar performance in generating molecular conformation.

\begin{table*}[ht]
\caption{Recall and precision metrics for conformer generation using different energy models. Coverage values are for thresholds of 1.25\AA. Standard deviations are computed across molecules in the test set. (These experiments were done with an earlier version of the code.)}
\centering
\renewcommand{\arraystretch}{1.3}
\resizebox{0.5\textwidth}{!}{%
\begin{tabular}{@{} l l cccc}
    \toprule
    & & \multicolumn{4}{c}{{SPICE}} \\
    \cmidrule(lr){3-6}
    & & \multicolumn{2}{c}{{Recall} } & \multicolumn{2}{c}{{Precision}} \\
    & Method & Cov. $\uparrow$ & AMR $\downarrow$ & Cov. $\uparrow$ & AMR $\downarrow$ \\
    \midrule
    & eSEN 
& 80.28{\color{gray}\tiny$\pm$27.68} & 0.96{\color{gray}\tiny$\pm$0.28} & 47.97{\color{gray}\tiny$\pm$32.24} & 1.25{\color{gray}\tiny$\pm$0.37} \\
& MACE 
& 74.93{\color{gray}\tiny$\pm$30.36} & 1.02{\color{gray}\tiny$\pm$0.27} & 50.22{\color{gray}\tiny$\pm$32.22} & 1.28{\color{gray}\tiny$\pm$0.39} \\
& eSEN (+pretrain) 
& 88.58{\color{gray}\tiny$\pm$19.96} & 0.85{\color{gray}\tiny$\pm$0.25} & 65.12{\color{gray}\tiny$\pm$29.90} & 1.13{\color{gray}\tiny$\pm$0.34}  \\
& MACE (+pretrain) 
& \cellhi 89.62{\color{gray}\tiny$\pm$19.25} & \cellhi 0.84{\color{gray}\tiny$\pm$0.25} & \cellhi 67.07{\color{gray}\tiny$\pm$29.48} & \cellhi 1.11{\color{gray}\tiny$\pm$0.34} \\

    \midrule
    \parbox[t]{2mm}{\multirow{4}{*}{\rotatebox[origin=c]{90}{\color{cadetblue}w/ relaxation}}} 
    & eSEN 
& 91.49{\color{gray}\tiny$\pm$18.27} & 0.71{\color{gray}\tiny$\pm$0.29} & 60.00{\color{gray}\tiny$\pm$29.97} & 1.02{\color{gray}\tiny$\pm$0.38} \\
& MACE 
& 91.06{\color{gray}\tiny$\pm$17.94} & 0.72{\color{gray}\tiny$\pm$0.28} & 65.19{\color{gray}\tiny$\pm$28.27} & 1.05{\color{gray}\tiny$\pm$0.37} \\
& eSEN (+pretrain) 
& \cellhi 96.36{\color{gray}\tiny$\pm$\hphantom{0}8.92} & \cellhi 0.60{\color{gray}\tiny$\pm$0.25} & 76.46{\color{gray}\tiny$\pm$26.22} & \cellhi 0.92{\color{gray}\tiny$\pm$0.35} \\
& MACE (+pretrain) 
& 95.73{\color{gray}\tiny$\pm$11.32} &  0.61{\color{gray}\tiny$\pm$0.25} & \cellhi 77.08{\color{gray}\tiny$\pm$25.73} & \cellhi 0.92{\color{gray}\tiny$\pm$0.34} \\

    \bottomrule
\end{tabular}
}
\label{tab:energy_function_choice}
\end{table*}

\clearpage

\subsection{Ablating Reciprocal Projection}\label{app:adjoint_ablation}
Here we present the same ablation of the Reciprocal Projection as performed for the Synthetic experiments in~\Cref{tab:synthetic}, but for the molecular conformer generation task. Recall that instead of using the Reciprocal projection (RP) to sample $X_t$ given $X_1$, we simply store sample pairs $(X_1, X_t)$ in the buffer and train on the Adjoint Matching objective \eqref{eq:am}. We call this \emph{Adjoint Sampling w/o RP}, which helps demonstrate the effectiveness of the Reciprocal projection. For the purpose of this ablation, we only consider the MACE energy and do not perform any post generation relaxation.

\begin{table*}[ht]
\caption{Recall and precision metrics for Adjoint Sampling with and without use of the Reciprocal Projection (RP). Coverage values are for thresholds of 1.25\AA. Standard deviations are computed across molecules in the test set. We see that the Reciprocal Projection results in significant improvements across all metrics, especially when generalizing to the unseen dataset GEOM-DRUGS. No pretraining and no relaxation is performed post generation. (These experiments were done with an earlier version of the code.)}
\centering
\renewcommand{\arraystretch}{1.3}
\resizebox{\textwidth}{!}{%
\begin{tabular}{@{} l l cccc cccc }
    \toprule
    & & \multicolumn{4}{c}{{SPICE}} & \multicolumn{4}{c}{{GEOM-DRUGS}} \\
    \cmidrule(lr){3-6} \cmidrule(lr){7-10}
    & & \multicolumn{2}{c}{{Recall} } & \multicolumn{2}{c}{{Precision}} & \multicolumn{2}{c}{{Recall}} & \multicolumn{2}{c}{{Precision}} \\
    & Method & Cov. $\uparrow$ & AMR $\downarrow$ & Cov. $\uparrow$ & AMR $\downarrow$ & Cov. $\uparrow$ & AMR $\downarrow$ &
    Cov. $\uparrow$ & AMR $\downarrow$ \\
    \midrule
& Adjoint Sampling w/o RP (\textbf{Ablation})
& 78.25{\color{gray}\tiny$\pm$30.06} & 0.98{\color{gray}\tiny$\pm$0.34} & 57.64{\color{gray}\tiny$\pm$32.00} & \cellhi 1.23{\color{gray}\tiny$\pm$0.43}
& 53.46{\color{gray}\tiny$\pm$36.17} & 1.30{\color{gray}\tiny$\pm$0.48} & 35.16{\color{gray}\tiny$\pm$33.60} & 1.74{\color{gray}\tiny$\pm$0.68} \\
& Adjoint Sampling (\textbf{Ours})
& \cellhi 80.28{\color{gray}\tiny$\pm$27.68} & \cellhi 0.96{\color{gray}\tiny$\pm$0.28} & \cellhi \cellhi 47.97{\color{gray}\tiny$\pm$32.24} & 1.25{\color{gray}\tiny$\pm$0.37} 
& \cellhi 61.36{\color{gray}\tiny$\pm$38.74} & \cellhi 1.17{\color{gray}\tiny$\pm$0.60} & \cellhi 41.19{\color{gray}\tiny$\pm$35.93} & \cellhi 1.46{\color{gray}\tiny$\pm$0.67} \\
\bottomrule
\end{tabular}
}
\label{tab:adjoint_sampling_ablation}
\end{table*}

\clearpage

\subsection{Threshold Ablation for Molecular Conformation Generation}

\begin{figure}[h]
    \centering
    \begin{subfigure}[b]{0.5\linewidth}
    \includegraphics[width=\linewidth]{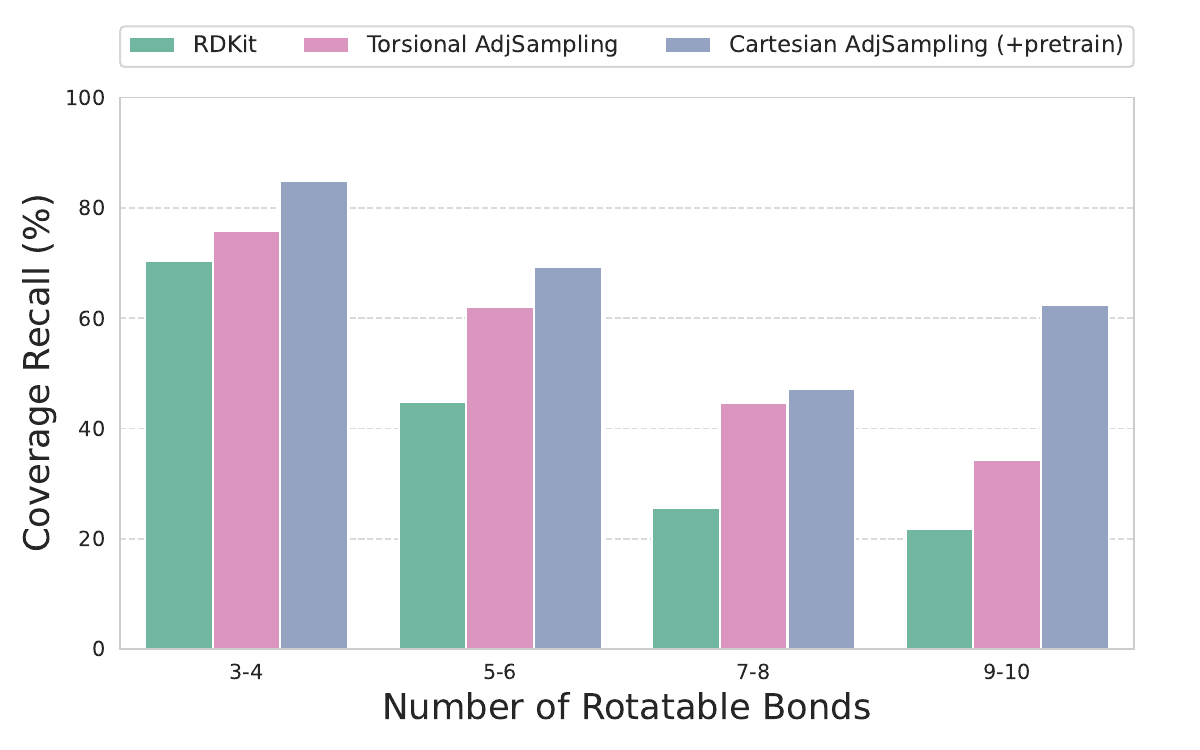}
    \caption*{SPICE (Threshold=0.75\AA)}
    \end{subfigure}%
    \begin{subfigure}[b]{0.5\linewidth}
    \includegraphics[width=\linewidth]{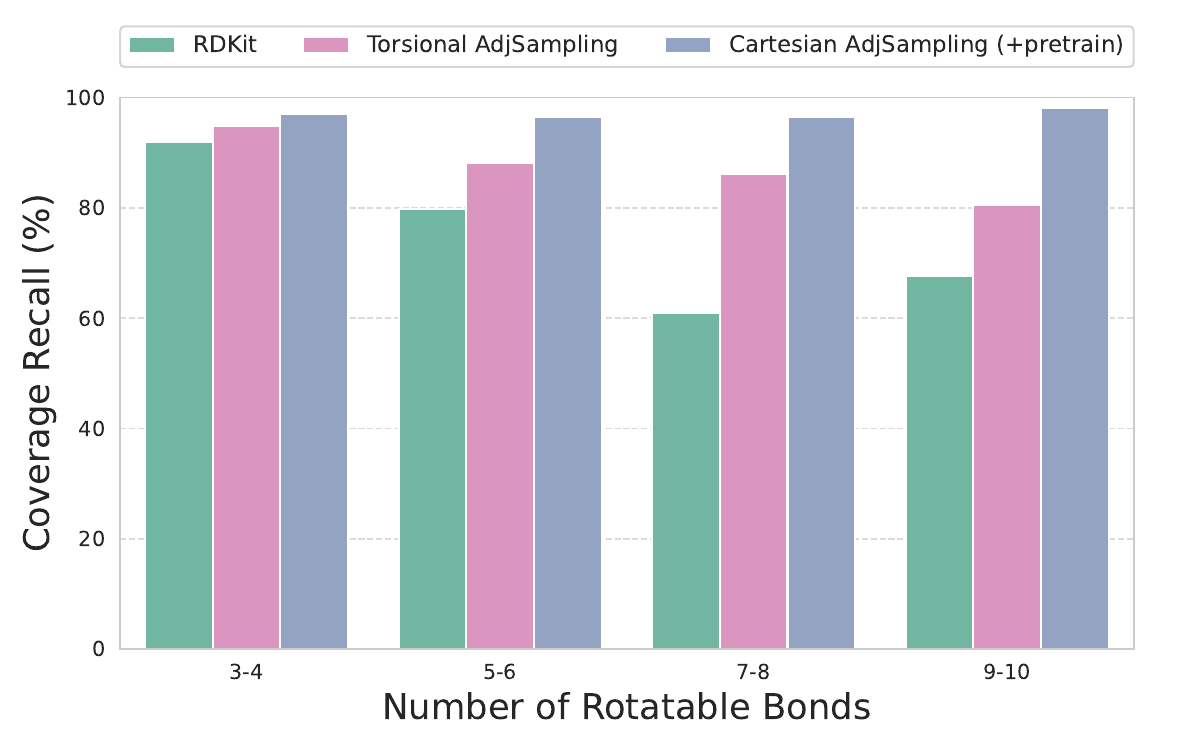}    
    \caption*{SPICE (Threshold=1.25\AA)}
    \end{subfigure}\\
    \begin{subfigure}[b]{0.5\linewidth}
    \includegraphics[width=\linewidth]{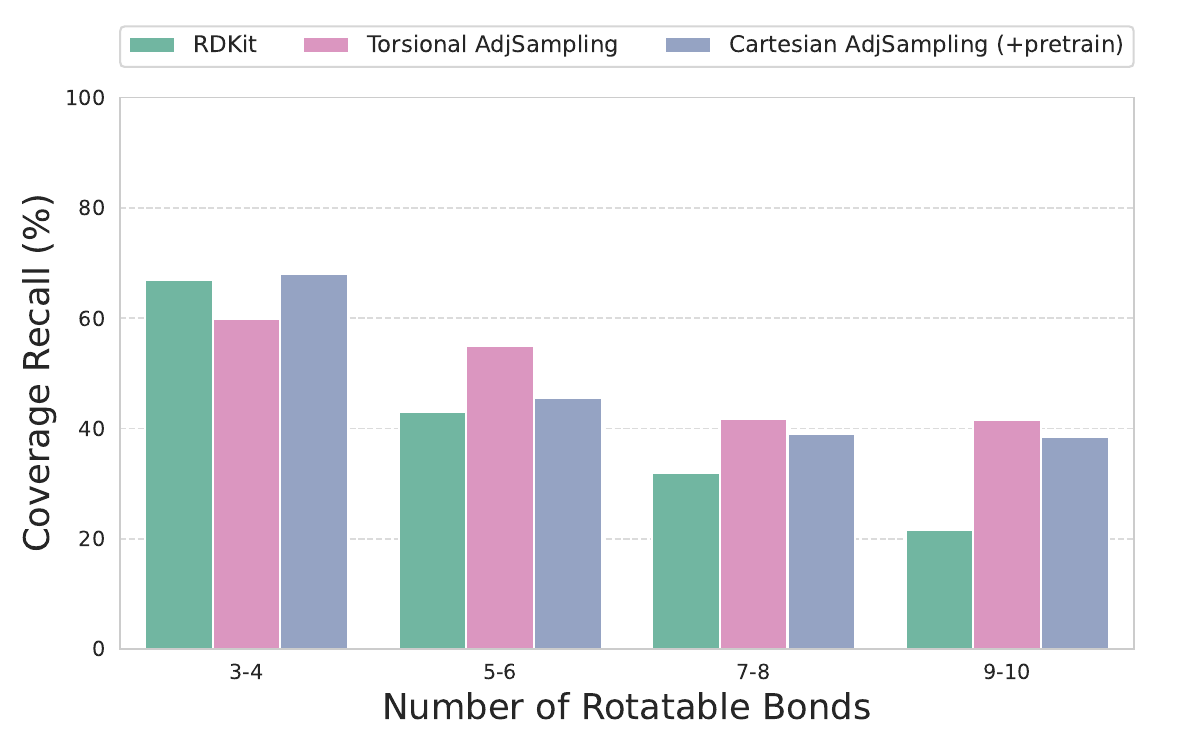}    
    \caption*{GEOM-DRUGS (Threshold=0.75\AA)}
    \end{subfigure}%
    \begin{subfigure}[b]{0.5\linewidth}
    \includegraphics[width=\linewidth]{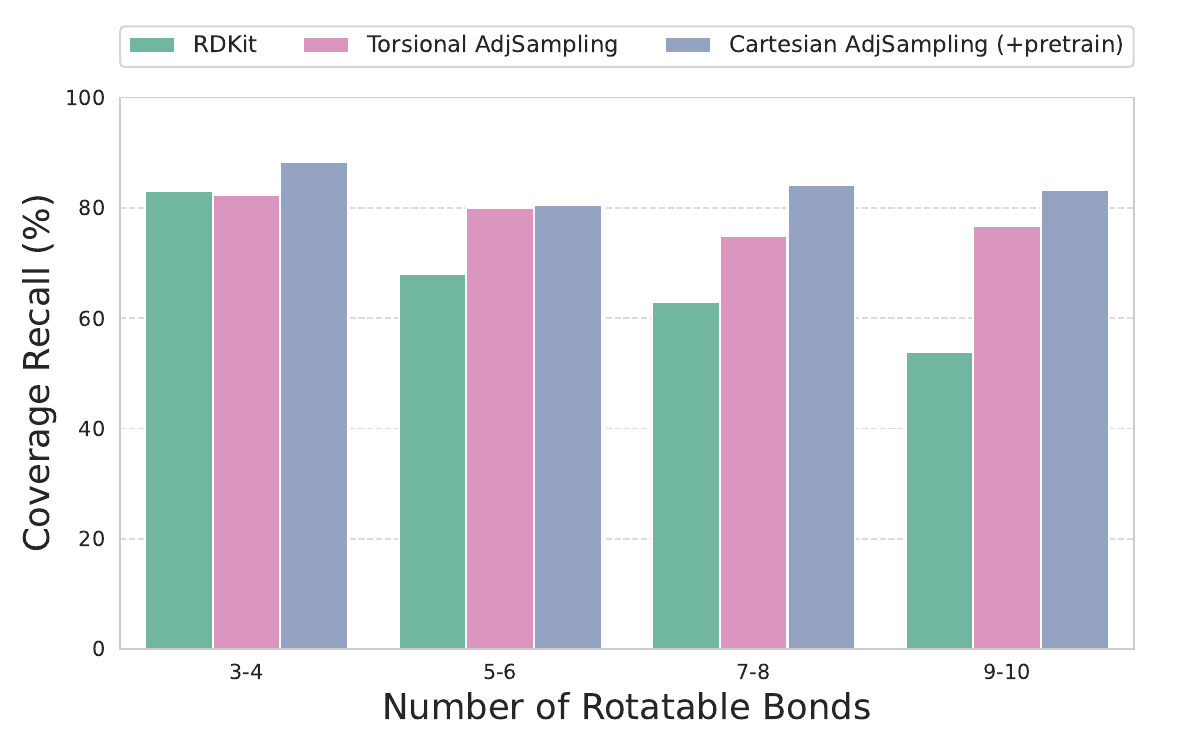}   
    \caption*{GEOM-DRUGS (Threshold=1.25\AA)}
    \end{subfigure}\\
    \caption{Coverage recall (\%) after relaxation for molecules of different number of rotatable bonds. We see that the gap between our model and RDKit increases for higher number of rotatable bonds.}
    \label{fig:recall_comparison_across_rotb}
\end{figure}

\end{document}